\documentclass[11pt]{article}

\usepackage[a4paper,margin=1in]{geometry}

\usepackage[round]{natbib}

\usepackage{amsmath, amsthm, amssymb }
\usepackage[utf8]{inputenc} 
\usepackage[T1]{fontenc}    
\usepackage{lmodern}          
\usepackage[backref=page, colorlinks = True, linkcolor = mydarkblue, citecolor = mydarkblue, urlcolor = mydarkblue]{hyperref}
\usepackage[capitalize,noabbrev]{cleveref}
\usepackage{url}            
\usepackage{booktabs}       
\usepackage{amsfonts}       
\usepackage{nicefrac}       
\usepackage{microtype}      
\usepackage{xcolor}         
\usepackage{multirow}
\usepackage{wrapfig}
\usepackage{authblk}
\usepackage{algorithm}
\usepackage{algpseudocode}

\usepackage[makeroom]{cancel}
\usepackage{graphicx}
\usepackage{svg}

\usepackage{caption}
\usepackage{subfigure}
\usepackage{pifont}
\usepackage{arydshln}
\usepackage{float}
\newtheorem{theorem}{Theorem}

\newtheorem{lemma}{Lemma}
\newtheorem{corollary}{Corollary}
\newtheorem{proposition}{Proposition}

\newtheorem{assumption}{Assumption}
\renewcommand{\theassumption}{A\arabic{assumption}}
\crefname{assumption}{}{}

\definecolor{gencolor}{HTML}{bd453c} 
\definecolor{c1color}{HTML}{942192}      
\definecolor{c2color}{RGB}{0,120,80}       
\definecolor{nlecolor}{HTML}{4477AA}
\definecolor{datacolor}{HTML}{5F5AA5} 
\definecolor{acecolor}{HTML}{CCBB44}
\definecolor{gbisrcolor}{HTML}{FF6B35}
\definecolor{rsnlecolor}{HTML}{FF69B4}
\definecolor{nperscolor}{HTML}{228833}
\definecolor{nplmmdcolor}{HTML}{66CCEE}
\definecolor{darkgreen}{HTML}{1B7F12}

\newcommand{\legendbox}[1]{%
  \textcolor{#1}{\rule{0.55em}{0.55em}}%
}

\usepackage{etoolbox}

\makeatletter

\newcommand{\appendix@addcontentsline}[3]{%
  \edef\@tempa{#1}%
  \edef\@tempb{toc}%
  \ifx\@tempa\@tempb
    \orig@addcontentsline{atoc}{#2}{#3}%
  \else
    \orig@addcontentsline{#1}{#2}{#3}%
  \fi
}

\newcommand{\startappendixtoc}{%
  \let\orig@addcontentsline\addcontentsline
  \let\addcontentsline\appendix@addcontentsline
}

\newcommand{\stopappendixtoc}{%
  \let\addcontentsline\orig@addcontentsline
}

\newcommand{\printappendixtoc}{%
  \clearpage
  \section*{Table of Contents}
  \begingroup
    \let\clearpage\relax
    \@starttoc{atoc}%
  \endgroup
  \clearpage
}

\makeatother

\crefname{oldassumption}{}{}

\def\X{\mathcal{X}}

\def\P{\mathbb{P}}

\definecolor{mydarkblue}{rgb}{0,0.08,0.45}

\DeclareMathOperator{\pif}{\operatorname{PIF}}

\def\KL{\text{KL}}

\newcommand{\xmark}{\ding{55}}

\newenvironment{talign*}
 {\csname align*\endcsname}
 {\endalign}
\newenvironment{talign}
{\align}
{\endalign}

\renewcommand{\thefootnote}{\fnsymbol{footnote}}

\title{Amortised and provably-robust simulation-based inference}

\author{%
Ayush Bharti$^{1,}$\thanks{Equal contribution.}\quad
Charita Dellaporta$^{2,*}$\quad
Yuga Hikida$^{1}$\quad
Fran\c{c}ois-Xavier Briol$^{2}$\\\vspace{0.6ex}
\small $^{1}$Department of Computer Science, Aalto University, Finland\\
\small $^{2}$Department of Statistical Science, University College London, UK\\\vspace{0.4ex}
\small Corresponding author: \texttt{ayush.bharti@aalto.fi}
}

\date{}
\begin{document}

\maketitle

\renewcommand{\thefootnote}{\arabic{footnote}}
\setcounter{footnote}{0}

\begin{abstract}
Complex simulator-based models are now routinely used to perform inference across the sciences and engineering, but existing inference methods are often unable to account for outliers and other extreme values in data which occur due to faulty measurement instruments or human error. In this paper, we introduce a novel approach to simulation-based inference grounded in generalised Bayesian inference and a neural approximation of a weighted score-matching loss. This leads to a method that is both amortised and provably robust to outliers, a combination not achieved by existing approaches. Furthermore, through a carefully chosen conditional density model, we demonstrate that inference can be further simplified and performed without the need for Markov chain Monte Carlo sampling, thereby offering significant computational advantages, with complexity that is only a small fraction of that of current state-of-the-art approaches.
\end{abstract}

\section{Introduction}

Simulation-based inference (SBI, \cite{Cranmer2020}) has emerged as a powerful framework for addressing complex Bayesian inference problems where likelihood functions are intractable but simulation is feasible. Its flexibility and scalability have led to widespread adoption in domains such as particle physics \citep{brehmer2021simulation}, epidemiology \citep{Kypraios2017}, wireless communications \citep{Bharti2022}, and cosmology \citep{alsing2018massive,Jeffrey2021}. In particular, recent advances have focused on conditional density estimation, often enabling amortised inference where a model trained once can be reused for multiple observations \citep{zammitmangion2024neural}.

Despite these successes, a major challenge remains: SBI methods often lack robustness to model misspecification \citep{Cannon2022,Nott2024,Kelly2025}, which can arise from corrupted or incomplete data, outliers, or inadequacies in the simulator. The recent review of \citet{Martin2024} identifies misspecification as one of the most urgent open problems in Bayesian computation, while \citet{Hermans2022} caution that over-confident posteriors, an issue amplified by misspecification, are at risk of precipitating a broader crisis of trust in SBI. Such is the current interest in this open challenge that \citet{Dellaporta2022} were awarded the Best Paper Award at AISTATS 2022 for making concrete progress toward its resolution.

\begin{figure}
    \centering
    \includegraphics[width=0.55\linewidth]{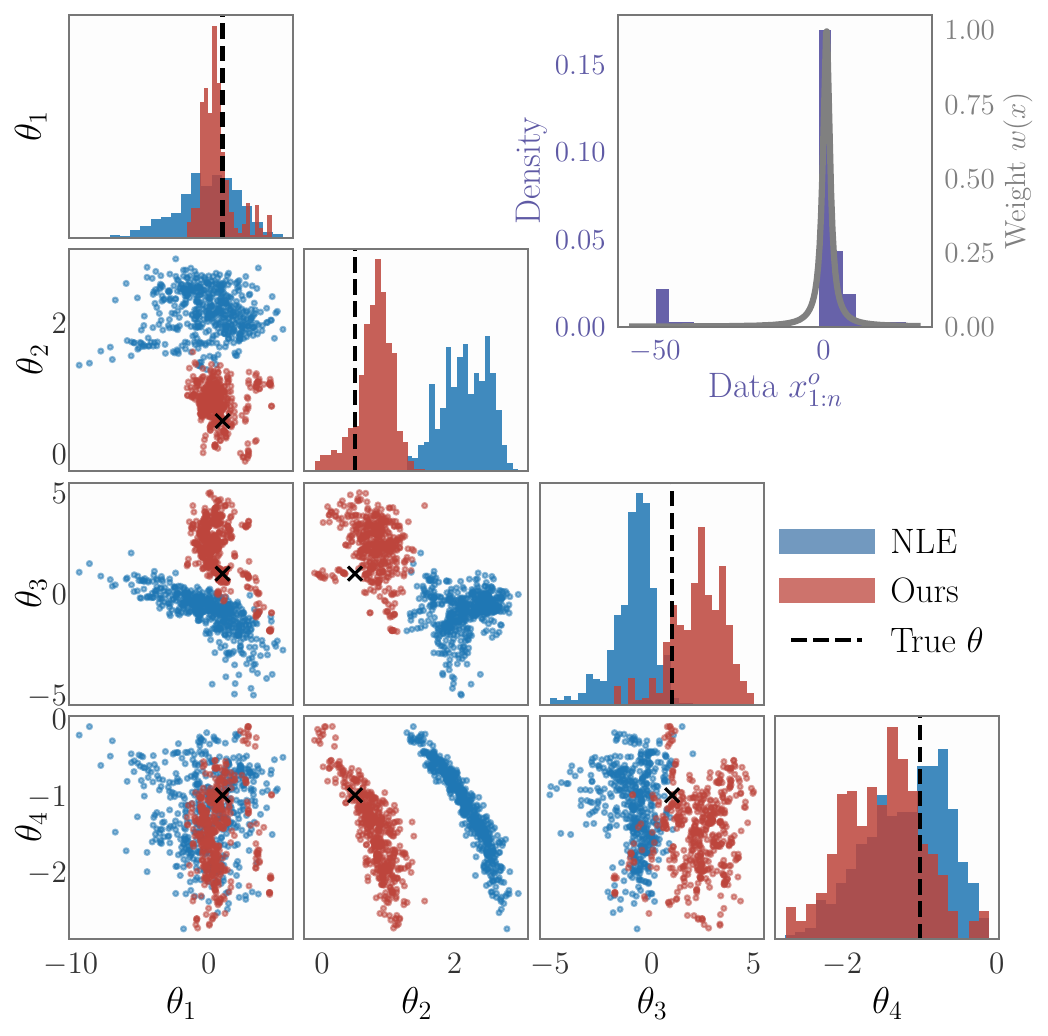}
    \caption{\textbf{Amortised SBI in the presence of outliers.} With a few one-sided outliers (10\%) in the observed data (\legendbox{datacolor}), shown in the top-right plot, the neural likelihood estimation (NLE \legendbox{nlecolor}) posterior of the g-and-k distribution is significantly impacted, specifically parameters $\theta_2$ and $\theta_3$ that govern the scale and skewness of the distribution (details in \Cref{sec:gnk}). In contrast, our proposed method, named neural score-matching Bayes (\legendbox{gencolor}), is robust as we use the function $w(x)$ (\legendbox{gray}) to down-weight the effect of the outliers.}
    \label{fig:illustration}
\end{figure}

Unfortunately, whilst various approaches have been proposed, they either \emph{lack formal robustness guarantees} \citep{Ward2022, Bharti22a, Huang2023,Gloeckler2023,Gao2023,quera2023bayesian, dyer2023gradient, Wehenkel2024,Kelly2024, wu2024testing, swierc2024domain, Mishra2025robustABI, Verma2025, Senouf2025, callaghan2025robust, ruhlmann2025flow, Sadasivan2025, elsemuller2025does, mao2025robust, sun2026amortized}, \emph{are incompatible with amortised inference} \citep{Briol2019,Kisamori2018,Frazier2020, 2021Fujisawa, Dellaporta2022, Frazier2021synthetic,Thomas2020, Pacchiardi2024, Legramanti2022}, or  both. Apart from lacking robustness guarantees, some of these existing  methods also require side information in the form of additional data to achieve robustness \citep{Wehenkel2024, Senouf2025,Mishra2025robustABI}, and therefore have more limited applicability.

We introduce a method that addresses these shortcomings and is specialised for robustness to outliers. Our method is centred around the score-matching objective of \citet{Hyvarinen2006} and resembles the two-step approach of neural likelihood estimation (NLE, \cite{Papamakarios2019}). The first step consists of training a \emph{neural} conditional density model using simulated data, typically implemented via a flexible architecture such as a normalising flow \citep{papamakarios2021normalizing} or an energy-based model. The second step corresponds to updating prior beliefs using generalised Bayesian inference \citep{Bissiri2016} with a weighted score-matching (SM) loss  \citep{Altamirano2023} evaluated on the observed data and approximated with the aforementioned neural surrogate. This results in \emph{neural score-matching Bayes} (\emph{NSM-Bayes}), a method that is both amortised and provably-robust to outliers. This is illustrated in \Cref{fig:illustration} for a g-and-k simulator, where the NLE posterior is severely biased in the presence of outliers, whilst NSM-Bayes remains centred around the correct parameter value. 

We propose two versions of our approach. The first version (\Cref{sec:SM-Bayes-SBI}) is the most general, wherein any flexible conditional density model and prior can be used, akin to NLE. The second version is a conjugate special case of NSM-Bayes, termed \emph{NSM-Bayes-conj} (\Cref{sec:SM-Bayes-SBI-conjugate}). NSM-Bayes-conj considers Gaussian priors and  restricts the flexibility of the conditional density model to a specific form of energy-based model, resulting in substantial reductions in the computational cost of inference. In \Cref{sec:theory}, we present theoretical guarantees that these two methods are robust to outliers, providing the first such guarantees for amortised SBI methods.
In \Cref{sec:experiments}, we then consider simulators from epidemiology and telecommunications engineering and demonstrate that our methods provide reliable inference under the presence of outliers, with NSM-Bayes-conj simultaneously leading to significant computational cost savings at inference time. 

\section{Background} \label{sec:background}

Let $\mathcal{X} = \mathbb{R}^{d_{\mathcal{X}}}$ be some domain and  $\mathcal{P}(\mathcal{X})$ be the space of probability distributions on $\mathcal{X}$. We consider Bayesian inference for some parameter $\theta \in \Theta \subseteq \mathbb{R}^{d_{\Theta}}$ of a model $\{\P_\theta\}_{\theta \in \Theta} \subset \mathcal{P}(\mathcal{X})$ with (typically unknown) density $p(\cdot|\theta)$, based on some independent and identically distributed (iid) observed data $x^o_{1:n} = \{x^o_1,\ldots,x^o_n\} \in \mathcal{X}^n$ obtained from the (unknown) data-generating process $\P_0 \in \mathcal{P}(\mathcal{X})$ with density $p_0$, which we assume for now to be uncorrupted. Assuming a prior density $\pi(\theta)$, the posterior density is
\begin{align}\label{eq:Bayes}
    \pi(\theta|x^o_{1:n}) \propto \prod_{i=1}^n { p(x^o_i|\theta)} \pi(\theta).
\end{align}
We now briefly review approximate Bayesian (\Cref{sec:background_sbi}) and generalised Bayesian (\Cref{sec:background_gbi}) inference methods for simulators, with an emphasis on amortisation and robustness.
\subsection{Approximate Bayesian Inference for Simulators}
\label{sec:background_sbi}

In simulation-based inference (SBI; also called likelihood-free inference), performing exact Bayesian inference is not feasible since the density $p(\cdot|\theta)$ (or equivalently, the likelihood) is intractable, typically due to unobserved latent variables or the computational cost of evaluation. Instead, we assume  one can simulate data from this model for any parameter value through a simulator, and we then use such simulations to construct an approximate posterior. 
The most established SBI approach is approximate Bayesian computation (ABC, \cite{Beaumont2019}), which can be made robust; see \citet{Kisamori2018,Frazier2020, 2021Fujisawa, Legramanti2022}. Another approach is Bayesian synthetic likelihood \citep{Price2018}, whose robustness has been extensively studied in the works of \citet{Frazier2021,Frazier2021synthetic}. Finally, \citet{Dellaporta2022} proposed to use the posterior bootstrap of \citet{Lyddon2018,Fong2019} with the minimum distance estimator from \cite{Briol2019} to get a robust non-parametric posterior which is highly parallelisable. Although most of these approaches are provably-robust to outliers, they unfortunately \emph{do not permit amortisation}; i.e. given new observations, the entire method needs to be re-run to obtain an approximate posterior, which can  be extremely computationally demanding.

Neural likelihood estimation (NLE, \cite{Papamakarios2019, lueckmann2019likelihood, boelts2022flexible, radev2023jana}) is an alternative approach to SBI. It consists of (i) using a conditional density model based on some parametric class $\{q_\phi(\cdot|\theta)\}_{\phi \in \Phi}$ with $\Phi \subseteq \mathbb{R}^{d_\Phi}$ (typically based on a neural network; e.g. a mixture density network or normalising flow) to estimate $p(\cdot|\theta)$ for any  $\theta \in \Theta$, then (ii) replacing the unknown likelihood with the fitted model to perform Bayesian inference. Step (i) is performed by obtaining samples $\{(\theta_i, x_i)\}_{i=1}^m$ through simulating parameter values from the prior $\theta_i \sim \pi$, then simulating data from the simulator $x_i \sim \P_{\theta_i}$. Given these simulations, the optimal value $\phi \in \Phi$ is usually obtained by minimising the empirical negative log-likelihood loss:
\begin{align}\label{eq:NLE_loss}
    \hat{\phi}_m := \arg\min_{\phi \in \Phi} - \frac{1}{m} \sum_{i=1}^m  \log q_\phi(x_{i} | \theta_i).
\end{align}
Given $\hat{\phi}_m$, step (ii) consists of replacing $p(\cdot|\theta)$ by $q_{\hat{\phi}_m}(\cdot|\theta)$ in \Cref{eq:Bayes} and approximating this posterior, denoted $\pi_\text{NLE}(\theta|x_{1:n}^o, \hat{\phi}_m)$, typically through Markov chain Monte Carlo (MCMC,  \cite{Robert2004}). NLE is often called \emph{partially amortised} since inference on a new set of observations only requires the MCMC step (i.e. step (ii)), without the need to re-train $q_\phi$ from scratch (i.e. we do not have to repeat step (i)). Moreover, a new observation $x_{n+1}^o \in \mathcal{X}$ can be included simply by adding a term $q_{\hat{\phi}_m}(x_{n+1}^o|\theta)$ in step (ii), which again requires no additional simulation. 
Alternatively, neural posterior estimation (NPE, \citep{Papamakarios2016, Lueckmann2017, Greenberg2019, Radev2022}) approximates the posterior directly using conditional density models, and is thus called \emph{fully amortised} as it does not require any MCMC post-training.

Several neural SBI methods have been shown \textit{empirically} to provide robustness to model misspecification. These are based on explicit modelling of the misspecification \citep{Ward2022,Kelly2024,Verma2025, callaghan2025robust}, penalisation of the learning objectives \citep{Huang2023,Gloeckler2023}, or side information \citep{Wehenkel2024,Mishra2025robustABI,Krouglova2025,Hikida2025}. Unfortunately, \emph{none of these methods have theoretical guarantees of robustness to outliers} or other forms of model misspecification.

Although less directly relevant to our paper, we note that an interesting complementary line of work has investigated diagnostics for prior or simulator misspecification, including formal tests \citep{Key2021,Ramirez-Hassan2024} and empirical checks \citep{Leclercq2022,Elsemuller2023,Schmitt2024,montel2024tests,yuyan2025robust}.

\subsection{Generalised Bayesian Inference for Simulators}
\label{sec:background_gbi}

An alternative to Bayesian inference which can provide robustness is generalised Bayesian Inference (GBI) \citep{Bissiri2016,Knoblauch2019}. This consists of updating beliefs through an empirical loss function $\mathcal{L}:\Theta \times \mathcal{X}^n \rightarrow \mathbb{R}$ computed for the observed data $x^o_{1:n}$:
\begin{align}\label{eq:genBayes}
    \pi_\mathcal{L}(\theta|x^o_{1:n}) \propto \exp\left(- \beta n \mathcal{L}(\theta;x^o_{1:n})\right) \pi(\theta),
\end{align}
where $\beta >0$ is a temperature/learning rate parameter \citep{Wu2023} which weights the relative importance of the loss and the prior. Several special cases are of interest: $\mathcal{L}_{\text{Bayes}}(\theta;x^o_{1:n}) := -\frac{1}{n}\sum_{i=1}^n \log p(x_i^o|\theta)$ and $\beta=1$ recovers standard Bayesian inference as given in \Cref{eq:Bayes},  $\mathcal{L}_{\text{NLE}}(\theta;x^o_{1:n}, \hat{\phi}_m) := - \frac{1}{n}\sum_{i=1}^n \log q_{\hat{\phi}_m}(x_i^o|\theta)$ with $\beta=1$ recovers NLE (as well as Bayesian synthetic likelihoods, which can be viewed as a special case), and \cite{Schmon2020} showed how to recover ABC by marginalising a generalised posterior on an extended space.

However, none of these losses provide robustness to model misspecification, and alternative robust losses that have been proposed for SBI have significant limitations. \citet{Matsubara2021,Matsubara2022} showed that loss functions based on score-based discrepancies are well-suited to performing GBI with intractable likelihoods, but their approach still requires access to an unnormalised expression of the likelihood. \citet{Pacchiardi2024} showed that kernel-based scoring-rule losses give posteriors akin to the maximum mean discrepancy (MMD) Bayes posterior of \citet{Cherief-Abdellatif2019-MMDBayes} and can resolve this limitation. Unfortunately, their approach requires the use of advanced Monte Carlo methods which can be hard to tune, and it is also not amortised.
\citet{Gao2023} proposed a method that amortises MMD-Bayes across different observed datasets available during training, but their method has no formal robustness guarantees. Concurrently to the present paper, \cite{sun2026amortized} proposed a power-posterior version of NPE, but power-posteriors are not provably robust to outliers and have the downside that they equally downweight corrupted and uncorrupted observations. These limitations leave open the question of how to construct an amortised SBI method which is provably robust.

\section{Methodology}
\label{sec:method}

We now present our novel approach, called neural score-matching Bayes (NSM-Bayes). A key tool used for this is the score matching (SM) divergence, which was popularised for parameter estimation by \citet{Hyvarinen2006}, and further extended by \citet{Hyvarinen2007,Lyu2009,Barp2019,Yu2019,Yu2022,Xu2022}. We use  the weighted (or `diffusion') SM divergence of \citet{Barp2019} as the weights will give us the additional flexibility needed to achieve robustness. For two  densities $p ,q$ on $\mathcal{X} = \mathbb{R}^{d_\X}$ 
and an invertible matrix-valued weight function $W:\mathcal{X} \rightarrow \mathbb{R}^{d_{\mathcal{X}} \times d_{\mathcal{X}}}$, this is given by:  
\begin{align}\label{eq:score-matching-definition}
    \text{SM}_W(p || q) :=  \mathbb{E}_{X \sim p}\left[\left\|W(X)^\top(\nabla_x \log q(X) - \nabla_x \log p(X))\right\|_2^2\right].
\end{align}
This divergence is well-defined when $p$ and $q$ are positive and continuously differentiable, $\nabla_x \log p - \nabla_x \log q \in L^2(q)$, and the weight $W$ is bounded and continuously differentiable.
Although we operate on  $\mathbb{R}^{d_{\mathcal{X}}}$ throughout,  these conditions can be extended to bounded domains \citep[see][Appendix B.5]{Altamirano2023}.
Under certain regularity conditions on $p$ and $q$ \citep{Barp2019} and using integration-by-parts, the SM divergence can be expressed as:
\begin{align}
    \text{SM}_W(p || q) = \mathbb{E}_{X \sim p}\left[\left\|W(X)^\top \nabla_x \log q(X) \right\|_2^2 + 2 \nabla_x \cdot \left( W(X) W(X)^\top \nabla_x \log q(X) \right)\right] + C_p,
    \label{eq:score-matching-divergence}
\end{align}
where $C_p \in \mathbb{R}$ is a constant with respect to $q$. This expression of the SM divergence forms the basis of our NSM-Bayes (\Cref{sec:SM-Bayes-SBI}) and NSM-Bayes-conj (\Cref{sec:SM-Bayes-SBI-conjugate}) method. In \Cref{sec:hyperparameter-selection}, we discuss how to select the hyperparameters of our method.

\subsection{Neural Score-matching Bayes for Simulation-based Inference}\label{sec:SM-Bayes-SBI}

We first present the most general case of our method, and assume the observed data $x^o_{1:n} \sim \P_0$ may contain outliers which can arbitrarily bias the standard Bayes posterior. One way to obtain robustness is through a specific form of GBI called weighted SM-Bayes \citep{Altamirano2023,Altamirano2024,Reimann2024,Laplante2025,Rooijakkers2025,Ezzerg2025}, which is inspired by the score-based approach to generalised Bayes of \citet{Matsubara2021,Matsubara2022}:
\begin{align}\label{eq:score_matching_posterior}
    \pi_\text{SM}(\theta|x^o_{1:n}) \propto \exp \left( - \beta n \mathcal{L}_{\text{SM}}\left(\theta;x_{1:n}^o \right) \right) \pi(\theta).
\end{align}
Here, the loss function $\mathcal{L}_{\text{SM}}(\theta;x_{1:n}^o)$ can be expressed as:
\begin{align}\label{eq:SMloss}
   \mathcal{L}_{\text{SM}}(\theta;x_{1:n}^o) 
     := \frac{1}{n} \sum_{i=1}^n \left\|W(x_i^o)^\top \nabla_x \log p(x_i^o|\theta) \right\|_2^2 +  2 \nabla_x \cdot \left( W(x_i^o)W(x_i^o)^\top \nabla_x \log p(x_i^o|\theta) \right).
\end{align}
This corresponds to a Monte Carlo estimator of $\text{SM}_W(p_0||p(\cdot|\theta))$, based on \Cref{eq:score-matching-divergence}, up to the additive constant which can be dropped from the loss term since it can be incorporated into the normalisation constant of the SM-Bayes posterior. 
Throughout this paper, we will assume the following about the matrix-valued weight function $W$:
\begin{assumption}\label{assumption:weight}
    $W(x) := w(x) I_{d_\X}$ for some $w: \mathcal{X} \rightarrow \mathbb{R}_{+}$ and $d_\X$-dimensional identity matrix $I_{d_\X}$, $\sup_{x \in \mathcal{X}} w(x) <\infty$ and $\sup_{x \in \mathcal{X}} \| \nabla_x w(x) \|_{2} < \infty$.
\end{assumption}
The diagonal form of $W$ is the most common choice in practice \citep{Altamirano2024}, and the boundedness of $w$ and $\nabla w$ are mild regularity conditions necessary to guarantee robustness, as formalised in \Cref{sec:theory}. Though we recommend a specific $W$ satisfying $\Cref{assumption:weight}$ later in the section, we keep the presentation general.
In SBI, the exact SM loss is intractable since $p(\cdot | \theta)$ is itself intractable. However, assuming we have a surrogate $q_{\hat \phi_m}(\cdot | \theta)$ satisfying certain regularity conditions, this loss can be easily approximated.

\begin{assumption} \label{assumption:regularity-q}
 For all $\phi \in \Phi$, $q_\phi(\cdot | \theta)$ is a twice-differentiable density and $p_0(\cdot) w(\cdot)^2 \nabla_x \log q_\phi(\cdot | \theta), \nabla_x \cdot (p_0(\cdot) w(\cdot)^2  \nabla_x \log q_\phi(\cdot| \theta))$ are integrable  on $\mathbb{R}^{d_{\mathcal{X}}}$.
\end{assumption}
Under \Cref{assumption:weight} and \Cref{assumption:regularity-q}, we call this approximation the \emph{neural SM} (NSM) loss, and it is given by
\begin{align}\label{eq:step2loss_simplified}
   \mathcal{L}_{\text{NSM}}\left(\theta;x_{1:n}^o, \hat \phi_m \right) 
     := &\frac{1}{n} \sum_{i=1}^n w(x_i^o)^2
 \left\| \nabla_x \log q_{\hat{\phi}_m}(x_i^o|\theta)\right\|_2^2
     +  2 \left(\nabla_x w(x_i^o)^2\right)^\top \nabla_x \log q_{\hat{\phi}_m}(x_i^o|\theta) \nonumber \\
     & +2w(x_i^o)^2 \text{Tr}\left(\nabla^2_x \log q_{\hat{\phi}_m}(x_i^o|\theta)\right).
\end{align}
Doing GBI with $\mathcal{L}_{\text{NSM}}$ results in our neural-surrogate approximation of SM-Bayes, called \emph{NSM-Bayes}:
\begin{align}
   \pi_\text{NSM}\left(\theta \big| x^o_{1:n}, \hat \phi_m \right)  \propto  \exp{\left( - \beta n \mathcal{L}_{\text{NSM}}\left(\theta;x_{1:n}^o, \hat \phi_m \right)\right)} \pi(\theta) \label{eq:nsm-Bayes_posterior}
\end{align}
Thus, similarly to NLE, our NSM-Bayes approach proceeds in two steps: we first train a conditional density model $q_\phi$ using simulated data $\{(\theta_i, x_i)\}_{i=1}^m$, and then update prior beliefs using the NSM loss in \Cref{eq:nsm-Bayes_posterior}. Importantly, we do not put any restrictions on how $q_\phi$ is selected or how $\hat{\phi}_m$ has been obtained. In our experiments, $q_\phi$ will either be a mixture density network (MDN) \citep{bishop1994mixture} or a normalising flow, e.g. a masked autoregressive flow (MAF) \citep{Papamakarios2017_maf} trained through \Cref{eq:NLE_loss}, but any alternative model where the score function is available could be used.
Note that NSM-Bayes is partially amortised; once $\hat{\phi}_m$ is estimated by training the neural surrogate, we can directly use $q_{\hat \phi_m}$ to add additional observations or perform inference on a new dataset, without ever needing to run additional simulations.

\subsection{Conjugate Neural Score-matching Bayes}
\label{sec:SM-Bayes-SBI-conjugate}

We now consider a special case of the framework above. assume, without loss of generality, that our model is based on an unnormalised family; i.e. $ q_\phi(x|\theta) = \tilde{q}_\phi(x|\theta) / C(\theta,\phi)$,
where $\tilde{q}_\phi(x|\theta)$ can be evaluated in closed-form, but ${C(\theta,\phi)}>0$ is a (possibly intractable) normalisation constant. 
In particular, we consider models which are (natural) exponential families in $\theta$.
\begin{assumption}\label{assumption:exptheta}
Let $q_{\phi}(x|\theta) \propto \exp(T_{\phi}(x)^\top \theta + b_{\phi}(x))$,
where the statistic $T_{\phi}:\mathcal{X} \rightarrow \mathbb{R}^{d_\Theta}$ and $b_\phi:\mathcal{X} \rightarrow \mathbb{R}$ are both twice-differentiable and such that $q_{\phi}$ satisfies \Cref{assumption:regularity-q}.
\end{assumption}

Examples include energy-based models which are linear in $\theta$, as well as certain instances of \citet{Arbel2017,Sasaki2018,Pacchiardi2020}. This form is restrictive in $\theta$, but $T_{\phi}$ and $b_\phi$ can be flexible non-linear maps, such as neural networks with smooth activation functions (e.g. tanh, softplus, and GELU). 

Since the model in \Cref{assumption:exptheta} is unnormalised, we will not be able to use $\mathcal{L}_\text{NLE}$ in \textbf{step (i)}. Instead, we utilise score-matching to train $\phi$ since scores do not require any knowledge of the normalisation constant: $\nabla_x \log q_\phi(x|\theta) = \nabla_x \log \tilde{q}_\phi(x|\theta)$. As $q_\phi$ is trained on simulated data $\{(\theta_i, x_i)\}_{i=1}^m$ which does not contain any outliers, we use the unweighted SM objective for conditional densities due to \citet{Altamirano2024} as the loss function $J(\phi) := \mathbb{E}_{\theta \sim \pi} \left[\text{SM} \left(p(\cdot|\theta) ||q_\phi(\cdot|\theta)\right) \right],$ and estimate the parameters of our conditional density estimator as $\hat{\phi}_m := \arg\min_{\phi \in \Phi} J_{m}(\phi)$, where $J_{m}(\phi)$ is a Monte Carlo estimate of $J(\phi)$ obtained through the simulated data, with the additive constant dropped and simplified:
\begin{align}\label{eq:step1SM_estimator}
   J_{m}(\phi)  = \frac{1}{m} \sum_{i=1}^m  \left\|\nabla_x \log q_\phi(x_i|\theta_i) \right\|_2^2   + 2 \text{Tr}\left(  \nabla_x^2 \log  q_\phi(x_i|\theta_i)\right).
\end{align}
This approach is similar to other score-based methods already proposed in the SBI literature \citep{Pacchiardi2020,Geffner2023, simons2023neural, linhart2024diffusion, Sharrock2024, Gloeckler24_simformer,  Khoo2025, nautiyal2025condisim, jiang2025simulation}. In particular, \citet{Pacchiardi2020} also propose learning a neural exponential family surrogate via score-matching. However, they use MCMC to sample from the Bayes posterior instead of using GBI, and their method is therefore not robust.

The advantage of  \Cref{assumption:exptheta} is that \textbf{step (ii)} of obtaining the NSM-Bayes posterior can now be performed in a fully conjugate manner leading to the so-called NSM-Bayes-conj posterior.
\begin{proposition}\label{prop:conjugacy}
Suppose \Cref{assumption:exptheta} holds, $\pi(\theta) \propto \mathcal{N}(\theta;\mu,\Sigma)$, and $W(x) = I_{d_\X} w(x)$ for some $w: \X \rightarrow \mathbb{R}$. Then  $\pi_{\text{NSM}}(\theta|x_{1:n}^o, \hat \phi_m) \propto \mathcal{N}(\theta;\mu_{n,m},\Sigma_{n,m})$ where:
\begin{align*}
& \mu_{n,m} := \Sigma_{n,m} \left( \Sigma^{-1} \mu  - 2 \beta \left( \sum_{i=1}^n w(x_i^o)^2 \nabla_x T_{\hat{\phi}_m}(x_i^o) \nabla_x b_{\hat{\phi}_m}(x_i^o)  - \nabla_x \cdot \left(w(x_i^o)^2 \nabla_x T_{\hat{\phi}_m}(x_i^o)^\top\right)^\top \right) \right),  \\
& \Sigma_{n,m}^{-1}  := \Sigma^{-1}+ 2 \beta \sum_{i=1}^n w(x_i^o)^2\nabla_x T_{\hat{\phi}_m}(x_i^o) \nabla_x T_{\hat{\phi}_m}(x_i^o)^\top. 
\end{align*}
\end{proposition}
Here we use the convention: $\nabla_x T_{\hat \phi_m}(x_i^o) \in \mathbb{R}^{d_\Theta \times d_\X}$.
This result is a direct application of Proposition 3.1 in \citet{Altamirano2023} with a diagonal weight function and surrogate likelihood, see \Cref{appendix:proof_conjugacy_prop} for the proof.
This result makes our approach fully amortised; once the model is trained in step (i), we can introduce new observed data points and/or change the (Gaussian-like) prior, and we directly obtain an expression for the NSM-Bayes posterior without the need for MCMC, which may have tuning parameters that are difficult to select. NSM-Bayes-conj also amortises over the choice of the weight function and the learning rate, allowing the method to be adapted for robustness according to the specific needs of the observed dataset. 
This computational advantage comes at the cost of restrictions on the choice of surrogate $q_\phi$ and the prior $\pi$, though we found that this did not have a large impact on performance in our numerical experiments.

\subsection{Hyperparameter Selection}
\label{sec:hyperparameter-selection}

We select the weight as $W(x) := w(x) I_{d_\X}$ so as to induce robustness, following the suggestions in  \citet{Barp2019,Altamirano2023,Altamirano2024,Duran-Martin2024,Liu2024}. In particular, we take 
$w:\mathcal{X} \rightarrow \mathbb{R}$ to be the inverse multi-quadratic (IMQ) function:
\begin{align} \label{eq:imq-weight}
 w_\zeta(x) := \left( 1+ \|x-\hat{\nu}_n\|_{\hat{\Xi}_n^{-1}}^2\right)^{-\frac{1}{\zeta}}
\end{align}
for some $\zeta > 0$, since this choice satisfies \Cref{assumption:weight}. Here, $\hat{\nu}_n, \hat{\Xi}_n$ are robust estimators of the mean and covariance \citep{rousseeuw1999fast} of $\P_0$ based on $x_{1:n}^o$, and robustness of these estimators is essential since we expect outliers. The function $w_\zeta$ down-weights the outlying data points that are far from $\hat{\nu}_n$ (since $w_\zeta(x) \rightarrow 0$ as $x \rightarrow \pm \infty)$, thereby reducing their impact on the generalised posterior. The strength of this down-weighting is controlled by $\zeta$, with most existing work fixing $\zeta=2$. However, the extra flexibility provided by $\zeta$ will be essential in our case, with $0 < \zeta \leq 1$ guaranteeing outlier-robustness when using masked autoregressive flows for $q_\phi$, as we will see theoretically in \Cref{sec:theory} and empirically in \Cref{sec:experiments}.

We select the temperature/learning rate $\beta>0$ with a focus on calibration.
This is important since \citet{Hermans2022} have highlighted issues with the calibration of SBI methods. To do so, we can rely on prior work by \citet{Lyddon2019,Syring2019,Wu2023,Matsubara2022}. Specifically, we use the method of \citet{Syring2019} and select $\beta$ such that the $100(1-\alpha)\%$ credible region of our generalised posterior has empirical frequentist coverage close to $1-\alpha$; see \Cref{alg:learning-rate-calibration-conjugate}. We made this choice because it is one of the most widespread in the GBI literature, but we note that our approach could also be used with some of the other calibration approaches. Following \citet{Syring2019}, we first solve $\hat{\theta}_n = \arg \min_\theta \mathcal{L}_{\text{NSM}}(\theta; x^o_{1:n}, \hat{\phi}_m)$ and then define a stochastic approximation procedure over $\beta$ using estimates of the credible region and empirical coverage from bootstrapped samples of $x^o_{1:n}$ and $\hat \theta_n$. At each iteration, we increase $\beta$ if the empirical coverage exceeds $1-\alpha$, and decrease $\beta$ otherwise.
Note that $\hat{\theta}_n$ can be obtained in closed-form for NSM-Bayes-conj \citep{laplante2025conjugate}, while for NSM-Bayes, we estimate $\hat{\theta}_n$ using the Adam optimiser \citep{kingma2014adam}. As repeated posterior inference for different bootstrapped data sample is computationally intensive for NSM-Bayes, we run MCMC once for an initial $\beta$ value and then adopt an importance sampling strategy to estimate the posterior at other $\beta$ values \citep{Syring2019}. As large updates to $\beta$ can make this procedure unstable, we rerun MCMC to refresh the samples at the current $\beta$ if the effective sample size falls below a preset threshold.
Once $\beta$ is selected, we sample again from $\pi_\text{NSM}(\theta|x^o_{1:n}, \hat{\phi}_m)$. Thus, NSM-Bayes requires (at least) two MCMC runs to calibrate $\beta$ and perform inference; see \Cref{app:setting-learning-rate} for further details. 

\begin{algorithm}[t]
\caption{Learning Rate Calibration}
\label{alg:learning-rate-calibration-conjugate}
\begin{algorithmic}[1] 
\Require Data $x_{1:n}^o$, GBI loss $\mathcal{L}$, initial learning rate $\beta_0$, iterations $T$, number of bootstrap samples $B$, step-size $\kappa_t$, target coverage $1-\alpha$
\State Compute $\hat \theta_n = \arg \min_{\theta 
    \in \Theta} \mathcal{L}(\theta;x_{1:n}^o )$ for any GBI loss $\mathcal{L}$. 
\For{$t = 1$ \textbf{to} $T$}
  \State Draw $B$ bootstrap datasets $x_{1:n}^{\star(b)}$ by resampling from $x_{1:n}^o$ with replacement.
  \For{$b = 1$ \textbf{to} $B$}
      \State Estimate the GBI posterior mean $\hat \mu_{\beta_t}^{(b)}$ and covariance $\hat \Sigma_{\beta_t}^{(b)}$. 
      \State Form the credible region using the $1-\alpha$ quantile $\chi^2_{1-\alpha,d_\Theta}$ of a chi-squared distribution:
            \begin{equation*}
                C_{\alpha, \beta_t}(x_{1:n}^{\star(b)}) = \left\{ \theta : \left(\theta - \hat \mu_{\beta_t}^{(b)}\right)^\top \left(\Sigma_{\beta_t}^{(b)}\right)^{-1} \left(\theta - \hat \mu_{\beta_t}^{(b)}\right) \leq \chi^2_{1-\alpha, d_\Theta}  \right\}.
            \end{equation*}
      \EndFor
    \State Estimate the coverage: $\hat c(\beta_t) = \frac{1}{B} \sum_{b=1}^B \mathbb{I}\left\{\hat{\theta}_{n} \in C_{\alpha, \beta_t}\left(x_{1:n}^{\star(b)}\right)\right\}$.
    \State Update $\beta$: $\beta_{t+1} = \beta_t + \kappa_t\big(\hat c(\beta_t) - (1-\alpha) \big)$.
\EndFor
\State \Return $\beta_T$
\end{algorithmic}
\end{algorithm}

\section{Robustness}
\label{sec:theory}

We are now ready to provide our formal results on robustness. We first prove that the NLE posterior is \textit{not} robust to outliers under standard assumptions (\Cref{theorem:non-robustness-NLE}) and then provide two robustness results for our method; a general result for NSM-Bayes (\Cref{theorem:robustness}), and a more refined result for the special case of NSM-Bayes-conj (\Cref{theorem:robustness-exp}).

\subsection{Lack of Robustness of Neural Likelihood Estimation}

To do so, we consider the classical framework of \cite{huber1981robust}, where given a data distribution $\mathbb{P}_0 \in \mathcal{P}(\mathcal{X})$, we consider its $\epsilon$-contaminated counterpart $\mathbb{P}_{\epsilon, x^c} = (1 - \epsilon) \mathbb{P}_0 + \epsilon \delta_{x^c}$, where $\delta_{x^c}$ denotes the Dirac measure at some point $x^c \in \mathbb{R}^{d_{\mathcal{X}}}$ and $\epsilon \in [0,1]$ denotes the contamination degree. Our data-generating process is therefore denoted $\mathbb{P}_{\epsilon,x^c}$, and the uncorrupted process $\mathbb{P}_0$ is now the distribution we are hoping to recover through inference. To this end, for any $\mathbb{P} \in \mathcal{P}(\mathcal{X})$ and loss $\mathcal{L}$, we define the GBI posterior given $\mathbb{P}$ as
\begin{align}
    \pi_\mathcal{L}(\theta|\mathbb{P}) \propto \exp\left(- \beta n \mathcal{L}(\theta;\mathbb{P})\right) \pi(\theta),
\end{align}
where $\mathcal{L}(\theta;\mathbb{P})$ is now an expected loss instead of an empirical loss.
Using this notation, the GBI posterior in 
\Cref{eq:genBayes} satisfies $\pi_{\mathcal{L}}(\theta | x^o_{1:n}) = \pi_{\mathcal{L}}(\theta | \mathbb{P}_n) $ where $\mathbb{P}_n := \frac{1}{n} \sum_{i=1}^n \delta_{x_i^o}$ is the empirical measure of the observations. 
Robustness is then defined through studying the behaviour of a \textit{posterior influence function (PIF)}, a quantity well-studied in GBI \citep{ghosh2016robust, Matsubara2021, Altamirano2023, Duran-Martin2024, Altamirano2024, Laplante2025,Rooijakkers2025} which measures the influence of a contamination on $\pi_\mathcal{L}$. Following \citet{ghosh2016robust, Altamirano2023, Matsubara2021}, we first define:
\begin{align} \label{eq:pif-gen}
    \pif_{\Delta}(x^c, \theta, \mathbb{P}_0, \mathcal{L}) := \frac{d}{d \epsilon} \pi_{\mathcal{L}}\left(\theta \Big| \mathbb{P}_{\epsilon, x^c} \right) \bigg|_{\epsilon=0}.
\end{align}
We say that $\pi_{\mathcal{L}}(\theta | \mathbb{P}_n)$ is \textit{globally-bias robust} if and only if $\sup_{\theta \in \Theta, x^c \in \mathcal{X}} \pif_{\Delta}(x^c, \theta, \mathbb{P}_n, \mathcal{L}) < \infty$, i.e., a GBI posterior is globally bias-robust if an infinitesimal contamination cannot impact it arbitrarily badly. We first show that the NLE posterior is \textit{not} globally-bias robust. 
\begin{theorem} \label{theorem:non-robustness-NLE}
    Assume $\max_{i}\mathbb{E}_{\theta \sim \pi_{\text{NLE}}(\cdot | x_{1:n}^o, \hat{\phi}_m)}[|\log q_{\hat{\phi}_m}(x^o_i | \theta)|] < \infty$. 
    Then, the NLE posterior is \textbf{not} globally-bias robust; i.e. 
    $\sup_{\theta \in \Theta, x^c \in \mathcal{X}} \pif_{\Delta}(x^c, \theta, \mathbb{P}_n, \mathcal{L}_{\text{NLE}}) = \infty$.
\end{theorem}
The proof is provided in Appendix \ref{appendix:proof_non_robust}. 
The assumption of Theorem \ref{theorem:non-robustness-NLE} is very mild and simply guarantees that $\text{PIF}_{\Delta}$ is well-defined and that the NLE posterior does not concentrate on parameter values that assign vanishing likelihood to the observed data. Despite ample experimental results demonstrating that NLE is not robust, this is the first theoretical justification of this claim.

\subsection{Robustness of Neural Score-matching Bayes}

In contrast with the previous result, the following theorem shows that NSM-Bayes is globally-bias robust under mild conditions. 
We denote by $L^q(\pi)$ the Lebesgue space of functions $f: \Theta \rightarrow \mathbb{R}$ such that $\int_{\Theta} | f(\theta) |^q \pi(\theta) d \theta < \infty$.
\begin{assumption} \label{assumption:general-pif}
    Suppose that $\sup_{\theta \in \Theta} \pi(\theta) < \infty$ and $\exists f: \Theta \rightarrow \mathbb{R}_{+}$, $g: \Theta \rightarrow \mathbb{R}_{+}$ 
     such that $f \in L^2(\pi), g \in L^1(\pi)$, $\sup_{\theta \in \Theta} \left(f(\theta)^2 + f(\theta) +  g(\theta) \right) \pi(\theta) < \infty$  and $ \forall \, x \in \mathcal{X}$:
    \begin{align*}
       & \left\| \nabla_x \log q_{\hat{\phi}_m}(x | \theta) \right\|_2 \leq f(\theta) \min\{w(x)^{-1},\|\nabla_xw(x)^2\|^{-1}_2\}, \\
        & \left| \text{Tr} \nabla^2_x \log q_{\hat{\phi}_m}(x | \theta)\right| \leq g(\theta) w(x)^{-2}.
    \end{align*}
\end{assumption}
\begin{theorem} \label{theorem:robustness}
    Suppose that \ref{assumption:weight}, \ref{assumption:regularity-q}  and \ref{assumption:general-pif} hold. Then, $\pi_{\text{NSM}}(\theta | x_{1:n}^o,\hat{\phi}_m)$ is globally-bias robust, i.e. $\sup_{\theta \in \Theta, x^c \in \mathcal{X}} \pif_{\Delta}(x^c, \theta, \mathbb{P}_n, \mathcal{L}_{\text{NSM}}) < \infty$.
\end{theorem}
The proof is in \Cref{appendix:proof_robust}. The intuition behind these assumptions is that the weight function $w$ and the prior $\pi$ must counter-balance any growth of $\| \nabla_x \log q_{\hat{\phi}_m}(x | \theta) \|_2$ and $| \text{Tr} \nabla^2_x \log q_{\hat{\phi}_m}(x | \theta)|$ towards infinity in $x$ and $\theta$, respectively. 
This is achieved by ensuring that the first two derivatives of $\log q_{\hat{\phi}_m}(x|\theta)$ are bounded pointwise by functions $f(\theta)$ and $g(\theta)$, whose growth is integrable under the prior, and offset by $w(x)$ in the data space. Hence,  any increase in likelihood sensitivity with respect to $x$ or $\theta$ is sufficiently controlled to keep the posterior influence function uniformly bounded.
These are very mild assumptions as $f,g$ can usually be found by inspecting the specific form of $q_\phi$, as will be shown below. 
A strength of this theorem is that it is applicable for a broad class of weights, including the IMQ weights used in this paper and a large part of the literature \citep{Altamirano2023,Altamirano2024,Reimann2024,Laplante2025,Rooijakkers2025}, but also for the squared-exponential weights in \citet{Altamirano2024} and the plateau-IMQ weights of \citet{Ezzerg2025}. However, when restricting ourselves to IMQ weights and sub-exponential priors with light tails, the assumptions can be significantly simplified. 
\renewcommand{\theassumption}{\ref{assumption:general-pif}$\mathbf{'}$}
\begin{assumption} \label{assumption:imq-pif}
Suppose that we take the IMQ weight $w_{\zeta}$ defined as in \Cref{eq:imq-weight} and a sub-exponential prior \citep[][Definition 2.7]{wainwright2019high} such that $\pi(\theta) \leq C \exp(- c \|\theta\|)$ for all $\| \theta \| \geq R$, and for some constants, $C,c,R >0$. 
Assume further that there exist constants $0 < K_1, K_2 < \infty$ and $k_1, k_2 \geq 0$
such that for all $x \in \mathcal{X}$:
\begin{align*}
    \left\| \nabla_x \log q_{\hat{\phi}_m}(x | \theta) \right\|_2 &\leq K_1 \left(1 + \| \theta \|^{k_1}\right) \left(1 + \| x \|_2^{\frac{2}{\zeta}}\right)
    \\
   \left| \text{Tr} \nabla^2_x \log q_{\hat{\phi}_m}(x | \theta)\right| &\leq 
   K_2 \left(1 + \| \theta \|^{k_2} \right)
   \left( 1 + \|x\|^{\frac{4}{\zeta}}_2 \right).
\end{align*}
\end{assumption}
\renewcommand{\theassumption}{A\arabic{assumption}}
\addtocounter{assumption}{-1}
\begin{corollary} \label{cor:imq}
    Suppose Assumptions \ref{assumption:regularity-q} and \ref{assumption:imq-pif} hold, then the statement of Theorem \ref{theorem:robustness} holds, i.e. $\sup_{\theta \in \Theta, x \in \mathcal{X}} \pif_{\Delta}(x^c, \theta, \mathbb{P}_n, \mathcal{L}_{\text{NSM}}) < \infty$.
\end{corollary}
A proof is provided in Appendix \ref{appendix:proof:cor}.
We note that these simplified assumptions are much easier to check; it suffices to check the growth of the score and trace terms in terms of powers of $\theta$ and $x$. The assumption on $\pi$ is very mild and includes most widely used  priors with sufficiently light tails such as Gaussian, Laplace, logistic and log-concave exponential-family priors.

\subsection{Robustness with Mixture Density Networks and Normalising Flows}
\label{sec:maf_description}

Next, we show that the assumptions on $q_\phi$ hold for MDNs \citep{bishop1994mixture} and MAFs \citep{Papamakarios2017_maf}. These results therefore guarantee robustness for NSM-Bayes and NSM-Bayes across all our experiments in \cref{sec:experiments}. 
Importantly, the assumptions concern only the fitted neural network architecture obtained in step (i): once $\hat{\phi}_m$ is specified, they can be verified directly before running NSM-Bayes in step (ii). 

\paragraph{Mixture density network.}
An MDN \citep{bishop1994mixture} represents a conditional density using a finite mixture of simple parametric distributions, typically Gaussians, with weights $\{\omega_k(\cdot; \phi)\}_{k=1}^K$ that sum to 1, mean vectors $\{\vartheta_k(\cdot; \phi)\}_{k=1}^K$, and covariance matrices $\{\Omega_k(\cdot;\phi)\}_{k=1}^K$, all parameterised by neural networks. Here $\phi$ is the vector of parameters of these networks. Thus, an MDN made of Gaussian distributions takes the form:
\begin{align} \label{eq:MDN}
    q_\phi(x | \theta ) = \sum_{k = 1}^{K} \omega_k \left(\theta; \phi\right) \mathcal{N} \left(x; \vartheta_k \left(\theta; \phi \right), \Omega_k \left(\theta; \phi \right)\right).
\end{align}
where $K$ is the number of mixture components. 
In the following result, we show that for certain values of the IMQ weight hyperparameter $\zeta$, MDNs satisfy Assumption \ref{assumption:imq-pif}, and hence the statement of Corollary \ref{cor:imq} holds in these cases. 
\begin{proposition}
\label{proposition:mixture_density_network}
    Suppose $q_{\hat{\phi}_m}$ is a Gaussian mixture-density network as in (\ref{eq:MDN}) such that:
\begin{itemize}
    \item The neural networks are continuous in $\theta$; i.e. $\omega_k (\cdot; \hat{\phi}_m), \vartheta_k(\cdot;\hat{\phi}_m), \Omega_k (\cdot;\hat{\phi}_m) \in C(\Theta)$.
    \item $\vartheta_k(\cdot;\hat{\phi}_m)$ has at most polynomial growth, i.e. $\exists r >0$, $C \geq 0$ such that $\max_k \| \vartheta_k(\theta,\hat{\phi}_m) \| \leq C( 1 + \| \theta\|^r)$.
    \item The covariance functions $\Omega_k(\theta, \hat{\phi}_m)$ satisfy uniform eigenvalue bounds, i.e. $0 < \lambda_{\min} \leq \lambda_{\text{min}}(\Omega_k(\theta; \hat{\phi}_m)) \leq \lambda_{\text{max}}(\Omega_k(\theta, \hat{\phi}_m)) \leq \lambda_{\text{max}} < \infty$ for all $k \in \{1,\ldots,K\}, \theta \in \Theta$.
\end{itemize} 
Then, Assumption \ref{assumption:imq-pif} is satisfied for any $0 < \zeta \leq 2$. 
\end{proposition}
The proof is provided in Appendix \ref{appendix:subsec:gaussian}. This means that under standard assumptions on the networks of an MDN, robustness is guaranteed for a large range of IMQ weights, including the commonly used choice of $\zeta = 2$. We now move on to showing a similar result for a class of normalising flows.
\paragraph{Masked autoregressive flows.}
An MAF \citep{Papamakarios2017_maf} is a conditional normalising flow defined by an invertible transformation $\varrho_\phi(\cdot; \theta):\mathbb{R}^{d_\mathcal{X}} \rightarrow \mathbb{R}^{d_\mathcal{X}}$ that maps a simple $d_\X$-dimensional base random variable $z \sim p_z(z) = \mathcal{N}(0, I_{d_\X})$ to our target variable $x$ conditioned on $\theta$ such that $x = \varrho_\phi(z; \theta)$ and $z = \varrho_\phi^{-1}(x; \theta)$. By change of variables, we can express $q_\phi(x|\theta)$ as $q_\phi(x|\theta) = p_z(\varrho_\phi^{-1}(x;\theta))| \det \partial \varrho_\phi^{-1}(x; \theta)/\partial x|$.
In practice, $\varrho_\phi^{-1}$ is composed of $L$ autoregressive layers $\varrho_\phi^{-1} = \varrho_L^{-1} \circ \dots \circ \varrho_1^{-1}$ with intermediate states $h_0 = x $, $h_l = \varrho_l^{-1}(h_{l-1}; \theta)$, and $z = h_L$:
\begin{equation*}
    \log q_\phi(x | \theta) =   \log p_z(z) + \sum_{l=1}^L \log \left| \det \frac{\partial \varrho_l^{-1}(h_{l-1}; \theta)}{\partial h_{l-1}} \right| ,
\end{equation*}
where $h_0 = x $, $h_l = \varrho_l^{-1}(h_{l-1}; \theta)$, $z = h_L$, and $\phi$ constitutes the combined parameters of $\varrho_1^{-1} , \dots, \varrho_L^{-1}$.
In each layer, the invertible map $\varrho_\phi^{-1}$ is an autoregressive affine transformation of the form (scaling and shifting operation):
\begin{equation} \label{eq:latent_maf}
    z_i = \left(\varrho_\phi^{-1}(x; \theta) \right)_i  = \frac{x_i - \mu_{\phi, i}(x_{<i}, \theta)}{\sigma_{\phi, i}(x_{<i}, \theta)} \quad i=1, \dots, d_\X,
\end{equation}
where $\mu_{\phi, i}$ and $\sigma_{\phi, i} > 0$ are outputs of masked multi-layer perceptrons (MLPs), parameterised by $\phi$, that ensure the autoregressive dependency structure, i.e. the $i^\textup{th}$ output only depends on $x_{<i} := (x_1, x_2, \ldots, x_{i-1})^\top \in \mathbb{R}^{i-1}$, $i \in \{1, \ldots, d_\X\}$. The conditioning on $\theta$ is included in every layer. The corresponding Jacobian is a lower triangular matrix, hence $\log\left|\det \partial z/\partial x\right|=-\sum_{i=1}^{d_\X}\log \sigma_{\phi,i}(x_{<i},\theta)$, and therefore the log-density becomes:
\begin{equation} \label{eq:MAF}
    \log q_\phi(x | \theta) =  \log p_z(z) - \sum_{i=1}^{d_\X}\log \sigma_{\phi,i}(x_{<i},\theta),
\end{equation}
where $z_i$ (with components $z_i$) given by \Cref{eq:latent_maf}. 
Similarly to MDNs, when $q_{\phi}$ is an MAF, robustness holds for certain values of $\zeta$ under standard network conditions.
\begin{proposition}\label{thm:masked_autoregressive_flow}
    Suppose $q_{\hat{\phi}_m}$ is a masked autoregressive flow as in \Cref{eq:MAF} such that 
    \begin{itemize}
        \item The networks computing $\mu_{\hat{\phi}_m,i}$ and $\sigma_{\hat{\phi}_m,i}$ are MLPs built from affine layers with tanh activation functions.
        \item An additional softplus activation is applied to the final output of $\sigma_{\hat{\phi}_m,i}$.
    \end{itemize}
    Then Assumption \ref{assumption:imq-pif} is satisfied for any $0 < \zeta \leq 1$.
\end{proposition}
The proof is provided in Appendix \ref{appendix:proof:maf}.
Note that the network assumptions in Proposition \ref{thm:masked_autoregressive_flow} follow the default implementation of MAF provided by the SBI library \citep{Tejero-Cantero2020}.
Most notably, this result guarantees that the examples considered in \Cref{sec:experiments} will be robust. 
Interestingly, it shows that the choice of $\zeta=2$, which is prevalent in the literature on SM-Bayes, may not be a good one for NSM-Bayes based on normalising flows.

\subsection{Robustness of Conjugate Neural Score-matching Bayes}

We now move to NSM-Bayes-conj, where we can build on \Cref{theorem:robustness}, but also study robustness under an alternative form of PIF suggested by \citet{Altamirano2024}. 
This alternative is preferable, as it is divergence-based and offers a more interpretable measure of robustness. It is not as widely used as it is often challenging to compute, but it is tractable for the special case of NSM-Bayes-conj.
Consider the case where a single observation $x_j^o$, for some $j \in \{1, \dots, n\}$, is replaced by an arbitrary contamination datum $x_j^c$, and denote the set of observed data as $x^o_{1:n}$ and the contamination dataset as $x_{1:n}^{c} := (x^o_{1:n} \backslash  x_j^o) \cup x_j^c$. The PIF of \citet{Altamirano2024} is defined by examining the influence of this contamination, measured using the Kullback-Leibler divergence between the posterior based on the original dataset $x^o_{1:n}$ and that based on the contaminated dataset $x^c_{1:n}$:
\begin{align*}
     \pif_{\KL}(x_j^c, \mathbb{P}_n, \mathcal{L}) :=  \KL\left(\pi_{\mathcal{L}}\left(\theta | x^o_{1:n}\right) \,\big\|\, \pi_{\mathcal{L}}\left(\theta | x^c_{1:n} \right)\right).
\end{align*}
Similarly to before, $\pi_{\mathcal{L}}$ is called robust if $\pif_{\KL}$ is bounded as a function of $x_j^c$. For any $n \times m$ matrix $A$, we write $\| A \|_{F} := (\sum_{i=1}^m \sum_{j=1}^n | A_{ij} |^2)^{\frac{1}{2}}$ for the Frobenius norm.
\begin{assumption} \label{assumption:kl-pif}
    Suppose $\exists C, C^\prime < \infty$ such that $ \forall x \in \mathcal{X}$:
    \begin{align*}
         & \max\left\{\left(\left\|\nabla_x T_{\hat{\phi}_m}(x)\right\|_F \left\| \nabla_x b_{\hat{\phi}_m}(x) \right\|_2 \right)^{\frac{1}{2}},   \left\| \nabla^2_x T_{\hat{\phi}_m}(x) \right\|_F^{\frac{1}{2}}, \left\| \nabla_x T_{\hat{\phi}_m}(x)\right\|_F \right\}   \leq C w(x)^{-1}, \\
         & \left\| \nabla_x T_{\hat{\phi}_m}(x) \right\|_F  \leq C^\prime  \left\| \nabla_x w(x)^2 \right\|_2^{-1}.
    \end{align*}
\end{assumption}
\begin{theorem} \label{theorem:robustness-exp}
Suppose \ref{assumption:weight}, \ref{assumption:exptheta} and \ref{assumption:kl-pif} hold.
Then, the NSM-Bayes-conj is outlier-robust; i.e.  $\sup_{x_j^c \in \mathbb{R}^{d_{\mathcal{X}}}}  \pif_{\KL}(x_j^c, \mathbb{P}_n, \mathcal{L}_{\text{NSM}})  < \infty$.
\end{theorem}
See \Cref{appendix:proof_robust_exp} for the proof. 
Assumption \ref{assumption:kl-pif} ensures that the first- and second-order derivatives of $T_{\hat{\phi}_m}$ and  $b_{\hat{\phi}_m}$ grow at most at a rate that is offset by $w$ and $\nabla w$, so that extreme observations cannot induce unbounded posterior influence.
Note that if Assumption \ref{assumption:kl-pif} is satisfied, then under mild additional assumptions, Assumption \ref{assumption:general-pif} is satisfied as well. Consequently, robustness is guaranteed under both PIF definitions (see Appendix \ref{appendix:assumptions-connection} for a proof). This time, we do not have additional assumptions on the prior since this must be Gaussian to ensure conjugacy. Here, the constants $C$ and  $C^\prime$ act very similarly to $f, g$ in Assumptions \ref{assumption:general-pif} and \ref{assumption:imq-pif}, with the difference being that they do not depend on $\theta$. 

We emphasise that a key strength of all robustness results in Section \ref{sec:theory} is that they depend only on Assumptions about the fitted density $q_{\hat{\phi}_m}$ rather than the entire model class $\{q_{\phi}: \phi \in \Phi\}$. Consequently, step (i) can be carried out independently to obtain $\hat{\phi}_m$, after which the weight $w$ can be selected such that the required Assumptions are satisfied.

\section{Experiments}
\label{sec:experiments}

We demonstrate the performance of our methods against NLE and other robust SBI baselines. Given the significant interest in the problem, an exhaustive comparison is infeasible. We therefore benchmark against five competitors representative of the literature: (i) the robust sequential NLE (RSNLE) method of \citet{Kelly2024}, (ii) the MMD posterior bootstrap (NPL-MMD) method of \citet{Dellaporta2022}, (iii) the NPE with robust statistics (NPE-RS) method of \citet{Huang2023}, (iv) the amortised cost estimation (ACE) method of \citet{Gao2023}, and (v) the GBI scoring rule (GBI-SR) method of \citet{Pacchiardi2024}. We exclude comparison with methods that require additional data for robustness as they are not as widely applicable. 

We use two metrics to evaluate performance. The first is the MMD$^2$ \citep{Gretton2006} between posterior samples and a reference NLE posterior obtained from uncorrupted data (denoted as $\text{MMD}^2_{\text{ref}}$). The second is the mean squared error $\text{MSE} = \mathbb{E}[\Vert \theta - \theta^\star \Vert_2^2]$ estimated using posterior samples, which measures whether the posteriors are centred around $\theta^\star$. Note that MSE can be computed in closed-form for NSM-Bayes-conj. We also report the time to perform inference on a single dataset. For partially amortised methods where the density estimator need not be re-trained for new observations, training and inference times are reported separately. For NSM-Bayes, NSM-Bayes-conj, GBI-SR and ACE, inference time also includes tuning the learning rate $\beta$. As ACE and GBI-SR do not provide a method to tune $\beta$, we set it using the same approach as NSM-Bayes; see \Cref{app:setting-learning-rate} for further details. 

For a fair comparison, we fix the simulation budget $m$ and the number of posterior samples for all methods. 
We use the same default conditional density estimator from the sbi library \citep{Tejero-Cantero2020} for both NLE and NSM-Bayes, which is an MAF for the g-and-k (\Cref{sec:gnk}) and SIR experiments (\Cref{sec:sir}), and an MDN for the radio propagation experiment (\Cref{sec:radio}). We set $\zeta=1$ in \Cref{eq:imq-weight} to ensure that our theory is satisfied for both MDNs and MAFs.
A slice sampler is used to obtain 500 samples (with 500 warm-up steps) from the NLE and the NSM-Bayes posterior in the g-and-k and SIR experiments. For the radio propagation experiment, we sample 1000 posterior samples. For simplicity, we compare all the runtimes on CPU (Apple M4, 16GB), but some methods may benefit from parallelisation (e.g. NPL-MMD). Further implementation details and explanation of the baselines are in \Cref{app:experimental-details}. The code to reproduce the experiments is available at \url{https://github.com/bharti-ayush/nsm-bayes}.

\subsection{Benchmarking on the g-and-k distribution}
\label{sec:gnk}

Our first example is the g-and-k \citep{Prangle2020}: a  flexible parametric family of univariate distributions which does not have a closed-form likelihood but is easy to simulate from. This is a popular example in the SBI literature \citep{Fearnhead2012, Price2018, wang2022approximate, Bharti22a, chen2023learning}, which has been used to benchmark robustness \citep{Dellaporta2022}.
The distribution has four parameters $(\theta_1, \theta_2, \theta_3, \theta_4)$, and sampling can be done by simulating $u \sim \mathcal{N}(0,1)$ and substituting it in the following generator:
\begin{align*}
    G_\theta(u) = \theta_1 + \theta_2 \left(1+0.8 \left(\frac{1-\exp(-\theta_3 u)}{1+\exp(-\theta_3 u)} \right)\right)\left(1+u^2 \right)^{\theta_4} u, \quad u \sim \mathcal{N}(0,1).
\end{align*}
For stability, we log-transform the scale and kurtosis parameters $(\theta_2, \theta_4)$. We then place a Gaussian prior on $(\theta_1, \log \theta_2, \theta_3, \log \theta_4)$ with mean $[0.0, 0.7, 0.0, -1.5]^\top$ and a diagonal covariance matrix where the diagonal entries are $[5.0, 0.5, 4.0, 0.25]^\top$.

\begin{table*}
\centering
\caption{\textbf{Benchmarking robust SBI methods on g-and-k distribution}. Standard deviation for the metrics across 20 runs are reported in the parenthesis.}
\label{tab:gnk}
\scalebox{0.86}{
\begin{tabular}{rccccllc}
\toprule
Method 
& $\text{MMD}^2_{\text{ref}}$
& MSE
& \begin{tabular}[c]{@{}c@{}}Empirical\\ {Coverage}\end{tabular}
& \begin{tabular}[c]{@{}c@{}}Training\\ {time [}s{]}\end{tabular}
& \begin{tabular}[c]{@{}c@{}}Inference\\ {time [}s{]}\end{tabular}
& Amortised
& \begin{tabular}[c]{@{}c@{}}Robustness\\ {guarantee}\end{tabular} \\
\midrule
NLE  
& 0.54 \textcolor{gray}{(0.03)} 
& 15.0 \textcolor{gray}{(1.8)}  
& 0\%
& 93 \textcolor{gray}{(37)} 
& 36 \textcolor{gray}{(19)} 
& \checkmark (Partially) 
& \xmark \\ 
\hdashline[1pt/2pt]
NPL-MMD  
& 0.69 \textcolor{gray}{(0.21)} 
& 37.2 \textcolor{gray}{(9.4)} 
& 30\%
& - 
& 0.6 \textcolor{gray}{(0.1)} 
& \xmark 
& \checkmark \\
RSNLE  
& 0.49 \textcolor{gray}{(0.08)} 
& 3.8 \textcolor{gray}{(1.6)} 
& 0\%
& - 
& 38 \textcolor{gray}{(4)}
& \xmark  
& \xmark \\ 
NPE-RS  
& 0.36 \textcolor{gray}{(0.15)} 
& 8.0 \textcolor{gray}{(4.9)} 
& 15\%
& - 
& 41 \textcolor{gray}{(19)}
& \xmark  
& \xmark \\
ACE  
& 0.24 \textcolor{gray}{(0.15)} 
& 4.6 \textcolor{gray}{(1.5)} 
& 85\%
& 376 \textcolor{gray}{(80)}
& 390 \textcolor{gray}{(5)}
& \checkmark (Partially)   
& \xmark \\
GBI-SR  
& 0.24 \textcolor{gray}{(0.15)} 
& 4.1 \textcolor{gray}{(5.4)} 
& 80\%
& - 
& 415 \textcolor{gray}{(8)}
& \xmark  
& \checkmark \\
\hdashline[1pt/2pt]
NSM-Bayes  
& 0.13 \textcolor{gray}{(0.09)} 
& 5.5 \textcolor{gray}{(2.4)}
& 100\%
& 93 \textcolor{gray}{(37)} 
& 163 \textcolor{gray}{(2)}
& \checkmark (Partially) 
& \checkmark \\
NSM-Bayes-conj  
& 0.20 \textcolor{gray}{(0.14)} 
& 6.1 \textcolor{gray}{(2.9)} 
& 100\%
& 176 \textcolor{gray}{(53)} 
& 5.1 \textcolor{gray}{(0.1)}
& \checkmark  (Fully)
&  \checkmark \\
\bottomrule
\end{tabular}}
\end{table*}

We generate observed data $x^o_{1:n}$ with one-sided outliers by sampling $n=100$ points from a Huber contamination model where $90\%$ of observations are drawn from some g-and-k distribution $\P_{\theta^\star}$ with $\theta^\star = [1, 0.5, 1, -1]^\top$, and the remainder from some corrupted distribution $\mathbb{Q}= \mathbb{P}_{\theta^\star}-50$. We fix the simulation budget to be $m=10^5$ samples and repeat the experiment $20$ times. Results are reported in \Cref{tab:gnk} and the marginal posterior densities are in \Cref{fig:gnk_posteriors}. We also report the empirical coverage for each method, which measures how often $\theta^\star$ was included in the 95\% credible region of their posterior. 

NLE is most affected by the one-sided contamination in the scale parameter $\theta_2$, leading to inflated scale and reduced skewness $\theta_3$ estimates. Correspondingly, it achieves $0\%$ empirical coverage and exhibits large MMD and MSE values. NPE-RS is fast at inference but fails to provide robust posteriors, with low coverage, and poor performance. Similarly, RSNLE is fast but achieves large MMD and 0\% coverage, despite having the lowest MSE. NPL-MMD is the fastest method but performs poorly in this setting. This is because it places a non-informative prior on $\P_0$ rather than the informative prior all other methods use for $\theta$, and therefore would require a much larger simulation budget to improve performance; as shown in \Cref{app:add_exp_baselines}. 
ACE attains a relatively good coverage and competitive MMD/MSE values, but has large training and inference times. 
GBI-SR also achieves competitive performance and maintains good coverage, but its performance is more variable across runs, particularly for $\theta_3$, and it has the largest inference time. 

In contrast with these benchmark methods, NSM-Bayes and NSM-Bayes-conj achieve 100\% empirical coverage and the lowest $\text{MMD}^2_{\text{ref}}$ values. Both methods also perform well in MSE, and NSM-Bayes-conj takes just a few seconds for inference, which is primarily spent on calibrating $\beta$. Training NSM-Bayes-conj is slower than NLE since it minimises an (unweighted) score-matching objective, and NSM-Bayes is typically $2-3 \times$ slower than NLE due to the score and Hessian computation required for $\mathcal{L}_{\text{NSM}}$.
Overall, NSM-Bayes and NSM-Bayes-conj are the only methods in this benchmark that are simultaneously amortised and provably robust, while achieving strong empirical robustness, reasonable coverage, and competitive inference times. 

\begin{figure}
    \centering
    \includegraphics[width=\linewidth]{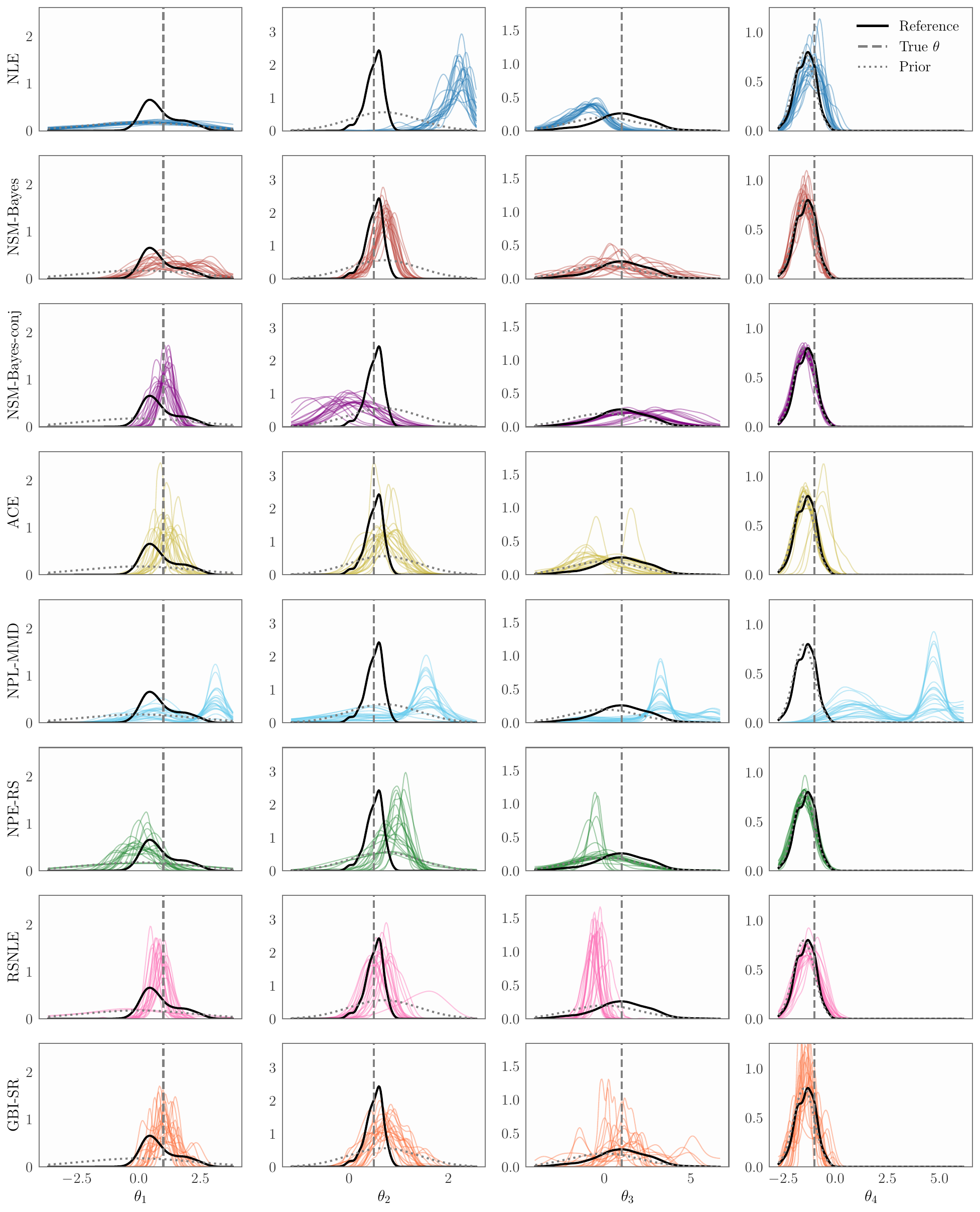}
    \caption{\textbf{The g-and-k distribution.} Kernel density estimates of the marginal posterior distributions from NLE (\legendbox{nlecolor}), NSM-Bayes (\legendbox{gencolor}), NSM-Bayes-conj (\legendbox{c1color}), ACE (\legendbox{acecolor}), NPL-MMD (\legendbox{nplmmdcolor}), NPE-RS (\legendbox{nperscolor}), RSNL (\legendbox{rsnlecolor}), and GBI-SR (\legendbox{gbisrcolor}) from 500 samples, along with the reference NLE posterior (\legendbox{black}) given uncorrupted data. The dashed and dotted gray lines denote the true parameter value and the prior distribution, respectively.}
    \label{fig:gnk_posteriors}
\end{figure}

\subsection{A Susceptible-Infectious-Recovered model with under-reporting of cases}
\label{sec:sir}

Next, we consider a stochastic, discrete-time Susceptible-Infectious-Recovered (SIR) model \citep{tuckwell2007some, allen2017primer} which is a set of stochastic recurrence relations that determine how many people in a given population $N$ are susceptible, infectious, and recovered over a given time period.
Let a population of size $N$ be divided into three compartments: Susceptible ($S_t$), Infectious ($I_t$), and Recovered ($R_t$), $t=0, \dots, T$. Given a time step $\Delta t$, the dynamics are governed by a set of stochastic recurrence relations. The transition of individuals between compartments at each time step $t \rightarrow t + \Delta t$ is modelled using Binomial distributions, which account for the discrete nature of the population and the inherent randomness of transmission and recovery processes: $\Delta I_t \sim \text{Binomial}\big(S_t, 1 - \exp(- \theta_1 \frac{I_t}{N}\Delta t)\big)$, and $\Delta R_t \sim \text{Binomial}\big(I_t, 1 - \exp(- \theta_2\Delta t)\big)$,
where $\theta_1$ is the transmission rate and $\theta_2$ is the recovery rate. The binomial transitions arise from assuming independent infection and recovery events over the interval $[t, t + \Delta t]$. The state variables are then updated as: $S_{t+1} = S_t - \Delta I_t$, $I_{t+1} = I_t + \Delta I_t - \Delta R_t$, $R_{t+1} = R_t + \Delta R_t$.
This formulation corresponds to a discrete-time approximation of the continuous-time SIR process, with binomial transitions capturing demographic stochasticity.
The system is initialized with $\theta_4$ infected individuals and $S_0 = N - \theta_4$ susceptible individuals, where $N$ is assumed known. 

Let $y_t$ denote the observed number of reported new infections at time $t$. The expected reported cases are $\theta_3 \cdot \Delta I_t$, where $\theta_3 \in [0,1]$ is the reporting rate. We consider the Poisson observation model for $y_t$:
$y_t \sim \text{Poisson}(\theta_3 \Delta I_t)$.
As only partial observations are available, inference relies on low-dimensional summary statistics rather than the full latent trajectory.
From a single observed epidemic trajectory $y_{1:T}$, we compute three summary statistics commonly used to characterise epidemic severity \citep{house2013big,meyer2025predictability}: (i) the attack rate: $x_1 = \frac{1}{N} \sum_{t=1}^T y_t$, (ii) the normalised peak timing: $x_2 = \frac{1}{T-1} \arg \max_{t \in \{1, \dots, T\}} y_t$, and (iii) the normalised peak height: $x_3 = \frac{1}{N} \max_{t} y_t$. 
Thus, $x = [x_1, x_2, x_3]^\top \in \mathbb{R}^3$ is the summary statistics vector. We transform the parameters to the unconstrained real space via the transformation: $\theta = [\log \theta_1, \log \theta_2, \text{logit}\,\theta_3, \log \theta_4]^\top$,
and place Gaussian prior with mean $[\log 0.5, \log 0.2, \text{logit} 0.5, \log 20]^\top$ and a diagonal covariance matrix with diagonal elements $[0.5^2, 0.5^2, 1.0, 0.7^2]^\top$. We set $\theta^\star = [\log 0.6, \log 0.15, \text{logit}\, 0.6, \log 20 ]^\top$, $N=1000$ and $T=150$.

Outlier robustness is particularly relevant in epidemiology due to the common practical problem of undercounting infectious cases, which leads to one-sided contamination in the data \citep{gibbons2014measuring, li2020substantial}.
To model the systematic under-reporting or undercounting of cases, we introduce a one-sided contamination mechanism on the observed epidemic trajectories. 
Let $y^{(i)} = (y^{(i)}_1, \ldots, y^{(i)}_T)$ denote the reported new infections for the $i^{\textup{th}}$ trajectory, $i=1, \ldots, n$, generated from the model with true parameter $\theta^\star$. 
We assume that a fraction $\epsilon \in (0,1)$ of observed trajectories are affected by undercounting. For each trajectory $i$, we draw an independent contamination random variable $U_i \sim \text{Bernoulli}(\epsilon)$ which indicates whether trajectory $i$ is undercounted or not. If $U_i = 1$, the entire trajectory is subject to one-sided undercounting, modelled via binomial thinning:
\begin{align*}
    \tilde{y}^{(i)}_t \vert y^{(i)}_t \sim \text{Binomial} \big( y^{(i)}_t, r \big), \quad t=1, \ldots, T,
\end{align*}
\begin{wrapfigure}[]{r}{0.4\textwidth}
    \centering
    \includegraphics[width=\linewidth]{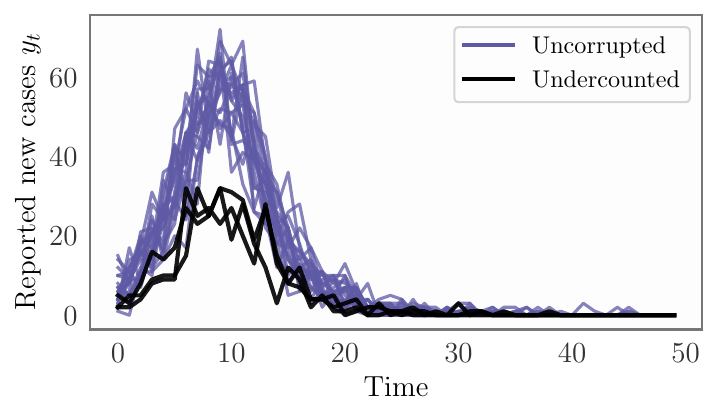}
    \vspace{-2ex}
    \caption{Observed uncorrupted (\legendbox{datacolor}) and undercounted (\legendbox{black}) trajectories of incidence from the SIR model with $\epsilon=5\%$ and $50\%$ retention probability.}
    \label{fig:sir_trajectories}
\end{wrapfigure}
where $r \in (0,1)$ is a \emph{retention probability}. Smaller values of $r$ correspond to more severe undercounting. Thus, the marginal distribution of observed trajectories follows the mixture model $\tilde{y}^{(i)} \sim (1- \epsilon) \P_{\theta^\star} + \epsilon \P_{\theta^\star, r}^{\text{under}}$,
where $\P_{\theta^\star, r}^{\text{under}}$ denotes the distribution induced by the binomial thinning of trajectories drawn from $\P_{\theta^\star}$. 
\Cref{fig:sir_trajectories} shows an example of an observed dataset with both uncorrupted and undercounted trajectories.
Out of $n=100$ observed time-series recordings, we set $\epsilon=5\%$ of them to be undercounted by $r=50\%$ to simulate outliers. We also simulate \textbf{heavy-tailed contamination} by adding Cauchy distributed noise to $\epsilon=10\%$ of the data samples. 

\begin{figure}
    \centering
    \subfigure[Undercounting contamination]
    {\includegraphics[width=0.49\columnwidth]{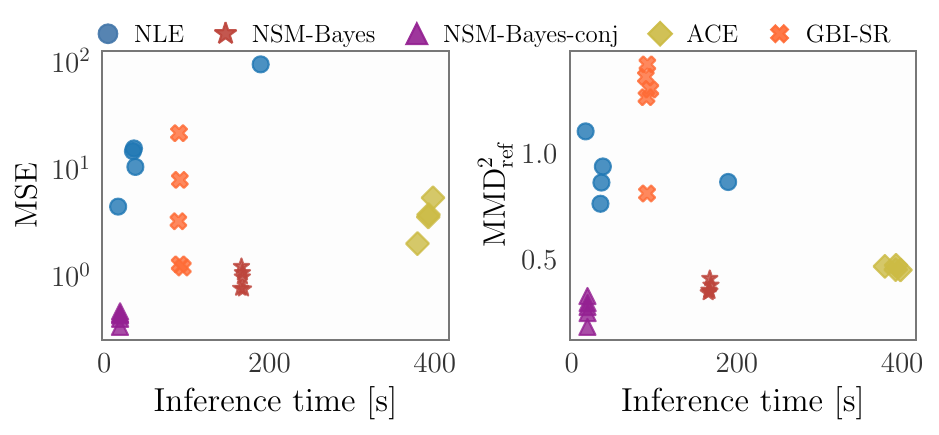}}
    \subfigure[Heavy-tailed contamination]
    {\includegraphics[width=0.49\linewidth]{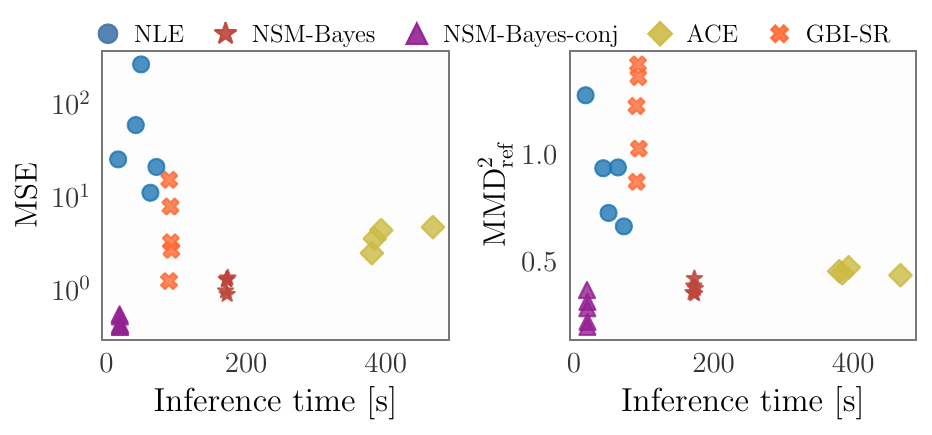}}
    \caption{\textbf{SIR model}. Performance vs. posterior inference time for NLE (\legendbox{nlecolor}), NSM-Bayes (\legendbox{gencolor}), NSM-Bayes-conj (\legendbox{c1color}), GBI-SR (\legendbox{gbisrcolor}), and ACE (\legendbox{acecolor}) under both undercounting and heavy-tailed contamination. In both cases, NSM-Bayes-conj performs the best whilst incurring the least time.}
    \label{fig:sir_plot}
\end{figure}

In \Cref{fig:sir_plot}, we plot both the metrics against inference time for NLE, NSM-Bayes, and NSM-Bayes-conj using just $m=5\times 10^4$ samples to measure their sample-efficiency. We also include GBI-SR and ACE: the best-performing baselines from \Cref{tab:gnk}. However, we cannot use the stochastic gradient version of GBI-SR (which we used in \Cref{sec:gnk}) as the SIR simulator is non-differentiable. We therefore use pseudo-marginal MCMC \citep{Andrieu2009} version of GBI-SR instead, which performs poorly here, primarily because the simulation budget is not large enough. We actually find that the method requires orders of magnitude more simulations to become competitive; see \Cref{app:add_exp_baselines}. ACE performs better than NLE and GBI-SR, but requires large inference time and also fails to include the true parameter in its posteriors, as seen from the posterior plots in \Cref{app:additional_experiments}. 
NSM-Bayes-conj performs the best whilst requiring the least inference time, followed by NSM-Bayes. This shows that our methods can provide robust inferences whilst being sample-efficient, unlike methods like GBI-SR which require advanced MCMC samplers.

\subsection{A radio propagation model with corrupted receiver}
\label{sec:radio}

Finally, we consider a model that simulates the behaviour of radio signals at a receiving antenna \citep{Turin1972, Pedersen2019}, and is used to design and test wireless communication systems in different propagation environments.
The model simulates the scenario where channel transfer function data is being measured at $K$ equidistant points within a frequency bandwidth $B$, resulting in a frequency separation of $\Delta f = B/(K-1)$. Assuming additive white Gaussian measurement noise, the measured data $Y_k$ at the $k^\textup{th}$ point can be written as $Y_k = H_k + Z_k$, $k=0,1,\dots,K-1$,
where $H_k$ is the transfer function at the $k^\textup{th}$ point and $Z_k \sim \mathcal{CN}(0, \sigma_Z^2)$ is the zero-mean complex circular symmetric Gaussian noise. The transfer function is modelled as 
$H_k = \sum_l \alpha_l \exp(-j2\pi \Delta f k \tau_l)$, where $\tau_l$ is the time-delay and $\alpha_l$ is the complex-valued gain of the $l^\textup{th}$ signal component. The arrival time of the delays is modelled as a one-dimensional homogeneous Poisson point process $\tau_l \sim \mathrm{PPP}(\mathbb{R_+}, \lambda)$, with $\lambda>0$ being the arrival rate parameter. The complex-valued gains $\alpha_l$ are conditioned on the delays, and are modelled as iid zero-mean complex Gaussian random variables with conditional variance $\mathbb{E}[\vert \alpha_l \vert^2 | \tau_l] = G_0 \exp(-\tau_l/T)/ \lambda$. Here, $G_0$ is the parameter that governs the starting value of the variance, and the parameter $T$ controls its rate of decay. Since all the parameters are non-negative, we parameterise the model using their log values, i.e., the parameter vector becomes $\theta = [\log G_0, \log T, \log \nu, \log \sigma^2_Z]^\top$.
We convert the frequency-domain signal into a time-domain signal $y(t)$ by taking the inverse Fourier transform as $y(t) = \frac{1}{K} \sum_{k=0}^{K-1} Y_i \exp(j2\pi k \Delta f t)$.
We then summarise the time-domain signal using their log temporal moments \citep{Bharti2021_IEEE}, which are computed as $x_j = \log \int y(t) t^j dt$, $j=0,1, \dots, J$. We set $J=2$, $B=4$ GHz, and $K=801$ similar to \citet{Huang2023}, and place independent Gaussian priors on the log of all the parameters with mean vector $[-19.0, -19.0, 22.0, -22.0]^\top$ and an identity covariance matrix. 

\begin{wrapfigure}[]{r}{0.45\textwidth}
    \centering
    \vspace{-3ex}
    \includegraphics[width=\linewidth]{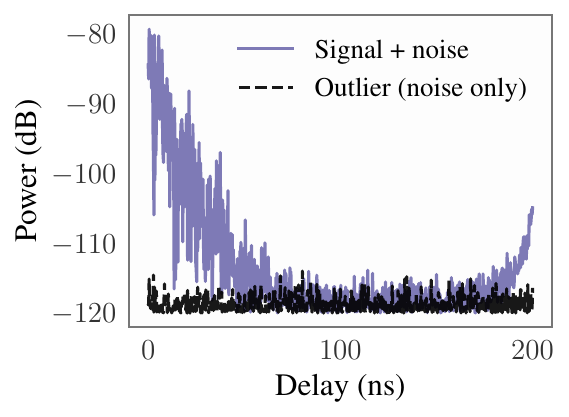}
    \vspace{-4ex}
    \caption{A valid (\legendbox{datacolor}) and an outlying (\legendbox{black}) time series  from the radio propagation model.}
    \label{fig:turin_data}
    \vspace{-2ex}
\end{wrapfigure}

The need for outlier-robust methods  has long been recognised in this field \citep{Savic2015}. Typical causes of outliers are faulty or unstable measurements \citep{quimby2020channel} or destructive interference \citep{schneckenburger2018characterization}. We simulate a scenario where we observe $n=100$ time-series data of which $\epsilon=10\%$ record only noise due to faulty antennas. Example of an uncorrupted and an outlying time-series is in \Cref{fig:turin_data}. 

We simulate $m= 50,000$ samples from the Gaussian prior, and use an MDN for NLE and NSM-Bayes in this example. The pairwise bivariate scatter plots of the posterior samples from NLE, NSM-Bayes, and ACE is shown in \Cref{fig:turin_scatter}, along with the bivariate posterior densities from NSM-Bayes-conj. Again, this simulator is non-differentiable, and the pseudo-marginal MCMC version of GBI-SR does not work well for this problem so we exclude it. We see that the NLE posterior is severely biased, grossly underestimating $\theta_3$ under the presence of outliers. In contrast, the posteriors for NSM-Bayes and NSM-Bayes-conj are robust and centred around the true parameter value. The posterior from ACE is also robust for $\theta_3$, however, it yields a diffuse, and slightly biased marginal posterior for $\theta_2$. The inference times are reported in \Cref{app:additional_experiments}. Thus, NSM-Bayes and NSM-Bayes-conj yield robust posteriors under outlier contamination even for non-differentiable simulators whilst being well-calibrated, demonstrating that our approach can deliver reliable uncertainty quantification.

\begin{figure}
    \centering
    \includegraphics[width=0.95\linewidth]{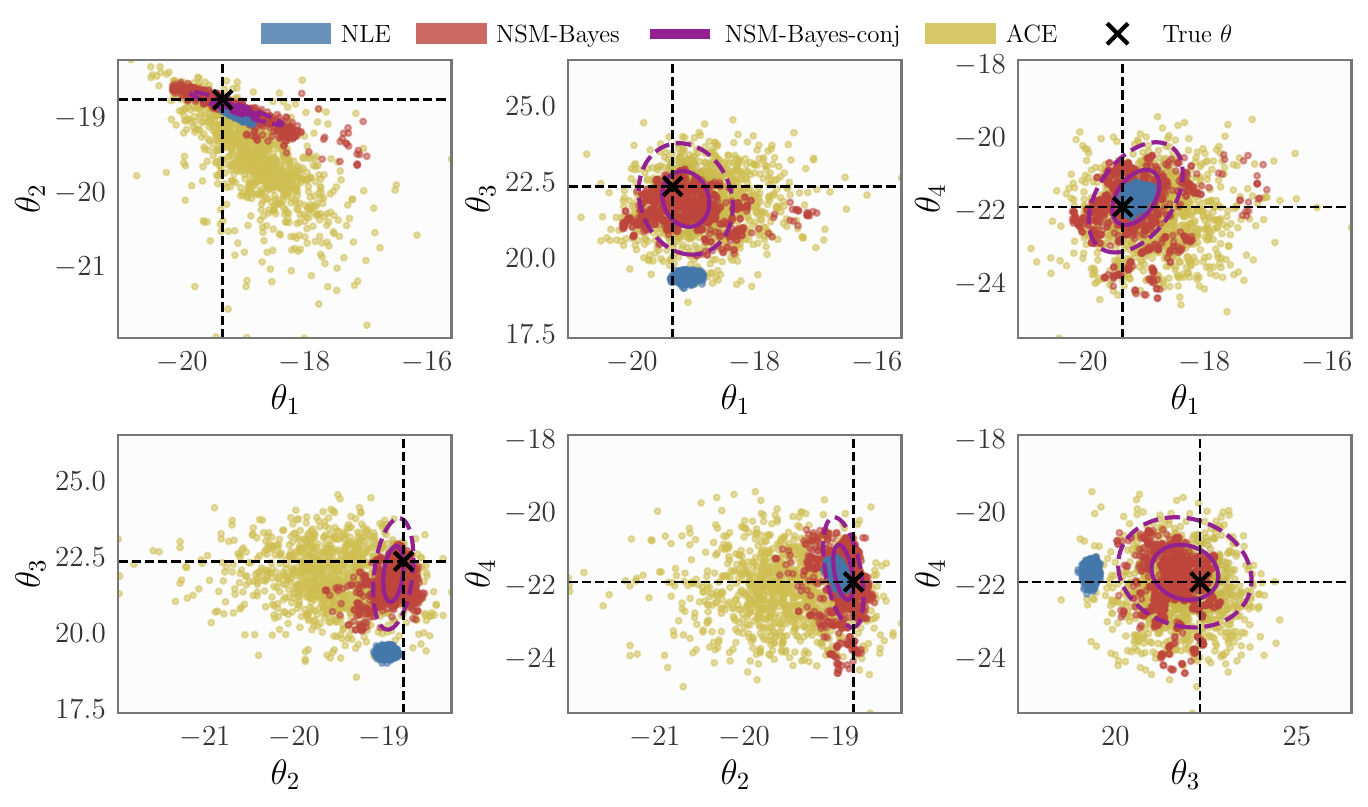}
    \caption{\textbf{Radio propagation model.} Pairwise scatter plots of approximate posterior samples from NLE (\legendbox{nlecolor}), NSM-Bayes (\legendbox{gencolor}), and ACE (\legendbox{acecolor}) for the radio propagation model under 10\% outliers, along with the NSM-Bayes-conj (\legendbox{c1color}) posterior density. The NLE posterior is severely biased, while the NSM-Bayes and the NSM-Bayes-conj posteriors are around the true parameter, and therefore robust to outliers. ACE yields biased posterior for $\theta_2$, but works well for $\theta_3$ and $\theta_4$. ACE also yields fairly robust posterior, but with increased uncertainty, especially for $\theta_2$.}
    \label{fig:turin_scatter}
\end{figure}

\section{Conclusion}
\label{sec:conclusion}

We proposed a novel SBI method based on a neural approximation to SM-Bayes which is provably robust to outliers and applies to arbitrary neural likelihood surrogates and prior distributions. Restricting this flexibility to a specific form of energy-based model with a Gaussian prior yields conjugate posterior updates, enabling fully amortised inference. 

As with all score-based methods, our approach is ill-suited to inference problems with highly multimodal, well-separated likelihoods \citep{Wenliang2020,Zhang2022}. In such cases, the GBI update may miss the full posterior structure, and can also affect training of the neural surrogate in the conjugate setting. Moreover, our formulation assumes continuous data and does not directly handle discrete or structured data. Extensions of score-matching to such settings exist \citep{Hyvarinen2007, Lyu2009,Matsubara2022}, including some permitting conjugacy \citep{laplante2025conjugate}. Integrating these within SBI remains future work. Finally, we focused on outlier robustness, whereas other forms of misspecification, such as distributional shifts or misspecified dependence structures, may also be relevant and warrant investigation.

\section*{Acknowledgements}
The authors thank Conor Hassan for tips on implementation and William Laplante for helpful feedback on a draft.
AB and YH were supported by the Research Council of Finland grant no. 362534. CD and FXB were supported by the EPSRC grant [EP/Y022300/1] and CD was additionally supported by the project `Reliable insights from scientific simulations' funded from the EPSRC grant [UKRI3030].

\bibliography{bibliography}
\bibliographystyle{apalike}

\appendix

\newpage
{
\begin{center}
\Large
    \textbf{Supplementary Materials}
\end{center}
}

\Cref{app:proofs} contains the proofs of our theoretical results, and \Cref{app:experimental-details} consists of the implementation details and additional results for the experiments in \Cref{sec:experiments}.

\section{Proofs}
\label{app:proofs}

We now present our proofs. Section \ref{appendix:proof_conjugacy_prop} proves conjugacy for NSM-Bayes-conj. Section \ref{appendix:proof_non_robust} proves that NLE is not globally-bias robust, whereas Section \ref{appendix:proof_robust} proves robustness of NSM-Bayes in the general setting. Finally, \ref{app:sec:imq} contains proofs of all results associated with the IMQ weight and sub-exponential or Gaussian priors.

\subsection{Proof of \Cref{prop:conjugacy}}
\label{appendix:proof_conjugacy_prop}

\begin{proof}
We aim to derive the conjugate form of $\pi_{\text{NSM}}(\theta | x_{1:n}^o)$ given in \Cref{prop:conjugacy}. Recall that:
\begin{talign*}
   \mathcal{L}_{\text{NSM}}\left(\theta;x_{1:n}^o, \hat \phi_m \right) 
     := \frac{1}{n} \sum_{i=1}^n \left\|W(x_i^o)^\top \nabla_x \log q_{\hat{\phi}_m}(x_i^o|\theta) \right\|_2^2  +  2 \nabla_x \cdot W(x_i^o)W(x_i^o)^\top \nabla_x \log q_{\hat{\phi}_m}(x_i^o|\theta).
\end{talign*}
Substituting $W(x) := w(x) I_{d_\X}$ and setting $q_{\phi}(x|\theta) \propto \exp(T_{\phi}(x)^\top \theta + b_{\phi}(x))$, we get
\begin{align*}
    \mathcal{L}_{\text{NSM}}\left(\theta;x_{1:n}^o, \hat{\phi}_m \right) &= \frac{1}{n} \sum_{i=1}^n \left\Vert w(x_i^o)^\top \left( \nabla_x T_{\hat{\phi}_m}(x_i^o)^\top \theta + \nabla_x b_{\hat{\phi}_m}(x_i^o) \right)\right \Vert_2^2  + 2 \nabla_x \cdot \left(w(x_i^o)^2 \nabla_x b_{\hat{\phi}_m}(x_i^o)^\top\right) \nonumber \\
    & \quad + 2  \nabla_x \cdot \left(w(x_i^o)^2 \nabla_x T_{\hat{\phi}_m}(x_i^o)^\top  \right)\theta 
\end{align*}
We can then express this loss as a quadratic objective in $\theta$:
\begin{talign*}
    \mathcal{L}_{\text{NSM}}\left(\theta;x_{1:n}^o, \hat{\phi}_m \right) = \theta^\top M_1(x_i^o) \theta + 2\theta^\top \ell_1(x_i^o) + \ell_2(x_i^o) + 2\theta^\top M_2(x_i^o) + 2\ell_3(x_i^o), \quad \text{where}
\end{talign*}
\begin{talign*}
    M_1(x_{1:n}^o) &= \frac{1}{n} \sum_{i=1}^n w(x_i^o)^2 \nabla_x T_{\hat{\phi}_m}(x_i^o) \nabla_x T_{\hat{\phi}_m}(x_i^o)^\top, \quad \ell_1 (x_{1:n}^o ) = \frac{1}{n} \sum_{i=1}^n w(x_i^o)^2 \nabla_x T_{\hat{\phi}_m}(x_i^o) \nabla_x b_{\hat{\phi}_m}(x_i^o)\\
    \ell_2 (x_{1:n}^o ) &= \frac{1}{n} \sum_{i=1}^n w(x_i^o)^2 \nabla_x b_{\hat{\phi}_m}(x_i^o)^\top \nabla_x b_{\hat{\phi}_m}(x_i^o), \quad M_2 (x_{1:n}^o ) = \frac{1}{n} \sum_{i=1}^n \nabla_x \cdot \left(w(x_i^o)^2 \nabla_x T_{\hat{\phi}_m}(x_i^o)^\top\right) ^\top \\
    \ell_3 (x_{1:n}^o ) &= \frac{1}{n} \sum_{i=1}^n \nabla_x \cdot \left(w(x_i^o)^2 \nabla_x b_{\hat{\phi}_m}(x_i^o)^\top\right)^\top.
\end{talign*}
For the Gaussian-like prior $\pi(\theta) \propto \exp(\frac{1}{2}(\theta - \mu)^\top \Sigma^{-1}(\theta - \mu))$, we then obtain the conjugate expression for the NSM-Bayes posterior $\pi_{\text{NSM}}(\theta | x_{1:n}^o, \hat{\phi}_m) \propto \mathcal{N}(\mu_{n,m}, \Sigma_{n,m})$ by completing the square:
\begin{talign*} 
\log \pi_{\text{NSM}}&\left(\theta \big| x_{1:n}^o, \hat{\phi}_m \right) = \log \pi(\theta) - \beta n \mathcal{L}_{\text{NSM}}\left(\theta;x_{1:n}^o, \hat{\phi}_m \right) + C \\ 
&= -\frac{1}{2}(\theta - \mu)^\top \Sigma^{-1}(\theta - \mu) - \beta n \left[ \theta^\top M_1\left(x_{1:n}^o \right) \theta + 2\theta^\top \left(\ell_1 \left(x_{1:n}^o \right) + M_2 \left(x_{1:n}^o \right) \right) \right] + C' \\ 
&= -\frac{1}{2}(\theta^\top \Sigma^{-1} \theta - 2\theta^\top \Sigma^{-1}\mu) - \beta n \theta^\top M_1(x_{1:n}^o) \theta - 2\beta n\theta^\top \left(\ell_1(x_{1:n}^o) + M_2(x_{1:n}^o)\right) + C'' \\
&= -\frac{1}{2} \left[ \theta^\top(\Sigma^{-1} + 2\beta n M_1(x_{1:n}^o))\theta - 2\theta^\top\left(\Sigma^{-1}\mu - 2\beta n (\ell_1(x_{1:n}^o) + M_2(x_{1:n}^o)) \right) \right] + C''\\
&= -\frac{1}{2} \left[ \theta^\top \Sigma_{n,m}^{-1} \theta - 2\theta^\top \mu_{n,m} \right] + C''.
\end{talign*}
Here $C,C',C''$ are constants in $\theta$, and 
\begin{talign} \label{eq:closed-form-sigma}
    \Sigma_{n,m}^{-1} &:= \Sigma^{-1} + 2 \beta n M_1(x_{1:n}^o) = \Sigma^{-1} + 2 \beta \sum_{i=1}^n w(x_i^o)^2 \nabla_x T_{\hat{\phi}_m}(x_i^o) \nabla_x T_{\hat{\phi}_m}(x_i^o)^\top,\\
    \mu_{n,m} & := \Sigma_{n,m}\left[ \Sigma^{-1}\mu - 2\beta n M_2(x_{1:n}^o)  - 2\beta n \ell_1(x_{1:n}^o) \right] \label{eq:closed-form-mu}\\
    &= \Sigma_{n,m}\left[ \Sigma^{-1}\mu - 2\beta \sum_{i=1}^n \left(\nabla_x \cdot \left(w(x_i^o)^2 \nabla_x T_{\hat{\phi}_m}(x_i^o)^\top\right)^\top  + w(x_i^o)^2 \nabla_x T_{\hat{\phi}_m}(x_i^o) \nabla_x b_{\hat{\phi}_m}(x_i^o)\right) \right]. \nonumber
\end{talign}
This completes the proof.
\end{proof}

\subsection{Proof of \Cref{theorem:non-robustness-NLE}} \label{appendix:proof_non_robust}
\begin{proof}
Recall the PIF for the NLE posterior:
\begin{talign*}
    \pif_{\Delta}(x^c, \theta, \mathbb{P}_0, \mathcal{L}_{\text{NLE}}) := \frac{d}{d \epsilon} \pi_{\text{NLE}}\left(\theta \Big| \mathbb{P}_{\epsilon, x^c}, \hat \phi_m \right) \bigg|_{\epsilon=0}.
\end{talign*}
     From (17) of \cite{ghosh2016robust},
     we have for the NLE posterior that:
    \begin{talign*}
        &\sup_{\theta \in \Theta, x^c \in \mathcal{X}} \left|\text{PIF}_{\Delta} (x^c,\theta;\mathbb{P}_n, \mathcal{L}_{\text{NLE}}) \right|  \\
        & = \beta n \times \sup_{\theta \in \Theta, x^c \in \mathcal{X}}  \pi_\text{NLE}\left(\theta \big|x_{1:n}^o, \hat{\phi}_m \right) \times \left|- D \mathcal{L}_{\text{NLE}}(x^c;\theta,\mathbb{P}_n) + \int D \mathcal{L}_{\text{NLE}}(x^c;\theta',\mathbb{P}_n) \pi_{\text{NLE}}\left(\theta' \big|x_{1:n}^o, \hat{\phi}_m \right) d\theta'\right| 
    \end{talign*}
    where 
    \begin{talign*}
        D\mathcal{L}_{\text{NLE}}(x^c;\theta,\mathbb{P}_n) &:= \frac{d}{d \epsilon} \left(-\mathcal{L}_{\text{NLE}}(\theta; \mathbb{P}_{\epsilon,x^c}) \right)\bigg|_{\epsilon=0} \\
         &= \frac{d}{d \epsilon} \left((1-\epsilon)\left(\frac{1}{n}\sum_{i=1}^n \log q_{\hat{\phi}_m}(x_i^o | \theta)\right) + \epsilon \left(\log q_{\hat{\phi}_m}(x^c | \theta)\right)\right) \bigg|_{\epsilon=0} \\
         &=  \log q_{ \hat{\phi}_m}(x^c|\theta) - \frac{1}{n}\sum_{i=1}^n \log q_{\hat{\phi}_m}(x_i^o|\theta).
    \end{talign*}
    Instead of considering the supremum over all possible $\theta$ values, we consider the lower bound obtained by fixing $\theta = \theta_{\text{MAP}}$, the maximum \textit{a posteriori} (MAP) estimate of the NLE posterior: 
    \begin{talign*}
         &\sup_{\theta \in \Theta, x^c \in \mathcal{X}} \left|\text{PIF}_{\Delta} (x^c,\theta,\mathbb{P}_n, \mathcal{L}_{\text{NLE}})\right| \\ 
        &\quad \geq C_1  \sup_{ x^c \in \mathcal{X}} \left|- D \mathcal{L}_{\text{NLE}}(x^c;\theta_\text{MAP},\mathbb{P}_n) + \int D \mathcal{L}_{\text{NLE}}(x^c;\theta',\mathbb{P}_n) \pi_{\text{NLE}}\left(\theta' \big|x_{1:n}^o,\hat{\phi}_m \right) d\theta'\right| \\
        &\quad \geq C_1  \left[\sup_{ x^c \in \mathcal{X}} \int D \mathcal{L}_{\text{NLE}}(x^c;\theta',\mathbb{P}_n) \pi_{\text{NLE}}\left(\theta' \big|x_{1:n}^o,\hat{\phi}_m \right) d\theta' - \inf_{ x^c \in \mathcal{X}}  D \mathcal{L}_{\text{NLE}}\left(x^c;\theta_\text{MAP},\mathbb{P}_n\right)\right],  
    \end{talign*}
    where the constant $C_1 = \beta n \times  \pi_\text{NLE}(\theta_\text{MAP}|x_{1:n}^o,\hat{\phi}_m) = \beta n \times \frac{1}{Z}\prod_{i=1}^n q_{\hat{\phi}_m}(x^o_i|\theta_\text{MAP}) \pi(\theta_\text{MAP}) > 0$  (this must be strictly greater than zero since $\theta_{\text{MAP}}$ lies in the support of the posterior), and the second inequality uses $\sup |f(x)-g(x)|\geq \sup f(x) - \inf g(x)$. \\Now we note that we can lower bound this supremum by replacing $x^c$ by $x^o_n$:
    \begin{talign*}
        &\sup_{x^c \in \mathcal{X}} \int D \mathcal{L}_{\text{NLE}}(x^c;\theta',\mathbb{P}_n) \pi_{\text{NLE}}\left(\theta' \big|x_{1:n}^o,\hat{\phi}_m \right) d\theta' \\
        & = 
        \sup_{x^c \in \mathcal{X}} \int \left(\log q_{\hat{\phi}_m}(x^c|\theta') - \frac{1}{n}\sum_{i=1}^n \log q_{\hat{\phi}_m}(x^o_i|\theta')\right) \pi_{\text{NLE}}\left(\theta' \big|x_{1:n}^o,\hat{\phi}_m \right) d\theta' \\
        & \geq 
         \int \left(\log q_{\hat{\phi}_m}(x^o_n|\theta') - \frac{1}{n}\sum_{i=1}^n \log q_{\hat{\phi}_m}(x^o_i|\theta')\right) \pi_{\text{NLE}}\left(\theta' \big|x_{1:n}^o,\hat{\phi}_m \right) d\theta'\\
        &\geq \left(1- \frac{1}{n}\right)  \int \log q_{\hat{\phi}_m}(x^o_n|\theta') \pi_{\text{NLE}}\left(\theta' \big|x_{1:n}^o,\hat{\phi}_m \right) d\theta' -  \frac{1}{n} \sum_{i=1}^{n-1} \int \log q_{\hat{\phi}_m}(x^o_i|\theta')\pi_{\text{NLE}}\left(\theta' \big|x_{1:n}^o,\hat{\phi}_m \right)d\theta'\\
        & > - \infty
    \end{talign*}
    where the final step follows from the fact that based on the assumptions the posterior expectations are finite, i.e. $\int |\log q_{\hat{\phi}_m}(x_i^o | \theta^\prime) | \pi_{\text{NLE}}(\theta^\prime | x_{1:n}^o,\hat{\phi}_m ) d \theta^\prime < \infty$ for $i = 1, \dots, n$. 
    Finally, we have 
    \begin{align*}
        - \inf_{x^c \in \mathcal{X}} D \mathcal{L}_{\text{NLE}}(x^c, \theta_{\text{MAP}}, \mathbb{P}_n) &= - \inf_{x^c \in \mathcal{X}} \log q_{\hat{\phi}_m}(x^c | \theta_{\text{MAP}}) + \frac{1}{n}\sum_{i=1}^n \log q_{\hat{\phi}_m}(x^o_i | \theta_{\text{MAP}}) \rightarrow \infty
    \end{align*}
     since we assumed that $q_{\hat{\phi}_m}$ is a normalised density in $\mathbb{R}^d$ and hence $\inf_{x\in\mathcal{X}} q_{\hat{\phi}_m}(x | \theta_{\text{MAP}}) = 0$. 
     Thus, 
     $-\inf_{x \in \mathcal{X}} \log q_{\hat{\phi}_m}(x | \theta_{\text{MAP}}) = \infty $ and hence
     $\sup_{\theta \in \Theta, x^c \in \mathcal{X}} | \pif_\Delta(x^c, \theta, \mathbb{P}_n, \mathcal{L}_{\text{NLE}})| = \infty$. 
\end{proof}
\subsection{Proof of \Cref{theorem:robustness}} \label{appendix:proof_robust}
\begin{proof}
By Proposition B.1 of \cite{Altamirano2023}, we need to show there exists $\gamma: \Theta \rightarrow \mathbb{R}$ such that: (1) $\sup_{x \in \mathcal{X}} \mathcal{L}_{\text{NSM}}(\theta; x, \hat{\phi}_m) \leq \gamma(\theta)$, (2) $\sup_{\theta \in \Theta} \gamma(\theta) \pi(\theta) < \infty$, and (3)
\item $\int_{\Theta} \gamma(\theta) \pi(\theta) d \theta < \infty$.
Let $S_\theta(x) := \| \nabla_x \log q_{\hat{\phi}_m}(x | \theta) \|_2$ and $H_\theta(x) := | \text{Tr} (\nabla_x^2 \log q_{\hat{\phi}_m}(x | \theta)) |$, then:
\begin{align*}
    \mathcal{L}_{\text{NSM}}\left(\theta ; x, \hat \phi_m \right) &= w(x)^2 S_\theta(x)^2 + 2 \left(\nabla_x w(x)^2 \right)^\top \nabla_x \log q_{\hat{\phi}_m}(x | \theta) + 2 w(x)^2 H_\theta(x).
\end{align*}
Therefore, under the assumptions of the theorem, we have: 
\begin{align*}
    \mathcal{L}_{\text{NSM}}\left(\theta; x, \hat \phi_m \right) &= w(x)^2 S_\theta(x)^2 + 2\left(\nabla_x w(x)^2 \right)^\top \nabla_x \log q_{\hat{\phi}_m}(x | \theta) + 2 w(x)^2 H_\theta(x) \\
    &\leq f(\theta)^2  + 2\left\| \nabla_x w(x)^2 \right\|_2 S_\theta(x) + 2 g(\theta)  \\
    &\leq f(\theta)^2  + 2 f(\theta)  + 2 g(\theta) 
\end{align*}
where we used the Cauchy-Schwarz inequality in the first inequality. 
Therefore, it follows from Assumption \ref{assumption:general-pif} that
\begin{align*}
    \sup_{x} \mathcal{L}_{\text{NSM}}\left(\theta;x, \hat \phi_m \right) \leq f(\theta)^2 +  2 f(\theta) + 2 g(\theta) < \infty 
\end{align*}
and therefore conditions 2 and 3 above hold with $\gamma(\theta) := f(\theta)^2 +  2 f(\theta) + 2 g(\theta)$.
The result then follows from Proposition B.1. of \cite{Altamirano2023}. 
\end{proof}

\subsection{Robustness with the IMQ weights and sub-exponential priors} \label{app:sec:imq}
In this section, we prove that under the simplified Assumption \ref{assumption:imq-pif}, global-bias robustness holds and verify that these assumptions hold in the cases of mixture density models and masked autoregressive flows.

\subsubsection{Proof of Corollary \ref{cor:imq}} \label{appendix:proof:cor}

\begin{proof}
Similarly to the proof of Theorem \ref{theorem:robustness}, we prove the result using Proposition B.1 of \cite{Altamirano2023}. Following this result, it is sufficient to show that there exists a function $\gamma: \Theta \rightarrow \mathbb{R}$ such that: (1) $\sup_{x \in \mathcal{X}} \mathcal{L}_{\text{NSM}}(\theta; x,\hat{\phi}_m) \leq \gamma(\theta)$, (2)  $\sup_{\theta \in \Theta} \gamma(\theta) \pi(\theta) < \infty$, (3) $\int_{\Theta} \gamma(\theta) \pi(\theta) d \theta < \infty$.

For simplicity let $S_\theta(x) := \| \nabla_x \log q_{\hat{\phi}_m}(x | \theta) \|_2$ and $H_\theta(x) := | \text{Tr} (\nabla_x^2 \log q_{\hat{\phi}_m}(x | \theta)) |$. Then using the Cauchy-Schwarz inequality we have: 
\begin{align} \label{eq:lnsm}
    \mathcal{L}_{\text{NSM}}\left(\theta ; x, \hat \phi_m \right) &= w_{\zeta}(x)^2 S_\theta(x)^2 + 2 \left(\nabla_x w_{\zeta}(x)^2 \right)^\top \nabla_x \log q_{\hat{\phi}_m}(x | \theta) + 2 w_{\zeta}(x)^2 H_\theta(x) \nonumber\\
    &\leq w_{\zeta}(x)^2 S_\theta(x)^2 + 2 \left\|\nabla_x w_{\zeta}(x)^2 \right\|_2 S_\theta(x) + 2 w_{\zeta}(x)^2 H_\theta(x).
\end{align}
We first make use of the specific form of the IMQ weight to bound all terms relating to $w_{\zeta}(x)$ in terms of a function of $\| x\|_2$. We start with bounding the $\| \nabla_x w_{\zeta}(x)^2\|_2$ term. 
Recall that  
$
    w_{\zeta}(x) = ( 1 + r(x)^2 )^{- 1/\zeta}
$
for $\zeta > 0$ where $r(x)^2 = \| x - \hat{\nu}_n \|^2_{\hat{\Xi}_n^{-1}}$. 
Then 
\begin{talign*}
    \nabla_x w_{\zeta}(x)^2 &= -  \frac{2}{\zeta}\left(1 +  r(x)^2 \right)^{-\frac{2}{\zeta}-1} \nabla_x \left( r(x)^2\right) = - \frac{4}{\zeta} \left(1 + r(x)^2 \right)^{-\frac{2}{\zeta}-1}  \hat{\Xi}_n^{-1} (x - \hat{\nu}_n) 
\end{talign*}
and hence by letting $\lambda_{\text{max}}(\hat{\Xi}_n^{-1})$ denote the largest eigenvalue of $\hat{\Xi}_n^{-1}$ (which is positive since $\hat{\Xi}_n^{-1}$ is positive-definite) we obtain: 
\begin{talign} \label{eq:bound-on-w}
    \left\| \nabla_x w_{\zeta}(x)^2 \right\|_2 \leq \frac{4}{\zeta} \left(1 + r(x)^2 \right)^{- \frac{2}{\zeta}-1} \left\|\hat{\Xi}_n^{-1} (x - \hat{\nu}_n) \right\|_2 
    &\leq \frac{4}{\zeta} \left(1 + r(x)^2\right)^{- \frac{2}{\zeta}-1} \left\| \hat{\Xi}_n^{-1} \right\|_{\text{op}} \left\| x - \hat{\nu}_n \right\|_2 \nonumber\\
    &\leq \frac{4}{\zeta} \left(1 + r(x)^2\right)^{- \frac{2}{\zeta}-1} \lambda_{\text{max}}(\hat{\Xi}_n^{-1}) \left\| x - \hat{\nu}_n \right\|_2 \nonumber\\
    &= \frac{4}{\zeta} w_{\zeta}(x)^{2 + \zeta} \lambda_{\text{max}}(\hat{\Xi}_n^{-1}) \left\| x - \hat{\nu}_n \right\| _2 
\end{talign}
where the second inequality follows from the definition of the operator norm, the third inequality follows from the fact that $\hat \Xi_n^{-1}$ is a symmetric, positive definite matrix and the last equality follows from the fact that 
\begin{talign*}
\left(1 + r(x)^2\right)^{- \frac{2}{\zeta} - 1} = \left(w_{\zeta}(x)^{- \zeta} \right)^{- \frac{2}{\zeta} - 1} = w_{\zeta}(x)^{2 + \zeta}.
\end{talign*}
Now by the Rayleigh quotient bounds \citep[][Theorem 4.2.2]{horn2012matrix} it follows that $r(x)^2 = \| x - \hat{\nu}_n \|^2_{\hat{\Xi}_n^{-1}} = (x-\hat{\nu}_n)^\top \hat{\Xi}_n^{-1}(x - \hat{\nu}_n) \geq \lambda_{\text{min}}(\hat{\Xi}_n^{-1}) \| x - \hat{\nu}_n \|^2_2 $ and hence $\| x - \hat{\nu}_n \|_2 \leq r(x)/(\lambda_{\text{min}}(\hat{\Xi}_n^{-1}))^{1/2} $ where $\lambda_{\text{min}}(\hat{\Xi}_n^{-1})$ denotes the smallest eigenvalue of $\hat{\Xi}_n^{-1}$ which satisfies $\lambda_{\text{min}}(\hat{\Xi}_n^{-1}) > 0$ since $\hat{\Xi}_n^{-1}$ is positive definite. Therefore, by using this in \Cref{eq:bound-on-w} we obtain:
\begin{talign*}
    \left\| \nabla_x w_{\zeta}(x)^2 \right\|_2 &\leq 4 \frac{\lambda_{\text{max}}(\hat{\Xi}_n^{-1})}{\sqrt{\lambda_{\text{min}}(\hat{\Xi}_n^{-1})}} r(x) w_{\zeta}(x)^{2+\zeta} \leq \underbrace{4 \frac{\lambda_{\text{max}}(\hat{\Xi}_n^{-1})}{\sqrt{\lambda_{\text{min}}(\hat{\Xi}_n^{-1})}}}_{:= C(\hat{\Xi}_n^{-1}) < \infty}   w_{\zeta}(x)^{2+\frac{\zeta}{2}}
\end{talign*}
since $r(x) = (w_{\zeta}(x)^{-\zeta} - 1)^{1/2} \leq (w_{\zeta}(x)^{-\zeta})^{1/2} = w_{\zeta}(x)^{\zeta/2}$.
Substituting in (\ref{eq:lnsm}) we obtain:
\begin{align} \label{eq:lnsm-updated}
    \mathcal{L}_{\text{NSM}}\left(\theta; x, \hat{\phi}_m \right) \leq w_{\zeta}(x)^2 S_\theta(x)^2 + 4 C(\hat{\Xi}_n^{-1}) w_{\zeta}(x)^{2 + \frac{\zeta}{2}} S_\theta(x) + 2 w_{\zeta}(x)^2 H_\theta(x).
\end{align}
We now proceed by bounding the terms in (\ref{eq:lnsm-updated}) involving $w_{\zeta}(x)^2$. 
Under Assumption \ref{assumption:imq-pif}, we have that: $S_\theta(x) \leq K_1 (1 + \| \theta \|^{k_1}) ( 1 + \|x\|^{2/\zeta}_2)$ and $H_\theta(x) \leq K_2 (1 + \| \theta \|^{k_2}) ( 1 + \|x\|_2^{4/\zeta})$ hence we aim to bound $1 + \| x\|_2^{2/\zeta}$ and $1 + \|x\|_2^{4/\zeta}$ in terms of $w_{\zeta}(x)$. 
To achieve this, we study the behaviour of $w_{\zeta}(x) = (1+ r(x)^2)^{-1/\zeta}$ relative to $\|x\|_2$ separately when $\|x\|_2 \leq 1$ and $\|x\|_2 \geq 1$. 

First consider the case $\|x\|_2 \geq 1$.
As previously stated, since $\hat{\Xi}_n^{-1}$ is positive definite, its eigenvalues satisfy $0 < \lambda_{\text{min}}(\hat{\Xi}_n^{-1}) \leq \lambda_{\text{max}}(\hat{\Xi}_n^{-1})$. Therefore, for any $x \in \mathbb{R}^{d_\mathcal{X}}$ we have:
\begin{align} \label{eq:r-ineq}
    \| x\|_{\hat{\Xi}_n^{-1}}^2 = x^\top \hat{\Xi}_n^{-1} x \leq  \lambda_{\text{max}}(\hat{\Xi}_n^{-1}) \|x\|_2^2
\quad \text{and hence,} \quad  
    r(x)^2  \leq     \lambda_{\text{max}}(\hat{\Xi}_n^{-1}) \left\|x - \hat{\nu}_n \right\|_2^2.
\end{align}
From (\ref{eq:r-ineq}) and using the triangle inequality we have: 
\begin{talign*}
    r(x)^2 \leq \lambda_{\text{max}}(\hat{\Xi}_n^{-1}) \left\|x - \hat{\nu}_n \right\|_2^2 \leq \lambda_{\text{max}}(\hat{\Xi}_n^{-1}) \left(\|x\|_2^2 + \|\hat{\nu}_n\|_2^2 \right) & \leq \lambda_{\text{max}}(\hat{\Xi}_n^{-1}) \left(\| x\|_2^2 + \|\hat{\nu}_n\|_2^2\| x\|_2^2\right) \\
    &:= C_1 \|x\|_2^2
\end{talign*}
where $C_1 := \lambda_{\text{max}}(\hat{\Xi}_n^{-1})\left(1 + \|\hat{\nu}_n\|_2^2\right) < \infty$.
Substituting this into the definition of $w_{\zeta}$ we have:
\begin{align*}
    w_{\zeta}(x) = \left(1 + r(x)^2\right)^{-\frac{1}{\zeta}} \leq \left(r(x)^2\right)^{-\frac{1}{\zeta}} \leq C_1^{- \frac{1}{\zeta}} \|x\|_2^{-\frac{2}{\zeta}}
\end{align*}
which yields:
\begin{align} \label{eq:ineq1}
 1+\|x\|_2^{\frac{2}{\zeta}} \leq C_1^{-\frac{1}{\zeta}}  w_{\zeta}(x)^{-1}+1,
    \quad  1+\|x\|_2^{\frac{4}{\zeta}} \leq C_1^{-\frac{2}{\zeta}}  w_{\zeta}(x)^{-2} + 1. 
\end{align}
Now consider the case $\| x\|_2 \leq 1$. Since $w(x) \leq 1$ we have:
\begin{align} \label{eq:ineq2}
     1 + \|x\|_2^{\frac{2}{\zeta}} \leq 2 \leq  w_{\zeta}(x)^{-1} + 1, \quad  1 + \|x\|_2^{\frac{4}{\zeta}} \leq 2 \leq w_{\zeta}(x)^{-2} + 1.
\end{align}
By letting $C_2 := \max\{1, C_1^{-1/\zeta}\}$ and $C_3 := \max\{1, C_1^{-2/\zeta}\}$ and using (\ref{eq:ineq1}) and (\ref{eq:ineq2}), we have that for all $x \in \mathbb{R}^{d_{\mathcal{X}}}$:
\begin{align} \label{eq:final-ineq}
      1 + \|x\|_2^{\frac{2}{\zeta}} \leq C_2 w_{\zeta}(x)^{-1} + 1, \quad   \|x\|_2^{\frac{4}{\zeta}} \leq C_3 w_{\zeta}(x)^{-2} + 1.
\end{align}
Substituting (\ref{eq:final-ineq}) in (\ref{eq:lnsm-updated}) and using Assumption \ref{assumption:imq-pif} we obtain: 
\begin{talign*}
    \mathcal{L}_{\text{NSM}}\left(\theta; x, \hat \phi_m \right) &\leq  w_{\zeta}(x)^2 S_\theta(x)^2 + 4 C(\hat{\Xi}_n^{-1}) w_{\zeta}(x)^{2 + \frac{\zeta}{2}} S_\theta(x) + 2 w_{\zeta}(x)^2 H_\theta(x) \\
    &\leq w_{\zeta}(x)^2 K_1^2 \left(1 + \| \theta \|^{2k_1}\right) \left(1+\|x\|_2^{\frac{4}{\zeta}}\right) +  4C(\hat{\Xi}_n^{-1}) w_{\zeta}(x)^{2 + \frac{\zeta}{2}} K_1 \left(1 + \| \theta \|^{k_1}\right) \left(1+\| x\|_2^{\frac{2}{\zeta}}\right) \\
    &\quad \quad+ 2 w_{\zeta}(x)^2 K_2 \left(1 + \| \theta \|^{k_2}\right) \left(1+\|x\|_2^{\frac{4}{\zeta}}\right) \\
    &\leq K_1^2 \left(1 + \| \theta \|^{2k_1}\right) \left(C_3w_{\zeta}(x)^{-2} w_{\zeta}(x)^2 + w_{\zeta}(x)^2\right) \\
    &\quad + 4 C(\hat{\Xi}_n^{-1}) K_1 \left(1 + \| \theta \|^{k_1}\right) \left(w_{\zeta}(x)^{2 + \frac{\zeta}{2}} C_2 w_{\zeta}(x)^{-1} + w_{\zeta}(x)^{2 + \frac{\zeta}{2}} \right) \\
    &\quad \quad + 2  K_2 \left(1 + \| \theta \|^{k_2}\right) \left( C_3w_{\zeta}(x)^{-2} w_{\zeta}(x)^2 + w_{\zeta}(x)^2\right) \\
    &= K_1^2 \left(1 + \| \theta \|^{2k_1}\right) \left(C_3+ w_{\zeta}(x)^2\right)  4 C(\hat{\Xi}_n^{-1}) K_1 \left(1 + \| \theta \|^{k_1}\right) \left(w_{\zeta}(x)^{1 + \frac{\zeta}{2}} C_2 + w_{\zeta}(x)^{2 + \frac{\zeta}{2}} \right) \\
    &\quad \quad + 2  K_2 \left(1 + \| \theta \|^{k_2}\right) \left( C_3 + w_{\zeta}(x)^2\right).
\end{talign*}
Since $\zeta > 0$ we have $\sup_x w_{\zeta}(x)^{1+\zeta/2} = \sup_x w_{\zeta}(x)^{2+\zeta/2} = \sup_x w_{\zeta}(x)^{2} = 1$ so  using the upper bound above we obtain:
\begin{talign*}
    \sup_{x} \mathcal{L}_{\text{NSM}}\left(\theta;x, \hat \phi_m \right) &\leq K_1^2 (C_3+1) \left(1 + \| \theta \|^{2k_1}\right) + 4 (C_2+1) C(\hat{\Xi}_n^{-1}) K_1 \left(1 + \| \theta \|^{k_1}\right)  \\
    &\quad + 2 (C_3+1) K_2 \left(1 + \| \theta \|^{k_2}\right) \\
    &:= \gamma(\theta) < \infty. 
\end{talign*}
Take $k := \max\{2k_1, k_2\}$ then 
$\sup_\theta | \gamma(\theta)| \lesssim (1 + \| \theta \|^k)$. Moreover, since $\pi(\theta)$ has infinitely many moments, we have that: $\int_{\Theta} \pi(\theta) \gamma(\theta) d \theta < \infty$.
Moreover, by assumption for $\| \theta \| \geq R$, $\pi(\theta) \leq C \exp(- c \| \theta \|)$ hence:
\begin{align*}
    \pi(\theta) \gamma(\theta) \lesssim \left(1 + \| \theta\|^k\right) C \exp\left(- c \| \theta \|\right) < \infty. 
\end{align*}
On the ball $ \{ \theta: \| \theta \| \leq R\}$ we have that $\gamma(\theta) \lesssim (1 + R^k)$ and $\pi(\theta)$ is finite since it is a density hence $\sup_{\theta} \gamma(\theta) \pi(\theta) 
    < \infty$.
The result follows from Proposition B.1. of \cite{Altamirano2023}.
\end{proof}
\subsubsection{Proof of Proposition \ref{proposition:mixture_density_network} (robustness for Gaussian Mixture-density Networks)}\label{appendix:subsec:gaussian}

\begin{proof} 
For simplicity, we drop the dependence on $\theta$ and $\hat{\phi}_m$ and denote $\vartheta_k := \vartheta_k(\theta; \hat{\phi}_m)$ and $\Omega_k := \Omega_k(\theta; \hat{\phi}_m)$. Let $f_k(x):= \mathcal{N}(x; \vartheta_k, \Omega_k)$ denote the $k$-th Gaussian density. Then the score of each density is
$
    \nabla_x \log f_k(x) = - \Omega_k^{-1}(x-\vartheta_k) 
$
and
\begin{talign*}
    \left\| \nabla_x \log f_k(x) \right\|_2 = \left\|  - \Omega_k^{-1}(x-\vartheta_k) \right\|_2 
    \leq \left\| \Omega_k^{-1} \right\|_2 \left\| x- \vartheta_k \right\|_2 
    \leq \lambda_{\text{max}}\left(\Omega_k^{-1}\right) \left(\| x \|_2 + \| \vartheta_k \|_2 \right).
\end{talign*}
Since $\Omega_k$ is symmetric positive-definite, it follows that $\lambda_{\text{max}}(\Omega_k^{-1}) = 1/\lambda_{\text{min}}(\Omega_k)$ and by Assumption of Proposition \ref{proposition:mixture_density_network} we have that for all $k \in \{1,\dots,K\}$, $\lambda_{\text{max}}(\Omega_k^{-1})  = 1/\lambda_{\text{min}}(\Omega_k)\leq 1/\lambda_{\text{min}}$. Hence, for the mixture, it follows that:
\begin{talign*}
   \left\|  \nabla_x \log q_{\hat{\phi}_m} (x | \theta)\right\|_2 = \left\| \frac{\sum_{k=1}^K \omega_k(\theta; \hat{\phi}_m) \nabla_x f_k(x)}{\sum_{k=1}^K \omega_k(\theta; \hat{\phi}_m) f_k(x)} \right \|_2 
  & = \left \| \sum_{k=1}^K \rho_k(x, \theta)\nabla_x \log f_k(x) \right \|_2 \\
  &\leq \sum_{k=1}^K \rho_k(x,\theta) \left\| \nabla_x \log f_k(x)\right\| \\ 
   &\leq \frac{1}{\lambda_{\text{min}}} \left(\| x\|_2 + \max_k \| \vartheta_k \|_2 \right)
\end{talign*}
where $\rho_k(x ; \theta) := \omega_k(\theta; \hat{\phi}_m) f_k(x)/\sum_{j=1}^K \omega_j(\theta,\hat{\phi}_m) f_j(x) $. 
Let $h(\theta) := \frac{1}{\lambda_{\text{min}}} (1 + \max_k \| \vartheta_k \|_2)$, then:
\begin{align*}
    \left\| \nabla_x \log q_{\hat{\phi}_m}(x | \theta) \right\|_2 \leq \frac{1}{\lambda_{\text{min}}} \left(\| x\|_2 + \max_k \| \vartheta_k \|_2 \right) &\leq \frac{1}{\lambda_{\text{min}}} \left(1+\| x\|_2 \right)\left(1+\max_k \| \vartheta_k \|_2 \right) \\&= h(\theta) \left(1 + \| x\|_2\right).
\end{align*}
Thus, for any $\zeta \leq 2$, there exists $\tilde{C} < \infty$ such that for $\tilde{h}(\theta) := \tilde{C} h(\theta)$, we have 
\begin{align} \label{eq:norm-nabla}
    \left\| \nabla_x \log q_{\hat{\phi}_m}(x | \theta) \right\|_2 \leq h(\theta) \left(1 + \| x\|_2\right) \leq \tilde{h}(\theta) \left(1 + \|x\|_2^{\frac{2}{\zeta}}\right).
\end{align}
for all $x \in \mathbb{R}^{d_{\mathcal{X}}}$. Furthermore, we have that $\nabla_x^2 \log f_k(x) = - \Omega_k^{-1}$ hence standard operations yield:
\begin{align*}
    \nabla_x^2 \log q_{\hat{\phi}_m}(x) = \sum_{k=1}^K \rho_k(x; \theta) \nabla_x^2 \log f_k(x) + \text{Cov}_\rho\left(\nabla_x \log f_k(x)\right)
\end{align*}
where $\text{Cov}_\rho$ denotes the covariance under the weights $\rho_k(x; \theta)$. Therefore, 
\begin{talign*}
    \left |\text{Tr} \nabla_x^2 \log q_{\hat{\phi}_m}(x | \theta) \right| &\leq \left|\sum_{k=1}^K \rho_k(x; \theta) \text{Tr} \Omega_k^{-1} \right| + \left| \text{Tr} \text{Cov}_\rho\left(\nabla_x \log f_k(x)\right)\right| \\
    &\leq \sum_{k=1}^K \rho_k(x; \theta) \left | \text{Tr} \Omega_k^{-1}\right| + \text{Var}_\rho\left(\nabla_x \log f_k(x)\right) \\
    &\leq \frac{d_{\mathcal{X}}}{\lambda_{\text{min}}} + \text{Var}_\rho\left(\nabla_x \log f_k(x)\right)
\end{talign*}
where the first inequality follows from the fact that trace is a linear operation, the second follows from the fact that $\left | \text{Tr} \text{Cov}_\rho(\nabla_x \log f_k)\right| \leq \sum_k \rho_k \| \nabla_x \log f_k(x) \|^2 $ and the third follows from the fact that the trace is equal to the sum of its eigenvalues. 
From \eqref{eq:norm-nabla}, we have that for any $\zeta \leq 2$:
\begin{talign*}
    \text{Var}_\rho\left(\nabla_x \log f_k(x)\right) \leq h(\theta)^2\left(1+\| x\|_2\right)^2 \leq 2 h(\theta)^2\left(1 +  \|x\|_2^2\right) \leq 2 \tilde{h}(\theta)^2 \left(1 + \|x\|^{\frac{4}{\zeta}}\right).
\end{talign*}
Therefore $\left|\text{Tr} \nabla_x^2 \log q_{\hat{\phi}_m}(x | \theta) \right| \leq (C_1 + 2\tilde{h}(\theta)^2) (1 + \| x\|_2^{4/\zeta})$ for $C_1 := d_{\mathcal{X}}/\lambda_{\text{min}}$. 
Since $\tilde{h}(\theta) \propto 1/\lambda_{\text{min}}(1 + \max_k \| \vartheta_k \|_2)$, Assumption \ref{assumption:imq-pif} holds since $\max_k \| \vartheta_k(\theta) \|_2$ has polynomial growth by assumption. 
\end{proof}
\subsubsection{Proof of Proposition \ref{thm:masked_autoregressive_flow} (robustness for Masked Autoregressive Flows)} \label{appendix:proof:maf}

\begin{proof}
First we express the score and Hessian of $q_{\hat{\phi}_m}(x|\theta)$ in terms of $\mu_{\hat{\phi}_m,i}, \sigma_{\hat{\phi}_m,i}$ and their derivatives up to second order. Then we derive bounds for $\mu_{\hat{\phi}_m,i}, \sigma_{\hat{\phi}_m,i}$ and their first two derivatives which we then substitute into the norm of the score and trace expressions to verify that their growth is in accordance with Assumption \ref{assumption:imq-pif} for certain values of $\zeta$. 

We want to show that 
there exists $0 < K_1, K_2 < \infty$ and $k_1, k_2 \geq 0$
such that $\forall x \in \mathcal{X}$:
\begin{talign*}
    \| \nabla_x \log q_{\hat{\phi}_m}(x | \theta) \|_2 &\leq K_1 \left(1 + \| \theta \|^{k_1}\right) \left(1 + \| x \|_2^{\frac{2}{\zeta}}\right)
    \\
   | \text{Tr} \nabla^2_x \log q_{\hat{\phi}_m}(x | \theta)| &\leq 
   K_2 \left(1 + \| \theta \|^{k_2} \right)
   \left( 1 + \|x\|^{\frac{4}{\zeta}}_2 \right).
\end{talign*}
From the definition of the normalising flow and the fact that $p_z(z) = \mathcal{N}(0,I)$ we have that:
\begin{talign*}
    \log q_{\hat{\phi}_m}(x | \theta) &= \sum_{i=1}^{d_{\mathcal{X}}} \left[\log p_z(z_i) - \log \sigma_{\hat{\phi}_m,i}(x_{<i}, \theta) \right] = \sum_{i=1}^{d_{\mathcal{X}}} \left[- \frac{1}{2} z_i^2 - \frac{1}{2} \log (2 \pi) - \log \sigma_{\hat{\phi}_m,i}(x_{<i, \theta}) \right].
\end{talign*}
For each $j = 1, \dots, d_{\mathcal{X}}$ we have that:
\begin{talign} \label{eq:grad-log-q}
    \frac{\partial}{\partial x_j} \log q_{\hat{\phi}_m}(x | \theta) = \sum_{i=1}^{d_{\mathcal{X}}} - z_i \frac{\partial z_i}{\partial x_j} - \frac{1}{\sigma_{\hat{\phi}_m,i}(x_{<i, \theta})} \frac{\partial \sigma_{\hat{\phi}_m,i}(x_{<i, \theta})}{\partial x_j} 
\end{talign}
and using the product and chain rule we further obtain:
\begin{talign*}
    &\frac{\partial^2}{\partial x_j^2} \log q_{\hat{\phi}_m}(x | \theta)= \sum_{i=1}^{d_{\mathcal{X}}} - z_i \frac{\partial^2 z_i}{\partial x_j^2} - \left(\frac{\partial z_i}{\partial x_j} \right)^2 - \frac{1}{\sigma_{\hat{\phi}_m,i}(x_{<i,\theta})} \frac{\partial^2 \sigma_{\hat{\phi}_m,i}(x_{<i, \theta})}{\partial x_j^2} - \frac{1}{\sigma_{\hat{\phi}_m,i}(x_{<i,\theta})^2} \left(\frac{\partial \sigma_{\hat{\phi}_m,i}(x_{<i, \theta})}{\partial x_j} \right)^2.
\end{talign*}
Therefore, 
\begin{talign} \label{eq:grad}
    \left\| \nabla_{x} \log q_{\hat{\phi}_m}(x | \theta) \right\|_2 &=  \left\| \sum_{i=1}^{d_{\mathcal{X}}} - z_i \nabla_x z_i - \frac{1}{\sigma_{\hat{\phi}_m,i}(x_{<i, \theta})} \nabla_x \sigma_{\hat{\phi}_m,i}(x_{<i, \theta}) \right \|_2, 
\\ \label{eq:hess}
    \left|\text{Tr} \nabla_x^2 \log q_{\hat{\phi}_m}(x | \theta) \right| &= 
    \sum_{j=1}^{d_{\mathcal{X}}} \left[ 
    \sum_{i=1}^{d_{\mathcal{X}}} - z_i \frac{\partial^2 z_i}{\partial x_j^2} - \left(\frac{\partial z_i}{\partial x_j} \right)^2 - \frac{1}{\sigma_{\hat{\phi}_m,i}(x_{<i,\theta})} \frac{\partial^2 \sigma_{\hat{\phi}_m,i}(x_{<i, \theta})}{\partial x_j^2} \right.\nonumber\\
    &\left. \quad- \frac{1}{\sigma_{\hat{\phi}_m,i}(x_{<i,\theta})^2} \left(\frac{\partial \sigma_{\hat{\phi}_m,i}(x_{<i, \theta})}{\partial x_j} \right)^2 \right].
\end{talign}
We now bound all individual terms in  the expressions in (\ref{eq:grad}) and (\ref{eq:hess}).
Since $z_i = (x_i - \mu_{\hat{\phi}_m,i}(x_{<i},\theta))/\sigma_{\hat{\phi}_m,i}(x_{<i, \theta})$ it follows that:
\begin{talign} \label{eq:z-der}
    \frac{\partial z_i}{\partial x_j} &= 
    \begin{cases}
        \frac{1}{\sigma_{\hat{\phi}_m,i}(x_{<i},\theta)} & \text{if} \quad j=i\\
        - \frac{1}{\sigma_{\hat{\phi}_m,i}} \frac{\partial \mu_{\hat{\phi}_m,i}}{\partial x_j} - \frac{x_i - \mu_{\hat{\phi}_m,i}}{\sigma^2_{\hat{\phi}_m,i}} \frac{\partial \sigma_{\hat{\phi}_m,i}}{\partial x_j}&\text{if} \quad j <i \\
        0 & \text{if} \quad j > i 
    \end{cases} \quad \text{and}
\\\label{eq:z-2der}
    \frac{\partial^2 z_i}{\partial x_j^2} &=  \begin{cases}
         \frac{1}{\sigma_{\hat{\phi}_m,i}^2} \frac{\partial \mu_{\hat{\phi}_m,i}}{\partial x_j} - \frac{1}{\sigma_{\hat{\phi}_m,i}} \frac{\partial^2 \mu_{\hat{\phi}_m,i}}{\partial x_j^2}
         - \frac{1}{\sigma^2_{\hat{\phi}_m,i}} \frac{\partial \sigma_{\hat{\phi}_m,i}}{\partial x_j} - \frac{x_i - \mu_{\hat{\phi}_m,i}}{\sigma^2_{\hat{\phi}_m,i}} \frac{\partial^2 \sigma_{\hat{\phi}_m,i}}{\partial x_j^2}
         &\text{if} \quad j <i \\
        0 & \text{otherwise}
    \end{cases}.
\end{talign}
Hence we need to bound terms involving $\mu_{\hat{\phi}_m,i}, \sigma_{\hat{\phi},m}$ and their first and second derivatives.
Since for each coupling layer, the networks computing $\mu_{\hat{\phi}_m,i}$ and $\sigma_{\hat{\phi}_m,i}$ are MLPs built from affine layers with $\tanh$ activations and an additional $\text{softplus}$ at the final output for $\sigma_{\hat{\phi}_m,i}$, $\theta$ enters only in the first layer via:
\begin{align*}
    h_1(x,\theta) = \tanh \big(W_1 x + \tilde{g}(\theta)\big), \quad \tilde{g}(\theta) = \tanh\left(V_0 \theta + b_0\right), 
\end{align*}
whereas subsequent layers take the form: 
\begin{align*}
    h_{l+1}(x,\theta) = \tanh(W_l h_l(x,\theta)+b_l), \quad \text{for $l = 2, \dots, L$}.
\end{align*}
Since $\tanh$ is a bounded function, all intermediate layers are bounded. 
Hence, outputs $\mu_{\hat{\phi}_m,i}, s_{\hat{\phi}_m,i}$ are affine functions of $h_L$ and the scale is $\sigma_{\hat{\phi}_m,i} = \text{softplus}(s_{\hat{\phi}_m,i})$. Therefore, for $k=1,2$:
\begin{talign} \label{eq:mu-der}
    \left|\mu_{\hat{\phi}_m,i}(x_{<i},\theta) \right| \leq C, \quad \left\|\frac{\partial^k \mu_{\hat{\phi}_m,i}}{\partial x^k_j} \right\| \leq C_k, \quad \text{for $j = 1, \dots, d_{\mathcal{X}}$}
\end{talign}
and for some constants $C, C_1, C_2 < \infty$ independent of $\theta$. For $\sigma_{\hat{\phi}_m,i}$ we have that 
\begin{talign*}
    s_{\hat{\phi}_m,i}(x_{<i}, \theta) = a_i^\top h_L(x_{<i},\theta) + c_i , \quad \sigma_{\hat{\phi}_m,i}(x_{<i}, \theta) = \text{softplus}(s_{\hat{\phi}_m,i})
\end{talign*}
and since $\| h_L \|_{\infty} \leq 1$ (since $\text{tanh}$ is a bounded function) it follows that:
\begin{talign} \label{eq:s-bound}
    \left|s_{\hat{\phi}_m,i}(x_{<i}, \theta) \right | \leq \| a_i \|_1 + |c_i| := S_i.
\end{talign}
Since $\text{softplus}(t) := \log(1 + e^t)$ we have that $\max\{0,t\} \leq \text{softplus}(t) \leq \max\{0,t\} + \log 2$ and since $- S_i \leq s_{\hat{\phi}_m,i}(x_{<i}, \theta) \leq S_i$ from (\ref{eq:s-bound}), it follows that 
\begin{talign} \label{eq:sigma-bound}
    0 < \sigma_i^L := \text{softplus}(-S_i) \leq \sigma_{\hat{\phi}_m,i}(x_{<i},\theta)\leq \text{softplus}(S_i) := \sigma_i^U < \infty.
\end{talign}
Since the softplus function has globally bounded first and second derivatives, namely
\begin{talign*}
    \frac{\partial \text{softplus}(t)}{\partial t} = \frac{1}{1+e^{-t}} \in (0,1), \quad \frac{\partial^2 \text{softplus}(t)}{\partial t^2} = \frac{e^{-t}}{(1+e^{-t})^2} \leq \frac{1}{4},
\end{talign*}
it follows by the chain rule that the first and second derivatives of $\sigma_{\hat{\phi}_m,i}$ are also bounded, i.e. 
\begin{talign} \label{eq:sigma-der}
    \left\|\frac{\partial^k \sigma_{\hat{\phi}_m,i}}{\partial x^k_j} \right\| \leq \tilde{C}_k, \quad \text{for $k=1,2$}.
\end{talign}
Applying the bounds from (\ref{eq:mu-der}),(\ref{eq:sigma-bound}),(\ref{eq:sigma-der}) in the expressions in (\ref{eq:z-der}) and (\ref{eq:z-2der}) we have that for some constant $C^\prime <\infty$:
\begin{talign} \label{eq:z-bound}
    &| z_i | \leq \frac{|x_i| + C}{\sigma_i^L} \leq C^\prime(1 + | x_i|), \quad 
    \left|\frac{\partial z_i}{\partial x_j} \right| \leq \begin{cases}
        \frac{1}{\sigma_{i}^L} &\text{if} \quad i=j \\
        \frac{1}{\sigma_i^L}C + \frac{|x_i| + C}{(\sigma_i^L)^2}\tilde{C}_1 &\text{if} \quad j<i\\
        0 &\text{if} \quad j >i.
    \end{cases} \quad \text{and}
\\ \label{eq:z-grad2-bound}
   &\left|\frac{\partial^2 z_i}{\partial x_j^2} \right| \leq \frac{1}{(\sigma_i^L)^2} C_1 +\frac{1}{(\sigma_i^L)} C_2 + \frac{1}{(\sigma_i^L)^2} \tilde{C}_1 + \frac{|x_i|+C}{(\sigma_i^L)^2} \tilde{C}_2.
\end{talign}
Substituting (\ref{eq:z-bound}) in (\ref{eq:grad-log-q}), it follows that for each $j = 1, \dots , d_{\mathcal{X}}$ there exist constants $C_3, C_4 < \infty$ depending on $C, C_1,\tilde{C}_1, \sigma_i^L$ such that:
\begin{talign*}
    \left| \frac{\partial}{\partial x_j} \log q_{\hat{\phi}_m} \right| \leq d_{\mathcal{X}} C_3 (1 + \|x\|_2) + d_{\mathcal{X}}C_4.
\end{talign*}
Setting $C_5:= \max\{d_{\mathcal{X}}C_3, d_{\mathcal{X}}C_4\}$ it follows from (\ref{eq:grad}) that:
\begin{talign} \label{eq:grad-bound}
    \left\| \nabla_{x} \log q_{\hat{\phi}_m}(x | \theta) \right\|_2 =  \left\| \sum_{i=1}^{d_{\mathcal{X}}} - z_i \frac{\partial z_i}{\partial x_j} - \frac{1}{\sigma_{\hat{\phi}_m,i}(x_{<i, \theta})} \frac{\partial \sigma_{\hat{\phi}_m,i}(x_{<i, \theta})}{\partial x_j} \right \|_2 \leq C_5 \left(1 + \|x\|^2_2\right).
\end{talign}
Similarly, by substituting the bounds in  (\ref{eq:z-bound}), (\ref{eq:z-grad2-bound}) into (\ref{eq:hess}) we obtain:
\begin{talign*}
    \left|\text{Tr} \nabla^2_x \log q_{\hat{\phi}_m}(x | \theta) \right| \leq C_6 \left(1+\|x\|_2^2\right) 
\end{talign*}
for some constant $C_6 < \infty$ depending on $d_{\mathcal{X}},C^\prime, C_1,C_2,\tilde{C}_1,\tilde{C}_2, \sigma_i^L$. 
To conclude, we have that:
\begin{talign*}
     \left\| \nabla_{x} \log q_{\hat{\phi}_m}(x | \theta) \right\|_2 \leq C_5 \left(1 + \|x\|^2_2\right) 
     ,\quad 
     \left|\text{Tr} \nabla_x^2 \log q_{\hat{\phi}_m}(x | \theta)\right| \leq d_{\mathcal{X}} C_6 \left(1 + \|x\|_2^2 \right)
\end{talign*}
hence for any $\zeta \leq 1$ there exist $C_7, C_8 < \infty$ such that for any $x \in \mathbb{R}^{d_\mathcal{X}}$: 
\begin{talign*}
     \left\| \nabla_{x} \log q_{\hat{\phi}_m}(x | \theta) \right\|_2 \leq C_7 \left(1 + \|x\|^{\frac{2}{\zeta}}_2\right) 
     , \quad 
     \left|\text{Tr} \nabla_x^2 \log q_{\hat{\phi}_m}(x | \theta)\right| \leq  C_8 \left(1 + \|x\|_2^{\frac{4}{\zeta}}\right).
\end{talign*}
and the statement holds by noting that $C_7$, and $C_8$ are independent of $\theta$ (hence polynomial functions of zero degree).
\end{proof}

\subsubsection{Proof of \Cref{theorem:robustness-exp}} \label{appendix:proof_robust_exp}
\begin{proof}
By Proposition \ref{prop:conjugacy} we have that $p_{\text{NSM}}(\theta|x^o_{1:n},\hat{\phi}_m) \propto \mathcal{N}(\theta;\mu_{n,m},\Sigma_{n,m})$ with $\mu_{n,m}, \Sigma_{n,m}$ as in (\ref{eq:closed-form-mu}) and (\ref{eq:closed-form-sigma}).
Similarly, we have that $p_{\text{NSM}}(\theta | x^c_{1:n}, \hat{\phi}_m) \propto \mathcal{N}(\theta; \mu_{n,m}^c, \Sigma_{n,m}^c)$ where: 
\begin{talign*}
    \left(\Sigma_{n,m}^c\right)^{-1} &:= \Sigma^{-1} + 2 \beta n M_{1}(x^c_{1:n}) \quad \text{and} \\
    \mu_{n,m}^c &:= \Sigma_{n,m}^c\left[ \Sigma^{-1}\mu - 2\beta n M_2(x^c_{1:n})  - 2\beta n \ell_1(x^c_{1:n}) \right]. 
\end{talign*}
Let $\mu^{-}_{n,m} := (\mu_{n,m}^c - \mu_{n,m})$. Using the closed form of the KL divergence between Gaussian distributions \citep[see e.g.][]{zhang2023properties}, we hence have: 
\begin{talign*}
    \pif_{\KL}(x_j^c, \mathbb{P}_n, \mathcal{L}_{\text{NSM}}) = \frac{1}{2} \Big( 
     \underbrace{\text{Tr}\left((\Sigma_{n,m}^c)^{-1} \Sigma_{n,m} \right) - d_{\Theta}}_{(1)} + 
        \underbrace{(\mu^{-}_{n,m})^\top (\Sigma_{n,m}^c)^{-1} \mu^{-}_{n,m}}_{(2)} + \underbrace{\log\left(\frac{\det(\Sigma^c_{n,m})}{\det(\Sigma_{n,m})}\right)}_{(3)}
   \Big).
\end{talign*}
For the first term, we have that: 
\begin{talign*}
    (1) = \text{Tr}\left((\Sigma_{n,m}^c)^{-1} \Sigma_{n,m} \right) - d_{\Theta} & \leq \text{Tr}\left((\Sigma_{n,m}^c)^{-1} \right) \text{Tr}(\Sigma_{n,m}) \\
    &= \text{Tr}\left( \Sigma^{-1} + 2 \beta n M_{1}(x^c_{1:n})\right) \text{Tr}\left( (\Sigma^{-1} + 2 \beta n M_{1}(x^o_{1:n}) )^{-1} \right) 
\end{talign*}
where in the inequality we used the fact that $\text{Tr}(AB) \leq \text{Tr}(A)\text{Tr}(B)$ for two PSD matrices $A$ and $B$ and that $d_{\Theta} > 0$. Since $\Sigma^{-1} + 2 \beta n M_{1}(x^o_{1:n})$ does not depend on the contamination datum $x_j^c$, we write $C_1 := \text{Tr}\left( (\Sigma^{-1} + 2 \beta n M_{1}(x^o_{1:n}) )^{-1}\right)$ and we have:
\begin{talign*}
    (1) \leq \text{Tr}\left( \Sigma^{-1} + 2 \beta n M_{1}(x^c_{1:n})\right) C_1  &= (\text{Tr}(\Sigma^{-1}) + \text{Tr}(2 \beta n M_{1}(x^c_{1:n})))C_1  \\
    &= \left(\text{Tr}(\Sigma^{-1}) + 2 \beta \sum_{i \neq j}^n w(x_i^o)^2 \text{Tr}\left(\nabla_x T_{\hat{\phi}_m}(x_i^o)  \nabla_x T_{\hat{\phi}_m}(x_i^o)^\top\right) \right. \\
    &\left. \quad + 2 \beta  w(x_j^c)^2 \text{Tr}\left(\nabla_x T_{\hat{\phi}_m}(x_j^c)  \nabla_x T_{\hat{\phi}_m}(x_j^c)^\top\right)
    \right)C_1.
\end{talign*}
The first two terms in the bracket do not depend on $x_j^c$ so let 
\begin{talign*}
C_2 := \text{Tr}(\Sigma^{-1}) + 2 \beta \sum_{i \neq j}^n w(x_i^o)^2 \text{Tr}\left(\nabla_x T_{\hat{\phi}_m}(x_i^o)  \nabla_x T_{\hat{\phi}_m}(x_i^o)^\top\right).
\end{talign*}
Then by noting that $\text{Tr}\left(\nabla_x T_{\hat{\phi}_m}(x)  \nabla_x T_{\hat{\phi}_m}(x)^\top\right) = \| \nabla_x T_{\hat{\phi}_m}(x) \|_F^2$ we have: 
\begin{talign*}
    (1) &\leq \left(C_2 + 2 \beta w(x_j^c)^2 \| \nabla_x T_{\hat{\phi}_m}(x_j^c) \|_F^2 \right) C_1  \leq \left( C_2 + 2 \beta \sup_{x} w(x)^2 \| \nabla_x T_{\hat{\phi}_m}(x) \|_F^2 \right) C_1.
\end{talign*}
For the second term, since $(\Sigma_{n,m}^c)^{-1}$ is symmetric we have \citep[see e.g.][Lemma 4.2.2]{horn2012matrix}:
\begin{talign*}
    (2) &= (\mu^{-}_{n,m})^\top \left(\Sigma_{n,m}^c\right)^{-1} \mu^{-}_{n,m} \leq \lambda_{\text{max}}\left((\Sigma_{n,m}^c)^{-1}\right) \left\| \mu^{-}_{n,m}\right\|_2^2
\end{talign*}
where $\lambda_{\text{max}}\left((\Sigma_{n,m}^c)^{-1}\right)$ denotes the maximum eigenvalue of $(\Sigma_{n,m}^c)^{-1}$. Then, using Weyl's inequality \citep[][Theorem 4.3.1]{horn2012matrix} we obtain:
\begin{talign*}
\lambda_{\text{max}}(\Sigma_{n,m}^c)^{-1}) &= \lambda_{\text{max}}(\Sigma^{-1} + 2 \beta n M_{1}(x^c_{1:n})) \leq \lambda_{\text{max}}(\Sigma^{-1}) + \lambda_{\text{max}}\left(2 \beta n M_{1}(x^c_{1:n})\right). 
\end{talign*}
Since $2 \beta M_{1}(x^c_{1:n})$ is PSD we have that $\lambda_{\text{max}}(2 \beta M_{1}(x^c_{1:n})) \leq \text{Tr}(2 \beta M_{1}(x^c_{1:n}))$ (since the trace is equal to the sum of its eigenvalues). Hence, by letting $$C_3 := \lambda_{\text{max}}(\Sigma^{-1}) + 2 \beta \sum_{i \neq j}^n w(x_i^o)^2 \nabla_x T_{\hat{\phi}_m}(x_i^o)  \nabla_x T_{\hat{\phi}_m}(x_i^o)^\top$$ which does not depend on $x_j^c$, and using the same arguments as above we have:
\begin{talign*}
    \lambda_{\text{max}}((\Sigma_{n,m}^c)^{-1}) \leq C_3 + 2 \beta \sup_{x} w(x)^2 \left\| \nabla_x T_{\hat{\phi}_m}(x) \right\|_F^2 := C_4.
\end{talign*}
Substituting into $(2)$ we obtain:
\begin{talign*}
    (2) \leq C_4 \left\| \mu^-_{n,m} \right\|_2^2  
    &= C_4 \left\| \Sigma_{n,m}^c\left[ \Sigma^{-1}\mu - 2\beta n M_{2}(x^c_{1:n})  - 2\beta n l_{1}(x^c_{1:n}) \right] \right.\\
    & \quad \quad \left. -
    \Sigma_{n,m}\left[ \Sigma^{-1}\mu - 2\beta n M_{2}(x^o_{1:n})  - 2\beta n l_{1}(x^o_{1:n}) \right]
    \right \|_2^2.
\end{talign*}
Now let $v^c := - 2 \beta [d(x_j^c) - d(x_j) + h(x_j^c) - h(x_j)]$ and $A^c := 2 \beta [M(x_j^c) - M(x_j)]$ where 
\begin{align*}
    h(x) &:= w(x)^2 \nabla_x T_{\hat{\phi}_m}(x) \nabla_x b_{\hat{\phi}_m}(x) \in \mathbb{R}^{d_\Theta},  \quad
    d(x) := \nabla_x \cdot \left(w(x)^2 \nabla_x T_{\hat{\phi}_m}(x)^\top\right)^\top \in \mathbb{R}^{d_{\Theta}}, \quad \text{and} \\
    M(x) &= w(x)^2 \nabla_x T_{\hat{\phi}_m}(x) \nabla_x T_{\hat{\phi}_m}(x)^\top \in \mathbb{R}^{d_{\Theta} \times d_{\Theta}}.
\end{align*}
By re-arranging terms, we have:
\begin{talign*}
    \left\| \mu^-_{n,m} \right\|_2 &= \left\| \Sigma_{n,m}^c\left[ \Sigma^{-1}\mu - 2\beta n M_{2}(x^c_{1:n})  - 2\beta n l_{1}(x^c_{1:n}) \right] \right.\\
    & \quad \quad \left. -
    \Sigma_{n,m}\left[ \Sigma^{-1}\mu - 2\beta n M_{2}(x^o_{1:n})  - 2\beta n l_{1}(x^o_{1:n}) \right]
    \right \|_2 \\
    &= \left\| \Sigma_{n,m}^c v^c - \Sigma_{n,m}^c A^c \mu_{n,m} \right\|_2 \\
    &\leq \left\| \Sigma_{n,m}^c \right\|_{2} \left\| v^c - A^c \mu_{n,m} \right\|_2.
\end{talign*}
Since $(\Sigma_{n,m}^c)^{-1} = \Sigma^{-1} + 2 \beta n M_{1}(x^c_{1:n})$ and $2 \beta n M_{1}(x^c_{1:n})$ is PSD we have $(\Sigma_{n,m}^c)^{-1} \succeq \Sigma^{-1}$ and hence $\Sigma \succeq \Sigma_{n,m}^c$. For symmetric PSD matrices $A,B$, $A \succeq B$ implies $\lambda_{\text{max}}(A) \geq \lambda_{\text{max}}(B)$ from Weyl's theorem, and $\| A \|_2 = \lambda_{\text{max}}(A)$ hence since $\Sigma, \Sigma_{n,m}^c$ are symmetric PSD matrices:
\begin{talign*}
    \left\| \Sigma_{n,m}^c \right\|_2 = \lambda_{\text{max}}\left(\Sigma_{n,m}^c\right) \leq \lambda_{\text{max}}(\Sigma) = \| \Sigma \|_2. 
\end{talign*}
Now consider the term $\| v^c - A^c \mu_{n,m} \|_2$. Since $C_6 := \|\mu_{n,m}\|_2$ does not depend on $x_j^c$ we have,
\begin{talign*}
    \left\| A^c \mu_{n,m} \right\|_2 &\leq C_6 \| A^c \|_2 = C_6 2 \beta \left\| M(x^c_j) - M(x_j) \right\|_2. 
\end{talign*}
Next, we have
\begin{talign*}
    \| v^c \|_2 = \left\| - 2 \beta \left[d(x_j^c) - d(x_j) + h(x_j^c) - h(x_j)\right] \right\|_2 
    &\leq 2 \beta \left(\|d(x_j^c) \|_2 + \| d(x_j)\|_2 + \| h(x_j^c)\|_2 + \| h(x_j)\|_2\right) \\
    &\leq C_5 + 2 \beta \left(\| d(x_j^c)\|_2 + \| h(x_j^c) \|_2\right) 
\end{talign*}
where the first inequality follows from the triangle inequality and the second inequality follows from defining $C_5 := 2 \beta (\| d(x_j) \|_2 + \|h(x_j)\|_2)$, independent of $x_j^c$. Now,
\begin{talign*}
    \| h(x_j^c) \|_2 = w(x_j^c)^2 \left\|\nabla_x T_{\hat{\phi}_m}(x_j^c) \nabla_x b_{\hat{\phi}_m}(x_j^c) \right\|_2 
    &\leq  w(x_j^c)^2 \left\|\nabla_x T_{\hat{\phi}_m}(x_j^c)\right\|_2 \left\| \nabla_x b_{\hat{\phi}_m}(x_j^c) \right\|_2 \\
    &\leq w(x_j^c)^2 \left\|\nabla_x T_{\hat{\phi}_m}(x_j^c) \right\|_F \left\| \nabla_x b_{\hat{\phi}_m}(x_j^c) \right\|_2. 
\end{talign*}
On the other hand, for $k \in \{1, \dots, d_{\Theta}\}$, the $k$-th entry of $d(x)$ has the form:
\begin{talign*}
    (d(x))_k &= \sum_{j=1}^{d_{\mathcal{X}}} \frac{\partial}{\partial x_j} \left(w(x)^2 (\nabla_x T_{\hat{\phi}_m}(x))_{kj}\right) = \nabla w(x)^2 \cdot \left(\nabla_x T_{\hat{\phi}_m}(x)\right)_{k,:}^\top + w(x)^2 \text{Tr}\left(\nabla^2_x (T_{\hat{\phi}_m}(x))_k\right).
\end{talign*}
Let $d_1(x) \in \mathbb{R}^{d_{\Theta}}$ be such that $(d_1(x))_k := \nabla w(x)^2 \cdot (\nabla_x T_{\hat{\phi}_m}(x))_{k,:}^\top$ and $d_2(x) \in \mathbb{R}_{d_{\Theta}}$ be such that $(d_2(x))_k := w(x)^2 \text{Tr}(\nabla^2_x (T_{\hat{\phi}_m}(x))_k)$ for $k \in \{1, \dots, d_{\Theta}\}$. Then $d(x) = d_1(x) + d_2(x)$. By the triangle inequality it follows that
$
    \| d(x_j^c) \|_2 \leq \| d_1(x_j^c) \|_2 + \| d_2(x_j^c)\|_2. 
$
By the Cauchy-Schwartz inequality, we have:
\begin{talign*}
    \left\| d_1(x_j^c) \right\|_2 = \left(\sum_{k=1}^{d_{\Theta}} \left(d_1(x_j^c)\right)_k^2 \right)^{\frac{1}{2}} &\leq \left\| \nabla w(x_j^c)^2 \right\|_2 \left(\sum_{k=1}^{d_{\Theta}} \left\|\left(\nabla_x T_{\hat{\phi}_m}(x)\right)_{k,:} \right\|_2^2 \right)^{\frac{1}{2}}  \\
    &= \left\| \nabla w(x_j^c)^2 \right\|_2 \left\| \nabla_x T_{\hat{\phi}_m}(x_j^c) \right\|_F.
\end{talign*}
Similarly, for the second term, we have, 
\begin{talign*}
    \| d_2(x_j^c) \|_2 = \left(\sum_{k=1}^{d_{\Theta}} \left(w(x_j^c)^2 \text{Tr}(\nabla^2_x (T_{\hat{\phi}_m}(x_j^c))_k)\right)^2 \right)^{\frac{1}{2}}
    &\leq w(x_j^c)^2 \sqrt{\sum_{k=1}^{d_{\Theta}} d_{\mathcal{X}} \left\|\nabla^2_x (T_{\hat{\phi}_m}(x_j^c))_k\right\|^2_F} \\
    &=  w(x_j^c)^2 \sqrt{d_{\mathcal{X}}} \left\| \nabla^2_x T_{\hat{\phi}_m}(x_j^c) \right\|_F.
\end{talign*}
Putting these together, we obtain an upper bound for $\| d(x_j^c) \|_2$ as follows:
\begin{talign*}
    \left\| d(x_j^c) \right\|_2 \leq \left\| \nabla w(x_j^c)^2 \right\|_2 \left\| \nabla_x T_{\hat{\phi}_m}(x_j^c) \right\|_F + \sqrt{d_{\mathcal{X}}} w(x_j^c)^2  \left\| \nabla^2_x T_{\hat{\phi}_m}(x_j^c) \right\|_F. 
\end{talign*}
And therefore, by defining $C_7 := \| \Sigma \|_2^2$, $C_8 := 4 \beta C_6$ and $C_9 := d_{\mathcal{X}}$ we obtain:
\begin{talign*}
    (2) \leq C_4 \left\| \mu^{-}_{n,m} \right\|_2^2 
    &\leq \left\| \Sigma_{n,m}^c \right\|_2^2 \left\| v^c - A^c \mu_{n,m} \right\|_2^2 \\
    &\leq \| \Sigma \|_2^2 \left(w(x_j^c)^2 \left\|\nabla_x T_{\hat{\phi}_m}(x_j^c)\right\|_F \left\| \nabla_x b_{\hat{\phi}_m}(x_j^c) \right\|_2  + \left\| \nabla w(x_j^c)^2 \right\|_2 \left\| \nabla_x T_{\hat{\phi}_m}(x_j^c) \right\|_F \right.\\
    & \left. \quad + \sqrt{d_{\mathcal{X}}} w(x_j^c)^2  \left\| \nabla^2_x T_{\hat{\phi}_m}(x_j^c) \right\|_F + C_6 2 \beta \left\| M(x_j^c) - M(x_j) \right\|_2 \right)^2 \\
    &\leq C_7 \left(\sup_x w(x)^2 \left\|\nabla_x T_{\hat{\phi}_m}(x)\right\|_F \left\| \nabla_x b_{\hat{\phi}_m}(x) \right\|_2  + \sup_x \left\| \nabla w(x)^2 \right\|_2 \left\| \nabla_x T_{\hat{\phi}_m}(x) \right\|_F \right.\\
    &\left. \quad + C_9 \sup_x w(x)^2  \left\| \nabla^2_x T_{\hat{\phi}_m}(x) \right\|_F + C_8 \sup_x w(x)^2 \left\| \nabla_x T_{\hat{\phi}_m}(x) \right\|_F^2 \right)^2
\end{talign*}
Finally, for the third term, we have:
\begin{talign*}
    (3) &= \log\left(\frac{\det(\Sigma^c_n)}{\det(\Sigma_{n,m})}\right) = \log \left(\frac{\det\left((\Sigma^{-1} + 2 \beta n M_{1}(x^c_{1:n}))^{-1}\right)}{\det\left((\Sigma^{-1} + 2 \beta n M_{1}(x^o_{1:n}))^{-1}\right)} \right), \quad \text{and} \\ (\Sigma_{n,m}^c)^{-1} &= \underbrace{\Sigma_{n,m}^{-1} - 2 \beta n M(x_j)}_{:= C(x^o_{1:n})} + \underbrace{2 \beta n M(x_j^c)}_{:= C(x_j^c)} 
\end{talign*}
where
$
C(x^o_{1:n}) = \Sigma^{-1} + 2 \beta n M_{1}(x_{1:n}^o) - 2 \beta n M(x_j) = \Sigma^{-1} + 2 \beta \sum_{i=1, i \neq m}^n w(x_i^o)^2 \nabla_x T_{\hat{\phi}_m}(x_i^o) \nabla_x T_{\hat{\phi}_m}(x_i^o)^\top.
$
Since $\Sigma^{-1}$ is positive definite and $M(x)$ is positive semi-definite for all $x$, then $C(x^o_{1:n})$ is the sum of a positive semi-definite and a positive definite matrix and is hence positive definite. Moreover, $C(x_j^c) := 2 \beta n M(x_j^c)$ and is positive semi-definite. 
Let $C(x_j) := 2 \beta n M(x_j)$ then
\begin{talign*}
    \Sigma_{n,m}^{-1} = C(x^o_{1:n}) + C(x_j) = C(x^o_{1:n})^{1/2} (I + C(x^o_{1:n})^{-1/2} C(x_j) C(x^o_{1:n})^{-1/2})C(x^o_{1:n})^{1/2}
 \end{talign*}
and hence, 
\begin{talign*}
    \det\left(\Sigma_{n,m}^{-1}\right) &= \det\left(C(x^o_{1:n})^{1/2}\right) \det\left(I + C(x^o_{1:n}\right)^{-1/2} C(x_j) C(x^o_{1:n})^{-1/2}) \det\left(C(x^o_{1:n})^{1/2}\right)\\
    &= \det\left(C(x^o_{1:n})\right)  \det\left(I + C(x^o_{1:n})^{-1/2} C(x_j) C(x^o_{1:n})^{-1/2}\right).
\end{talign*}
Therefore, 
\begin{talign*}
    (3) = \log \left(\frac{\det(\Sigma_{n,m}^{-1})}{\det\left(C(x^o_{1:n}) + C(x_j^c)\right)} \right) 
    &= \log \left(\det(I + C(x^o_{1:n})^{-1/2} C(x_j) C(x^o_{1:n})^{-1/2}) \frac{\det(C(x^o_{1:n})) }{\det\left(C(x^o_{1:n}) + C(x_j^c)\right)} \right) \\
    &= \log\left(C_{10} \frac{\det(C(x^o_{1:n})) }{\det\left(C(x^o_{1:n}) + C(x_j^c)\right)} \right) 
\end{talign*}
where we defined $C_{10} := \det(I + C(x^o_{1:n})^{-1/2} C(x_j) C(x^o_{1:n})^{-1/2})$ as it does not depend on the contamination. Finally, $\det(C(x^o_{1:n})) > 0$ since it is the sum of two matrices, including a PD one ($\Sigma^{-1}$), hence since $C(x^0_{1:n}) + C(x_j^c) \succeq C(x^0_{1:n})$ we have that $\det(C(x^0_{1:n}) + C(x_j^c)) \geq \det(C(x^0_{1:n}))$ by monotonicity of the determinant and hence:
\begin{talign*}
    \frac{\det(C(x^o_{1:n})) }{\det\left(C(x^o_{1:n}) + C(x_j^c)\right)} \leq 1
\end{talign*}
and $(3) \leq \log C_{10} := C_{11}$.
Therefore, under the assumptions of the Theorem, we have that:
\begin{talign*}
    \sup_{x_j^c \in \mathbb{R}^{d_{\mathcal{X}}}} \pif_{\KL}(x_j^c, \mathbb{P}_n, \mathcal{L}_{\text{NSM}}) \leq \frac{1}{2} \left(
    C_2 + 2 \beta C^2) C_1 +
    C_7 (C^2  + C^\prime
     + C_9 C^2 + C_8 C^2)^2 
     + C_{11}
     \right) < \infty.
\end{talign*}
\end{proof}

\subsubsection{Connection between Assumptions \ref{assumption:general-pif} and \ref{assumption:kl-pif}} \label{appendix:assumptions-connection}
In this section, we prove the following Lemma which shows that under some simple additional conditions robustness in the $\text{PIF}_{\text{KL}}$ sense yields robustness in the $\text{PIF}_{\Delta}$ sense for NSM-Bayes-conj. 
\begin{lemma}
    Suppose Assumptions \ref{assumption:weight}, \ref{assumption:exptheta} and \ref{assumption:kl-pif} hold and additionally that there exist $\tilde{C}_1, \tilde{C}_2 < \infty$ such that
 $
        \left \|\nabla_x b_{\hat{\phi}_m}(x) \right\|_2  \leq \tilde{C}_1 \min\{w(x)^{-1}, \|\nabla_x w(x)^2\|_2^{-1}\}$, $
        \left \|\nabla^2_x b_{\hat{\phi}_m}(x) \right\|_F \leq  \tilde{C}_2 w(x)^{-2}
$.
Then, NSM-Bayes-conj is globally-bias-robust:  $\sup_{\theta \in \Theta, x^c \in \mathcal{X}}  \pif(x^c, \theta, \mathbb{P}_n, \mathcal{L}_{\text{NSM}})  < \infty$.
\end{lemma}
\begin{proof}
From Assumption \ref{assumption:exptheta} we have that 
$
q_{\hat{\phi}_m}(x | \theta) \propto \exp( 
T_{\hat{\phi}_m}(x)^\top \theta + b_{\hat{\phi}_m}(x) 
),
$
hence 
$
    \nabla_x \log q_{\hat{\phi}_m}(x | \theta) = (\nabla_{x} T_{\hat{\phi}_m}(x))^\top \theta + \nabla_x b_{\hat{\phi}_m}(x)
$
and under the assumptions:
\begin{talign*}
    \left\| \nabla_x \log q_{\hat{\phi}_m}(x | \theta) \right\|_2 \leq \left\| (\nabla_{x} T_{\hat{\phi}_m}(x))^\top \theta \right\|_2 + \left\| \nabla_x b_{\hat{\phi}_m}(x) \right\|_2 
    &\leq 
    \left(C \| \theta\|_2 + \tilde{C}_1\right) w(x)^{-1}.
\end{talign*}
and simultaneously (from the second bound of Assumption \ref{assumption:kl-pif}),
\begin{talign*}
     \left\| \nabla_x \log q_{\hat{\phi}_m}(x | \theta) \right\|_2 \leq \left\| (\nabla_{x} T_{\hat{\phi}_m}(x))^\top \theta \right\|_2 + \left\| \nabla_x b_{\hat{\phi}_m}(x) \right\|_2 
    &\leq (C^\prime \|\theta\|_2 + \tilde{C}_1) \left\| \nabla_x w(x)^2\right\|_2^{-1}.
\end{talign*}
Hence choosing $C_3 := \max\{C, C^\prime\}$ we have that 
\begin{talign*}
    \left\| \nabla_x \log q_{\hat{\phi}_m}(x | \theta) \right\|_2 \leq \left(C_3 \|\theta\|_2 + \tilde{C}_1 \right) \min\{w(x)^{-1}, \|\nabla_xw(x)^2\|_2^{-1}\}.
\end{talign*}
Similarly, 
\begin{talign*}
    \left|\text{Tr} \nabla_x^2 \log q_{\hat{\phi}_m}(x | \theta) \right| = \left|\text{Tr} \left(
    \nabla_x^2 T_{\hat{\phi}_m}(x)^\top \theta + \nabla_x^2 b_{\hat{\phi}_m}(x)
    \right) \right| &\leq 
    \left\| \nabla_x^2 T_{\hat{\phi}_m}(x) \right\|_F \|\theta\|_2 + \left\| \nabla_x^2 b_{\hat{\phi}_m}(x) \right\|_F \\
    &\leq C^2 w(x)^{-2} \|\theta\|_2 + \tilde{C}_2 w(x)^{-2} \\
    &= \left(C^2 \|\theta\|_2 + \tilde{C}_2\right) w(x)^{-2}.
\end{talign*}
Take $f(\theta) := C_3 \| \theta\|_2 + \tilde{C}_1$ and $g(\theta) := C^2 \|\theta\|_2 + \tilde{C}_2$. Recall that for NSM-Bayes-conj we select $\pi(\theta) \propto \mathcal{N}(\theta;\mu,\Sigma)$. Since both $f,g$ have linear growth in $\theta$ and $\pi$ has finite second moments it follows that $f \in L^2(\pi), g \in L^1(\pi)$. Moreover, since $\pi$ has exponential decay in $\theta$ and $f$ and $g$ have linear growth in $\theta$ it follows that $\sup_{\theta \in \Theta} \left(f(\theta)^2 + f(\theta) +  g(\theta) \right) \pi(\theta) < \infty$.
Therefore, Assumption \ref{assumption:general-pif} holds
and global-bias robustness follows from Theorem \ref{theorem:robustness}.
\end{proof}

\section{Experimental details and additional results}
\label{app:experimental-details}

\Cref{app:setting-learning-rate} provides further details on the algorithm we use to tune the learning rate in non-conjugate GBI methods such as NSM-Bayes.
Implementation details for our methods and the baselines are in \Cref{app:implementation} and \Cref{app:baselines}, respectively. \Cref{app:additional_experiments} contains the additional experiments and results from \Cref{sec:experiments}. In \Cref{app:add_exp_baselines}, we analyse of the performance of certain baseline methods under increased simulation budget.

\subsection{Setting the learning rate for non-conjugate GBI posteriors}
\label{app:setting-learning-rate}

The \citet{Syring2019} method for tuning the learning rate outlined in \Cref{alg:learning-rate-calibration-conjugate} algorithm requires the posterior (mean and covariance) to be computed $T \times B$ times, where $B$ is the number of bootstrap samples and $T$ is the number of stochastic steps. This is straightforward for NSM-Bayes-conj where the posterior is conjugate. However, for most GBI methods such as NSM-Bayes, running MCMC hundreds of times to tune $\beta$ can be infeasible. Therefore, we first run MCMC once at an initial value of $\beta$, and use importance sampling (IS) to estimate the posterior for each bootstrapped sample. We also monitor the effective sample size (ESS) of the importance weights at each step and re-run MCMC with the current value of $\beta$ if the ESS falls below 30\%. The full algorithm is outlined in \Cref{alg:nsm-lr-calibration}.

\begin{algorithm}[t]
\caption{Learning-rate calibration for GBI methods via bootstrap re-weighting}
\label{alg:nsm-lr-calibration}
\begin{algorithmic}[1]
\Require Data $x_{1:n}^o$, initial $\beta_0$, iterations $T$, bootstraps $B$, MCMC draws $M$, step-sizes $\{\kappa_t\}$, target $1-\alpha$, ESS threshold $\texttt{ESS}_{\min}$
\Ensure Calibrated learning rate $\beta^\star$

\State Compute $\hat\theta_n = \arg\min_{\theta} \mathcal{L}(\theta;x_{1:n}^o)$ using gradient-based solvers
\State Initialise $\beta \gets \beta_0$; draw $\{\theta^{(i)}\}_{i=1}^M \sim \pi_{\mathcal{L}}^{\beta}(\theta | x_{1:n}^o)$ via MCMC
\State Cache per-observation losses $\mathcal{L}_{ij} \gets \mathcal{L}(\theta^{(i)};x_j^o)$ for all $i=1{:}M$, $j=1{:}n$

\For{$t=1$ \textbf{to} $T$}
    \State Draw bootstrap counts $N_b \sim \text{Multinomial}\!\left(n;\frac1n,\ldots,\frac1n\right)$ for $b=1{:}B$
    \For{$b=1$ \textbf{to} $B$}
        \State Compute self-normalised IS weights $\{W_{b,i}(\beta)\}_{i=1}^M$ using cached $\mathcal{L}_{ij}$ and counts $N_b$
        \State Compute weighted mean/covariance $(\mu_b^{(\beta)},\Sigma_b^{(\beta)})$ from $\{(\theta^{(i)},W_{b,i}(\beta))\}_{i=1}^M$
        \State Compute $D^2_{b,i} = (\theta^{(i)} - \hat \mu_{\beta_t}^{(b)})^\top (\Sigma_{\beta_t}^{(b)})^{-1} (\theta^{(i)} - \hat \mu_{\beta_t}^{(b)})$ and 
        \State Form the credible region: $C_{\alpha,\beta_t}\big(x_{1:n}^{\star(b)}\big) = \Big\{\theta:\;D^2_{b,i} \le \tau_b(\beta_t)\Big\}$ for threshold $\tau_b(\beta)$ (weighted $(1-\alpha)$-quantile of $\{D^2_{b,i}\}_{i=1}^M$ under weights $W_{b,\cdot}(\beta_t)$)
        \State Set indicator $\mathbb{I}_b(\beta)\gets \mathbf{1}\!\left\{(\hat\theta_n-\mu_b^{(\beta)})^\top(\Sigma_b^{(\beta)})^{-1}(\hat\theta_n-\mu_b^{(\beta)})\le \tau_b(\beta)\right\}$
    \EndFor
    \State Estimate coverage: $\hat c(\beta)\gets \frac{1}{B}\sum_{b=1}^B \mathbb{I}_b(\beta)$
    \State Update $\beta \gets \beta + \kappa_t\big(\hat c(\beta)-(1-\alpha)\big)$
    \If{mean ESS across bootstraps $<\texttt{ESS}_{\min}$}
        \State Re-run MCMC at current $\beta$ to obtain $\{\theta^{(i)}\}_{i=1}^M$ and refresh cache $\mathcal{L}_{ij}$
    \EndIf
\EndFor
\State \Return $\beta^\star \gets \beta$
\end{algorithmic}
\end{algorithm}

 \paragraph{Closed-form estimation of $\hat \theta_n$ for NSM-Bayes-conj.} 
In the case where $q_\phi$ is an energy-based model satisfying \Cref{assumption:exptheta}, we can solve $\hat \theta_n = \arg \min_\theta \mathcal{L}_{\text{NSM}}(\theta;x_{1:n}^o, \hat{\phi}_m) $ in closed-form as the objective is quadratic in $\theta$. Recall the loss expression:
\begin{align*}
    \mathcal{L}_{\text{NSM}}\left(\theta;x_{1:n}^o, \hat{\phi}_m \right) &= \theta^\top A_n \theta + 2\theta^\top B_n + C_n, \quad \text{where}\\
    A_n &= \frac{1}{n} \sum_{i=1}^n w(x_i^o)^2 \nabla_x T_{\hat{\phi}_m}(x_i^o) \nabla_x T_{\hat{\phi}_m}(x_i^o)^\top,\\
    B_n &= \frac{1}{n} \sum_{i=1}^n \left[  w(x_i^o)^2 \nabla_x T_{\hat{\phi}_m}(x_i^o) \nabla_x b_{\hat{\phi}_m}(x_i^o) + \nabla_x \cdot \big(w(x_i^o)^2 \nabla_x T_{\hat{\phi}_m}(x_i^o)^\top\big) ^\top\right],
\end{align*}
and $C_n$ does not depend on $\theta$. If $A_n$ is positive definite, the minimiser of the optimisation problem is $\hat \theta_n = -A_n^{-1}B_n$. However, in practice, especially for low-dimensional data, the matrix $A_n$ can be severely ill-conditioned or nearly singular. To ensure numerical stability, we compute a ridge-stabilised estimator:
\begin{align*}
    \hat{\theta}_0^{(\lambda)} := \arg \min_\theta \left \{\mathcal{L}_{\text{NSM}}\left(\theta;x_{1:n}^o, \hat{\phi}_m \right) + \frac{\lambda}{2} \Vert \theta \Vert_2^2\right\} = -(A_n + \lambda I_{d_\Theta})^{-1}B_n.
\end{align*}
For all the experiments, we set the ridge parameter as $\lambda = 10^{-2}  \text{Tr}(A_n)/ d_\Theta + 10^{-12}$.
In principle, one could use an adaptive scheme which would compute eigenvalues of $A_n$ and then set $\lambda$ based on a target condition number $\kappa_{\text{target}}$, for example, $\lambda \geq \lambda_{\text{max}}(A_n) - \kappa_{\text{target}}\lambda_{\text{min}}(A_n)/(\kappa_{\text{target}} - 1)$.

\paragraph{Selecting $\beta$ for ACE and GBI-SR.}

Since neither \citet{Gao2023} or \citet{Pacchiardi2024} provide a method to select $\beta$ for their methods, we use an approach similar to that described in \citet{Syring2019} that we use for our methods. As a direct application of \Cref{alg:learning-rate-calibration-conjugate} or \Cref{alg:nsm-lr-calibration} to ACE and GBI-SR would require a substantially larger number of simulations than the fixed budget used in our experiments, we employ the method outlined in \Cref{alg:baseline-lr-calibration}.
We set $T = 200$ and $B = 20$ for the both methods and set $M = N_\text{post}$ for ACE, while for GBI-SR we select according to available simulation budget. We set $\beta_\text{list} = \{1, 10, 100, 1000, 10000\}$ and $\beta_\text{list} = \{1, 10, 100\}$ for ACE and GBI-SR, respectively. The grid of $\beta$ values is less for GBI-SR due to limited simulation budget and numerical issues encountered for $\beta = 1000, 10000$. As the SIR model is non-differentiable, estimating $\hat \theta_n$ using gradient-based methods is infeasible for GBI-SR. We therefore replace $\hat{\theta}_n$ with the posterior mean obtained using the initial value $\beta_0 = 1$, which requires additional simulation.

\begin{algorithm}[t]
\caption{Learning-rate calibration for ACE and GBI-SR}
\label{alg:baseline-lr-calibration}
\begin{algorithmic}
\Require Data $x_{1:n}^o$, bootstraps $B$, iterations $T$, learning rates in descending order $\beta_{\text{list}}$, MCMC draws $M$.
\State $\hat\theta_n \gets \arg\min_{\theta} \mathcal{L}(\theta; x_{1:n}^o)$, $\beta^\star \gets \text{None}$, $\Delta^\star \gets +\infty$
\For{$\beta \in \beta_{\text{list}}$}
    \State Draw $B$ bootstraps datasets $x_{1:n}^{\star(b)}$ and obtain $M$ posterior draw from each
    \State Estimate coverage $\hat c(\beta)$ using \Cref{alg:learning-rate-calibration-conjugate} \\
    \qquad \textbf{if} $\lvert \hat c(\beta) - (1-\alpha) \rvert \leq \Delta^\star$ \textbf{then} $\Delta^\star \gets \lvert \hat c(\beta) - (1-\alpha) \rvert$, $\beta^\star \gets \beta$
    \textbf{end if }
\EndFor
\State \Return $\beta^\star$
\end{algorithmic}
\end{algorithm}

\subsection{Implementation details}
\label{app:implementation}

For all the experiments with NSM-Bayes-conj, we parameterise the exponential family density estimator components $T_\phi$ and $b_\phi$ using fully connected two-layer multilayer perceptrons with a single hidden layer of width 128, a tanh activation function, and linear output heads. Since tanh is bounded, the resulting $T_\phi(x)$ and $b_\phi(x)$ are bounded in $x$, which ensures that $q_\phi(x|\theta)$ is normalisable. All linear layers are initialised with Xavier/Glorot uniform weights, and biases are set to 0.01 to avoid near-zero initial activations. 
Following common practice, we standardise the simulated data to be zero-mean and identity covariance before training, and use a fixed learning rate of $5 \times 10^{-4}$, weight decay $10^{-5}$, batch size 128, and a maximum of 1000 epochs. The training dataset is randomly split into 80\% training and 20\% validation subsets, with the training set shuffled each epoch. We apply early stopping with patience 20 based on the validation loss, saving the model parameters with the best validation performance and restoring them at the end of training. For all the baseline methods, we use the default setting of the original code unless stated otherwise. This includes the MCMC sampler, the neural network architecture, and the associated hyperparameters; see \Cref{app:baselines} for more details.

For our methods in \Cref{sec:experiments}, we select the learning rate $\beta$ using the procedure described in \Cref{app:setting-learning-rate} with $\alpha=0.05$, which corresponds to 95\% coverage. We set the number of bootstrap samples to $B=100$, the number of stochastic approximation steps to $T=20$, and the adaptive step-size to $\kappa_t = 10/(t+10)$ in all the experiments. We use a log-transform for $\beta$ and set the initial learning rate $\beta_0$ for NSM-Bayes to 1.0, $10^{-6}$, and 0.01 for g-and-k, SIR and radio propagation models. The corresponding values for NSM-Bayes-conj are 0.1, 0.01, and 0.01. The learning rate is clamped at $\beta_0/100$ during the stochastic approximation. For the other GBI methods (ACE and GBI-SR), we select $\beta$ using a strategy inspired by the procedure in \Cref{app:setting-learning-rate}. We do not apply the exact same procedure, as it would require  significantly more simulations for both ACE and GBI-SR than the assigned budget, making the comparison unfair. Instead of the stochastic approximation steps, we use a grid of $\beta$ values; see \Cref{app:baselines}.

\paragraph{Performance metrics.} 
We use the MMD$^2$ as one of the performance metric. Given a reproducing kernel, the MMD$^2$ between distributions $\mathbb{Q}_1, \mathbb{Q}_2$ is:
\begin{talign*}
    \mathrm{MMD}^2(\mathbb{Q}_1, \mathbb{Q}_2)  = \mathbb{E}_{Y, Y' \sim \mathbb{Q}_1}[k(Y, Y')] - 2 \mathbb{E}_{Y \sim \mathbb{Q}_1,  Y'\sim \mathbb{Q}_2}[k(Y, Y')]  + \mathbb{E}_{Y, Y' \sim \mathbb{Q}_2}[k(Y, Y')].
\end{talign*}
We estimate this quantity using independent samples $y_{1:{n_1}} \sim \mathbb{Q}_1, \tilde y_{1:{n_2}} \sim \mathbb{Q}_2$: 
\begin{talign*}
    {\mathrm{MMD}}^2(y_{1:{n_1}}, \tilde y_{1:n_2}) :&= \frac{1}{n_1^2} \sum_{i,j=1}^{n_1}k(y_i, y_j) -  \frac{2}{n_1 n_2}\sum_{i=1}^{n_1}\sum_{j=1}^{n_2}k(y_i, \tilde y_j)  +  \frac{1}{n_2}\sum_{i,j = 1}^{n_2}k(\tilde y_i, \tilde y_j) 
\end{talign*}
Throughout, we use a Gaussian kernel defined as $k(y, y') = \exp\left(-\|y - y'\|^2_2/2 l^2\right)$ for $\text{MMD}^2_{\text{ref}}$ and wherever MMD is used in the baseline methods. Thus, $\mathrm{MMD}^2_\text{ref} := \mathrm{MMD}^2(\tilde\theta_{1:N_\text{post}}, \theta_{1: N_\text{post}})$,where $\tilde \theta_{1:N_\text{post}}$ are samples from the reference NLE posterior. The lengthscale $l$ is estimated using the median heuristic $l  = \sqrt{\text{median}(\|x_i - x_j\|^2 /2}$, $ i, j \in \{1, \ldots, n\})$. MSE is computed in closed-form for NSM-Bayes as $\text{MSE} = \Vert \mu_{n,m} - \theta^\star \Vert_2^2 + \text{tr}(\Sigma_{n,m})$, while for all the other methods, it is estimated using Monte Carlo samples from the posterior.
\subsection{Description of the baseline methods}
\label{app:baselines}

\paragraph{RSNLE \citep{Kelly2024}.}

This method considers model misspecification in the space of summary statistics $s \in \mathcal{S} \subseteq \mathbb{R}^{d_{S}}$ defined by a summary function $s: \mathcal{X}^n \to \mathcal{S}$. The model is said to be misspecified in the summary space if there exists 
no $\theta \in \Theta$ such that
$\mathbb{E}_{X_{1:n} \sim \mathbb{P{_\theta}}}[s(X_{1:n})] = \mathbb{E}_{X_{1:n}^o \sim \mathbb{P}_0}[s(X_{1:n}^o)]$. To address this issue, RSNLE introduces an auxiliary variable $\Gamma \in \mathbb{R}^{d_{\mathcal S}}$ that explicitly captures the discrepancy between the observed summary statistics $s(x^o_{1:n})$ and the simulated statistics $s(x_{1:n})$. Incorporating this discrepancy into the sequential NLE method consisting of $R$ rounds yields the augmented posterior:
\begin{talign*}
    \pi_{\mathrm{RSNLE}}\left(\theta, \Gamma \big| s(x^o_{1:n}), \hat\phi_m^R \right)
\propto
q_{\hat\phi_m^R}\!\left(s(x^o_{1:n}) - \Gamma | \theta\right)\,
\pi(\theta)\,
\pi^R \left(\Gamma | x^o_{1:n}\right).
\end{talign*}
The first round involves standard NLE training, yielding the estimator $q_{\hat\phi_m^1}$. At this stage, an independent Laplace prior is placed on each component of the discrepancy variable $\Gamma$. For subsequent rounds, the neural likelihood is retrained using parameters drawn from the posterior of the previous round, resulting in an updated estimator $q_{\hat\phi_m^{r+1}}$, $r=1, \ldots, R-1$. The discrepancy prior is updated according to $
\pi^r(\gamma_i | x^o_{1:n}) = \mathrm{Laplace}\!\left(0, \left|0.3\,\tilde s_i^r(x^o_{1:n})\right|\right)$,
where $\tilde s_i^r$ denotes the standardized $i$-th summary statistic, with mean and standard deviation computed over the accumulated simulated data
$\{\{x^{(r')}_{1:n,i}\}_{i=1}^m\}_{r'=0}^r$. Following the original implementation, we use a RealNVP normalizing flow for $q_\phi$ and the NUTS sampler to obtain posterior samples. In the g-and-k experiments, we use the quantile based summary statistics: $s_1 = q_{50}, s_2 = q_{75} - q_{25}, s_3 = (q_{90} + q_{10} - 2 q_{50}) / (q_{90} - q_{10}), s_4 = (q_{90} - q_{10}) / (q_{75} - q_{25})$, which we verify to be informative. As this is a statistics-based method, we need to simulate $n$ iid samples of $x$ per sample of $\theta$ in order to compute $s(x_{1:n})$. We therefore adjust the number of prior samples for RSNLE such that the total simulation budget is $m$; see \Cref{tab:baseline-sim-budget} for the specific values. As RSNLE uses the observed data in each round, it is not amortised.

\paragraph{NPE-RS \citep{Huang2023}.}

This method also considers model misspecification in the space of summary statistics, and aims to learn a summary network $s_\psi: \mathbb{R}^{n\times d_{\mathcal{X}}} \to \mathbb{R}^k$, with learnable parameter $\psi \in \Psi$, jointly with the posterior network $q_\phi$ such that $s_\psi(x_{1:n})$ is in the vicinity of $s_\psi(x^o_{1:n})$. To do that, NPE-RS penalises an NPE objective with the MMD$^2$ between summary statistics of simulated data and observed data:
\begin{talign*}
   \hat\phi_m^\lambda := \arg\min_{\phi \in \Phi} - \frac{1}{m} \sum_{i=1}^m \log q_\phi(\theta_i | s_\psi(x_{1:n, i})) + \lambda \sum_{i=1}^m {\mathrm{MMD}}^2\left(s_\psi(x_{1:n, i}), s_\psi(x^o_{1:n})\right).
\end{talign*}
The hyperparameter $\lambda >0$ controls the relative weight of the penalisation term in the optimisation. Once trained, the NPE-RS posterior $\pi_{\text{NPE-RS}}(\theta \mid x_{1:n}^o, \hat\phi_m^\lambda, \hat\psi_m^\lambda) = q_{ \hat\phi_m^\lambda}(\theta \mid s_{\hat\psi_m^\lambda}(x_{1:n}^o))$ is obtained by a simple forward pass.
Following the original implementation, we use an MAF for the inference network and set $\lambda = 1$.  We use a deep set architecture \citep{zaheer2017deep} for the summary network with $20$ dimensional output. Similar to RSNLE, we adjust the number of prior samples used for NPE-RS in order to keep the total simulation budget fixed to $m$. Note that NPE-RS is also not amortised as it uses $x_{1:n}^o$ during training.

\paragraph{NPL-MMD \citep{Dellaporta2022}.} 

NPL-MMD employs the Bayesian nonparametric learning approach with an MMD-based loss. Instead of using a prior on the parameters, this method sets a Dirichlet process (DP) prior on the true data-generating process $\mathbb{P}_0 \in \mathcal{P}(\mathcal{X})$, with two prior hyperparameters: a concentration parameter $\alpha \geq 0$ and a centring measure $\mathbb{F} \in \mathcal{P}(\mathcal{X})$. 
The posterior conditioned on $x^o_{1:n} \sim \mathbb{P}_0$ is also Dirichlet process due to conjugacy. 
Following the original paper, we set the concentration parameter to $\alpha = 0$, leading to a non-informative prior. In this case, (approximate) samples from the posterior are obtained through:
\begin{talign*}
    \mathbb{P}_j \mid x_{1:n}^o \sim \mathrm{DP}\text{-}\mathrm{Posterior}
    \Leftrightarrow \mathbb{P}_j = \sum_{i=1}^n w_{j,i} \delta_{x^o_i},\quad w_{j,1:n} \sim \mathrm{Dir}(1,\dots,1). 
\end{talign*}
Given a draw $\mathbb{P}_j$ from the DP posterior, a posterior sample in parameter space is obtained by solving
$\hat \theta_j
= \arg\min_{\theta \in \Theta}
{\mathrm{MMD}}^2(\tilde x_{1:n'}, x_{1:n})
\approx
\arg\min_{\theta \in \Theta}
\mathrm{MMD}^2(\mathbb{P}_j, \mathbb{P}_\theta)$. Thus, NPL-MMD is not amortised as it needs to be rerun from scratch for each new $x_{1:n}^o$. Each posterior draw requires $n' \times N_{\mathrm{step}}$ simulator evaluations, where $N_{\mathrm{step}}$ denotes the number of optimization iterations; see \Cref{tab:baseline-sim-budget} for the exact values used in order to keep the simulation budget the same.

\paragraph{ACE \citep{Gao2023}.}

ACE trains a neural network $\mathrm{NN}_\phi$ to approximate the loss function in GBI. Specifically they learn the MMD between each simulated dataset $x_{1:n,i}$, $i=1, \ldots, m$, and a target dataset $x'_{1:n,j}$, $j=1, \ldots, m'$. The target dataset consists of (i) $m$ simulated datasets, (ii) $m_{\mathrm{aug}}=100$  noise augmented samples, and (iii) $K$ independent observation datasets $x_{1:n, 1}^{o}, \dots, x_{1:n,K}^{o}$, 
such that $m' := m + m_{\mathrm{aug}} + K$ is the total number of target samples. Then,  the neural network is trained by solving:
\begin{talign*}
    \hat\phi_{m'} := \arg\min_{\phi \in \Phi} \frac{1}{m m'}\sum_{i=1}^m \sum_{j=1}^{m'} \mathrm{NN}_\phi(\theta_i, x'_{1:n,j}) - {\text{MMD}}^2(x_{1:n,i}, x'_{1:n,j})^2.
\end{talign*}
Once trained, the ACE posterior $\pi_{\mathrm{ACE}}(\theta \mid x^{o}_{1:n, k}, \hat\phi_{m'}) \propto \exp (-\beta\mathrm{NN}_{\hat\phi_{m'}}(\theta, x^{o}_{1:n, k})) \pi(\theta)$ can be obtained via MCMC, making it partially amortised. 
As the loss computation is amortised, implementing \Cref{alg:baseline-lr-calibration} for ACE does not require additional simulations.

\paragraph{GBI-SR \citep{Pacchiardi2024}.}  GBI-SR is a GBI method for simulators based on scoring rules. The original paper considers an energy score and a kernel score, with the latter giving a generalised posterior equivalent to MMD-Bayes \citep{Cherief-Abdellatif2019-MMDBayes}. In this paper, we only consider the kernel score, which gives robustness guarantees with a bounded kernel. In contrast, the energy score does not in fact lead to robustness beyond trivial settings (e.g. bounded domains) where even standard Bayesian posteriors are robust. Using the kernel score $\hat S(\mathbb{P}_\theta, x_i^o)$, the GBI-SR posterior is $\pi_{\mathrm{GBI-SR}}(\theta | x^o_{1:n}) \propto \exp(-\beta \sum_{i=1}^n \hat S(\mathbb{P}_\theta, x_i^o) ) \pi(\theta)$, where
\begin{talign*}
\hat S(\mathbb{P}_\theta, x_i^o) = \frac{1}{n'(n'-1)} \sum_{j,k=1, j\neq k}^{n'} k(x_j, x_k) - \frac{2}{n'}\sum_{j = 1} ^{n'}k(x_j, x^o_i), \quad x_{1:n'} \sim \mathbb{P}_\theta.
\end{talign*}
For differentiable simulators, as is the case with the g-and-k distribution, they propose using stochastic gradient MCMC, specifically adaptive stochastic gradient Langevin dynamics (adSGLD) \citep{jones2011adaptive}. For non-differentiable simulators such as the SIR and the radio propagation model where gradients cannot be computed, we use the pseudo-marginal MCMC \citep{Andrieu2009}, as per their recommendation. As the kernel score computation requires the observed data, GBI-SR is not amortised. Obtaining one posterior sample in GBI-SR requires $n'$ simulation. Assuming $N_{\text{warm}}$ warm-up steps in MCMC and total $N_{\text{post}}$ posterior samples, GBI-SR requires $n'(N_{\text{warm}} + N_{\text{post}})$ simulations. Additionally, simulation budget is spent on in computing $\hat \theta_n$ and tuning the learning rate using \Cref{alg:baseline-lr-calibration}; see \Cref{tab:baseline-sim-budget} for the exact values used in our experiments. $M$ denotes the number of posterior draws for the bootstrap samples as before, $n'_b$ denotes $n'$ in the bootstrap setting, and $N_{\text{post},\hat{\theta}}$ denotes the number of posterior samples used to estimate $\hat{\theta}_n$ for non-differentiable simulators.

\begin{table}[t]
\centering
\small
\setlength{\tabcolsep}{4pt}
\caption{Simulation budget formulas, parameter configurations, and total forward simulations for all methods on the g-and-k and SIR, and Turin simulators.}
\label{tab:baseline-sim-budget}
\begin{tabular}{p{2cm}|p{5cm}|p{6.5cm}|c}
\hline
Method & Formula & Value & Budget \\
\hline

\multicolumn{4}{l}{\textbf{g-and-k}}\\
\hline \hline
Our methods & $m$ & $m = 100{,}000$ & $100{,}000$ \\ \hline
RSNLE & $Rmn$ & $R = 2,\; m = 500, n = 100$ & $100{,}000$ \\ \hline
NPE-RS & $mn$ & $m = 1000,\; n = 100$ & $100{,}000$ \\ \hline
NPL-MMD & $N_{\text{post}}\,n'\,N_{\text{step}}$ & $N_{\text{post}} = 500,\;n' = 20,\; N_{\text{step}} = 10$ & $100{,}000$ \\ \hline
ACE & $mn$ & $m = 1000,\; n = 100$ & $100{,}000$ \\ \hline
GBI-SR &
$Tn'_b + |\beta_{\text{list}}|\,B\,M\,n'_b + (N_{\text{warm}}+N_{\text{post}})\,n'$ &
$ T = 200,\; n'_b = 7,\; |\beta_\text{list}|=3,\; B = 20\; M = 200\;,
N_\text{warm}= N_\text{post} = 500, \; n' = 15\; $ &
$100{,}400$ \\ 
\hline

\multicolumn{4}{l}{\textbf{SIR / Turin}}\\
\hline \hline
Our methods & - & $m = 50{,}000$ & $50{,}000$ \\ \hline
ACE & - & $m = 500,\; n = 100$ & $50{,}000$ \\ \hline
GBI-SR (only SIR)  &
$N_{\text{post},\hat \theta} n'_b + |\beta_{\text{list}}|\,B\,M\,n'_b + (N_{\text{warm}}+N_{\text{post}})\,n'$ &
$ N_{\text{post},\hat \theta} = 200, \;n'_b = 4,\; M = 172,\; n' = 8$ (others same as g-and-k) &
$50{,}080$ \\ \hline
\hline

\end{tabular}
\end{table}

\subsection{Additional results}
\label{app:additional_experiments}

We now present additional results from our experiments in \Cref{sec:experiments}.

\subsubsection{The g-and-k simulator}
\paragraph{Calibrating $\beta$ for g-and-k distribution.}

In \Cref{fig:gnk_beta_coverage}, we show the trace plots from estimating the learning rate $\beta$ for NSM-Bayes and NSM-Bayes-conj using the method described in \Cref{app:setting-learning-rate}. We observe that the starting value $\beta_0$ for NSM-Bayes was low, which resulted in 100\% coverage. As the stochastic steps proceed, $\beta$ keeps increasing until the coverage drops towards the target of 95\%. For NSM-Bayes-conj, the reverse occurs: starting value of $\beta$ is too high, leading to near 50\% coverage. So the $\beta$ value increases over the iterations, until they reach 95\% coverage, after which the values stabilise. In some of the runs for NSM-Bayes-conj, the coverage is much lower. This occurs when the matrix inversion in the closed-form expression for $\hat \theta_n$ becomes numerically unstable, leading to $\hat \theta_n$ being far from the true parameter value $\theta^\star$.

\begin{figure}[t]
\centering
\begin{minipage}[t]{0.6\textwidth}\vspace{0pt}
  \centering
  \includegraphics[width=\linewidth]{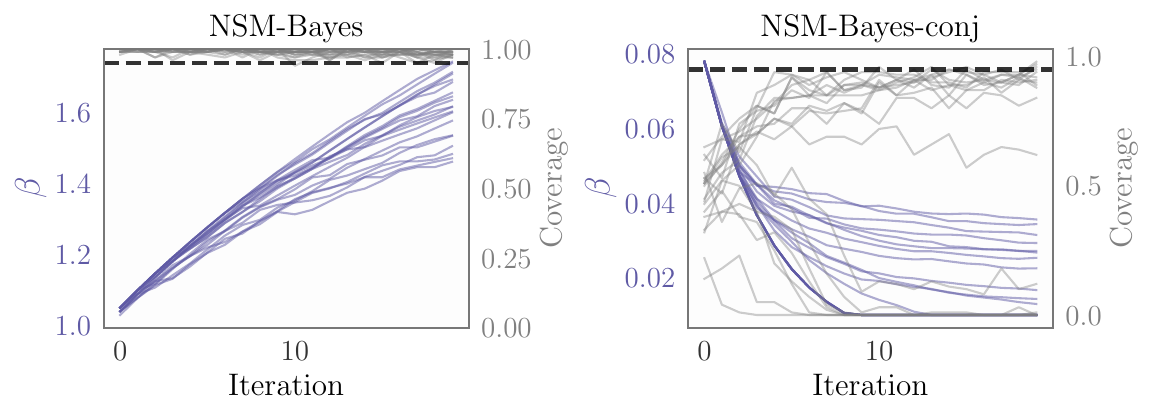}
  \caption{Trace plots obtained across the 20 runs during the stochastic approximation of the learning rate in the g-and-k experiment.}
  \label{fig:gnk_beta_coverage}
\end{minipage}\hfill
\begin{minipage}[t]{0.38\textwidth}\vspace{0pt}
  \centering
  \captionof{table}{Training and inference times for different methods in the SIR experiment with undercounting contamination model.}
  \label{tab:sir_times}
  \scalebox{0.72}{
    \begin{tabular}{@{}rcc@{}}
    \toprule
    \multicolumn{1}{l}{} & \begin{tabular}[c]{@{}c@{}}Training time\\ {[}s{]}\end{tabular} &
    \begin{tabular}[c]{@{}c@{}}Inference time\\ {[}s{]}\end{tabular} \\
    \midrule
    NLE & 92 (12) & 63 (71) \\
    ACE & 545 (39) & 390 (6) \\
    GBI-SR & - & 92 (2) \\
    NSM-Bayes & 92 (12) & 167 (1) \\
    NSM-Bayes-conj & 1312 (117) & 19 (0.1) \\
    \bottomrule
    \end{tabular}}
\end{minipage}
\end{figure}

\subsubsection{ The SIR simulator} The marginal posterior plots from the SIR undercounting experiment in the main text are shown in \Cref{fig:sir_posteriors_undercounting}. Here, NLE catastrophically fails, yielding very confident but heavily biased posteriors. In contrast, NSM-Bayes and NSM-conj posteriors are accurately concentrating around the true $\theta$. ACE and GBI-SR do not perform well in this case, mainly because the simulation budget is not large enough for them.
We report the training and inference times in \Cref{tab:sir_times}. In this higher dimensional case, NSM-Bayes-conj takes 19 seconds to calibrate $\beta$. Its training time is also large, as the network keeps training for all 1000 epochs. The validation loss keeps going down and so the stopping criterion does not get activated. 

\begin{figure}[t]
    \centering
    \includegraphics[width=0.9\linewidth]{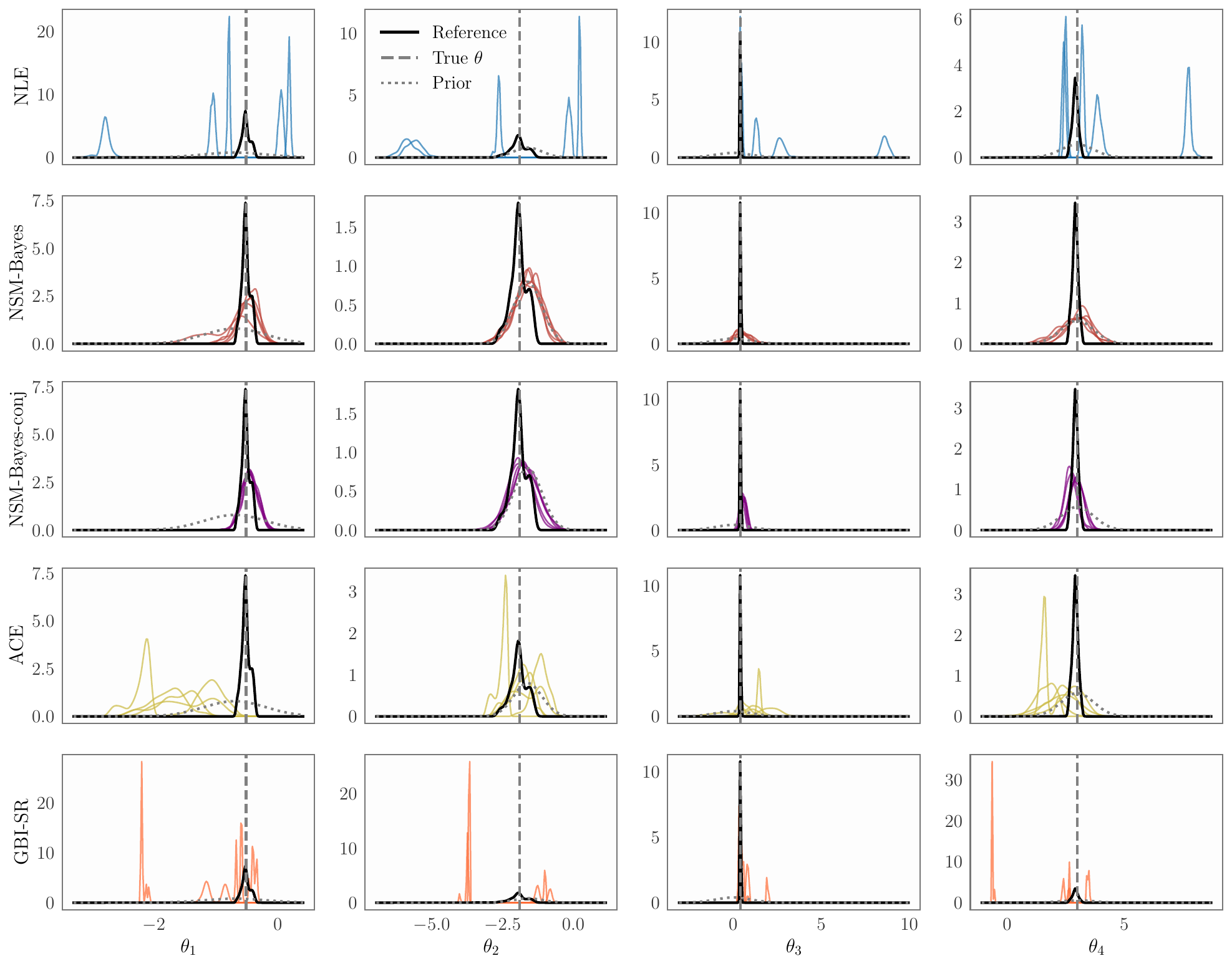}
    \caption{Kernel density estimates of the marginal posterior plots of NLE (\legendbox{nlecolor}), NSM-Bayes (\legendbox{gencolor}), NSM-Bayes-conj (\legendbox{c1color}), ACE (\legendbox{acecolor}), and and GBI-SR (\legendbox{gbisrcolor}) in the SIR undercounting experiment. 
    }
    \label{fig:sir_posteriors_undercounting}
\end{figure}

\paragraph{Varying learning rate $\beta$.} In \Cref{fig:sir_varying_beta}, we analyse the behaviour of NSM-Bayes-conj as the learning rate varies. The learning rate $\beta$ controls the coverage of the GBI posterior. As $\beta$ increases, the posterior gets concentrated around the loss minimising parameter estimate $\hat \theta_n$. When $\beta$ is low, the posterior becomes less impacted by the NSM loss and resembles the prior.

\begin{figure}
    \centering
    \includegraphics[width=\linewidth]{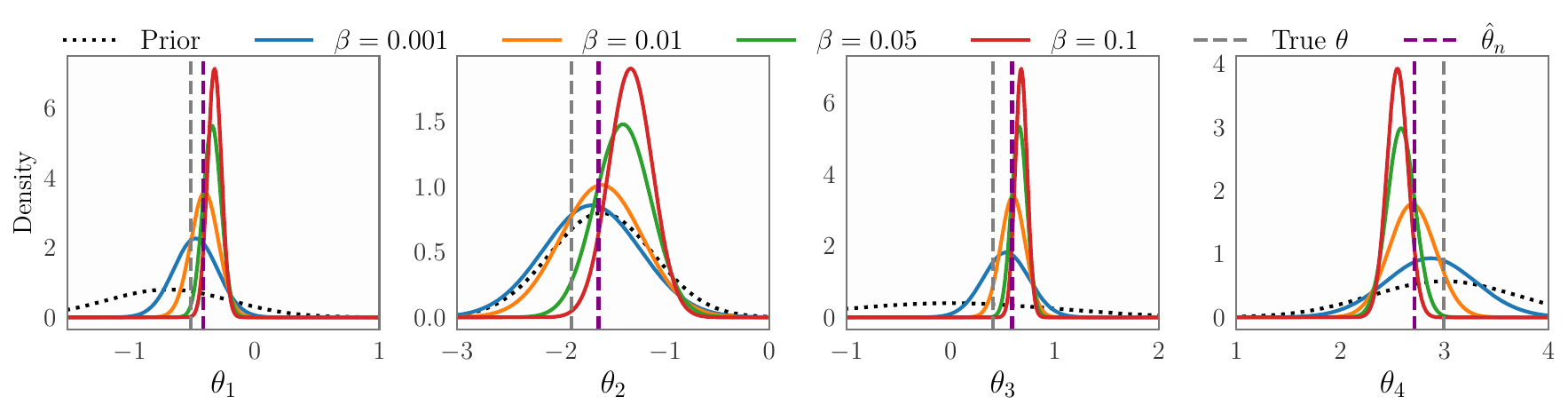}
    \caption{Varying the learning rate $\beta$ for NSM-Bayes-conj in the SIR undercounting experiment.}
    \label{fig:sir_varying_beta}
\end{figure}

\begin{wrapfigure}[]{r}{0.35\textwidth}
\begin{minipage}{0.35\textwidth}
    \centering
    \vspace{-2ex}
    \captionof{table}{Ablation study on NSM-Bayes-conj method involving the estimators used in the IMQ weight function $w(x)$.}
    \label{tab:sir_not_robust}
    \scalebox{0.77}{
    \begin{tabular}{@{}lll@{}}
    \toprule
                        & MSE                           & $\text{MMD}^2_{\text{ref}}$                     \\ \midrule
    Robust $\hat \nu_n$, $\hat \Xi_n$      & \textbf{0.42} \textcolor{gray}{(0.05)}  & \textbf{0.26} \textcolor{gray}{(0.06)} \\
    Sample $\hat \nu_n$, $\hat \Xi_n$  & 2.27 \textcolor{gray}{(0.00)} & 0.41 \textcolor{gray}{(0.04)} \\ \bottomrule
    \end{tabular}}
    \vspace{-2ex}
\end{minipage}
\end{wrapfigure}

\paragraph{Ablation study.} We now analyse the impact of our choice of robust mean (i.e. median) $\hat \nu_n$ and robust covariance $\hat \Xi_n$ estimators on the performance of NSM-Bayes-conj. To that end, we replace them with their non-robust, sample-based estimators, i.e., sample mean and sample covariance in the SIR undercounting experiment. The performance metrics are reported in \Cref{tab:sir_not_robust}. We observe that the performance degrades across both metrics when choosing the non-robust estimators for $\nu$ and $\Xi$ in the IMQ weight function.

\subsubsection{The radio propagation simulator}
 
Finally, the inference times for all the methods in the radio propagation experiment are reported in \Cref{tab:turin_inference_time}. Note that the MCMC in NLE is much faster in this case due to the conditional density being a mixture density network instead of the normalising flow. NSM-Bayes takes longer in this case due to repeated MCMC needed during learning rate tuning, as the effective sample size of the importance sample kept going below 30\%. When the data-dimension is higher, tuning $\beta$ takes longer in NSM-Bayes-conj.

\begin{table}
\centering
\caption{Inference times for the radio propagation example shown in \Cref{fig:turin_data}.}
\label{tab:turin_inference_time}
\scalebox{0.85}{
\begin{tabular}{@{}lcccc@{}}
\toprule
 & NLE  & ACE & NSM-Bayes & NSM-Bayes-conj \\ \midrule
\multicolumn{1}{r}{Inference time {[}s{]}} & 13.7 &  589.0   & 198.3     & 19.8           \\ \bottomrule
\end{tabular}}
\end{table}

\subsection{Analysis of baseline methods} \label{app:add_exp_baselines}

\paragraph{NPL-MMD.}

This method places a non-informative Dirichlet process prior on the data-generating measure, whereas all the other methods place a prior on the parameter space, which is informative in the g-and-k experiment. Thus, forcing NPL-MMD to have the same simulation budget as the other methods is somewhat unfair to it, considering it needs to learn the posterior without knowledge of the prior. As shown in \Cref{fig:baseline-add}(a), increasing the budget substantially improves accuracy, particularly for $\theta_1$ and $\theta_2$, with $\theta_4$ also moving closer to the true value. However, the posterior for $\theta_3$ remains biased and an even much larger budget would likely be required to improve on this.

\paragraph{ACE.}
Although ACE achieves competitive performance in the g-and-k and the radio propagation experiments, it fails to yield robust posteriors in the SIR experiment. This is primarily due to the limited simulation budget, which leads to a poor approximation of the loss function with the neural network. Doubling the simulation budget to $m=10^5$ substantially improves its performance, as shown in \Cref{fig:baseline-add}(b). This indicates that NSM-Bayes and NSM-Bayes-conj are more sample-efficient than ACE, requiring less simulation budget to achieve similar performance.

\paragraph{GBI-SR.}

While stochastic gradient MCMC could be used for the g-and-k experiment, the non-differentiability of the SIR simulator forced us to use pseudo-marginal MCMC.  Given the slower convergence and poorer mixing of pseudo-marginal MCMC \citep{Pacchiardi2024}, 500 posterior samples were insufficient for adequate mixing. We therefore increase the simulation budget for GBI-SR to more than $4\times$ the original budget and draw 10,000 posterior samples. This leads to a substantial improvement in its performance, as shown in \Cref{fig:baseline-add}(c). Once again, this shows that NSM-Bayes and NSM-Bayes-conj are more sample-efficient (in terms of the number of simulations).

\begin{figure}[H]
    \centering
    \begin{subfigure}[Marginal NPL-MMD (\legendbox{nplmmdcolor}) posterior for the g-and-k model using $500$ posterior samples. We set $n' = 2^9, N_{\text{step}} = 1000$, and step size to $0.1$, leading to a total simulation budget of $2.56 \times 10^{7}$.]{
        \includegraphics[width=0.9\linewidth]{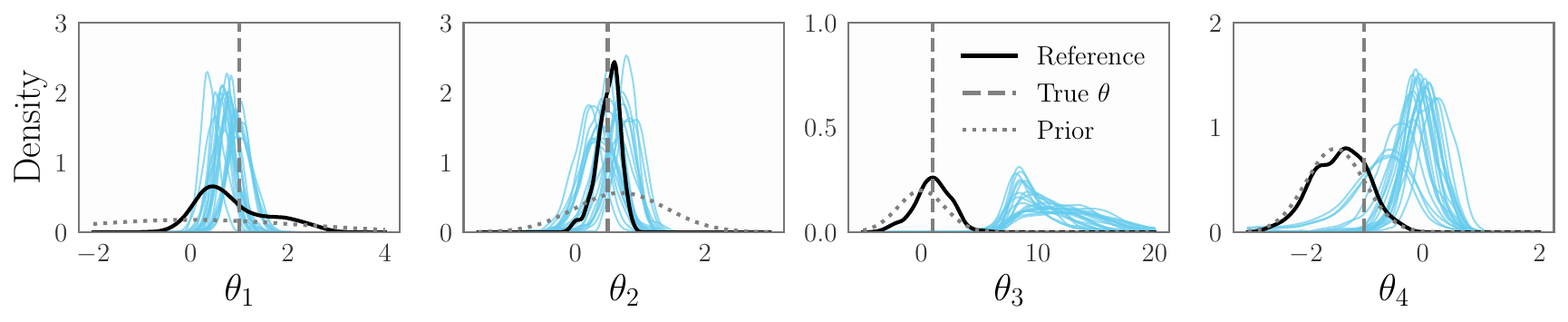}}
    \end{subfigure}
    \begin{subfigure}[Marginal ACE (\legendbox{acecolor}) posterior for the SIR undercounting experiment. We set $m = 1000$ with fixed $\beta = 100$, resulting in a simulation budget of $100{,}000$, which is twice the budget used in \Cref{sec:sir}.]{
        \includegraphics[width=0.9\linewidth]{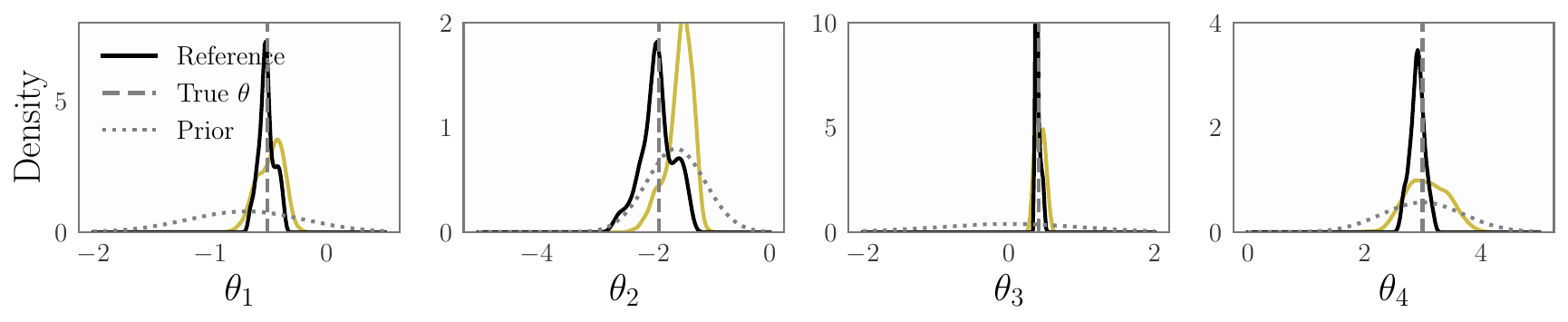}}
    \end{subfigure}
    \begin{subfigure}[Marginal GBI-SR (\legendbox{gbisrcolor}) posterior for the SIR undercounting experiment. We set $N_\text{post} = 10{,}000$, $N_\text{warmup}= 1000$, $n' = 200$, and fix $\beta = 1$, such that total simulation budget is $2.2 \times 10^{6}$.]{
        \includegraphics[width=0.9\linewidth]{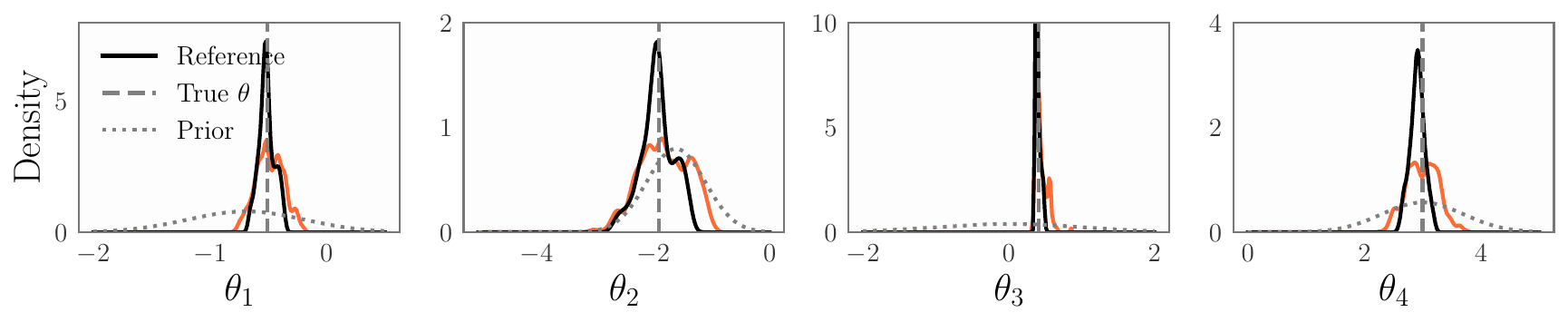}}
    \end{subfigure}
    \caption{Performance of NPL-MMD, ACE and GBI-SR under increased simulation budget.}
    \label{fig:baseline-add}
\end{figure}

\paragraph{RSNLE.}

\begin{wrapfigure}[]{r}{0.3\textwidth}
    \centering
    \vspace{-3ex}
    \includegraphics[width=0.9\linewidth]{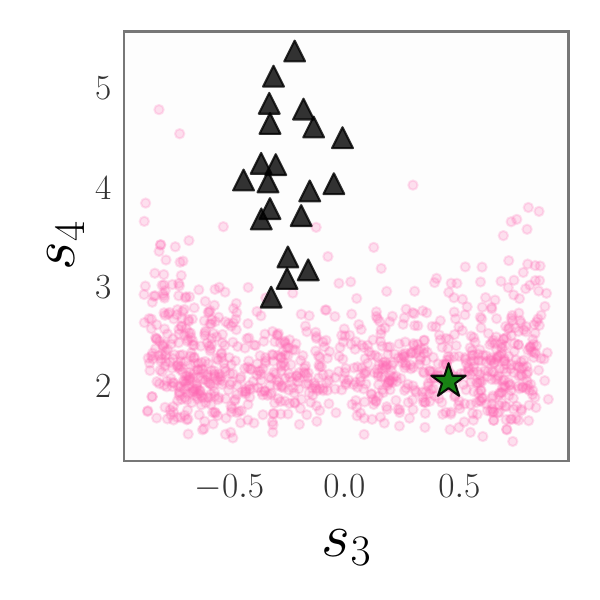}
    \vspace{-2ex}
    \caption{Analysing summary statistics in RSNLE.}
    \label{fig:gnk_rsnle_summary}
\end{wrapfigure}

Finally, we investigate the poor performance of RSNLE in estimating the $\theta_3$ parameter of the g-and-k distribution. To that end, we look at the scatter plot of the third and fourth simulated statistics (\legendbox{rsnlecolor}), $s_3$ and $s_4$, used for RSNLE in \Cref{fig:gnk_rsnle_summary}. We also plot the observed statistics from the 20 runs of the g-and-k experiment (\legendbox{black}), and the statistic corresponding to the uncorrupted data (\legendbox{darkgreen}) obtained from $\theta^\star$ with $\epsilon=0$. We see that the outliers cause a shift in both the observed $s_3$ and $s_4$. Note that while the observed $s_4$ go out of the simulated statistics distribution, this does not happen for $s_3$. Thus, the simulator can produce statistics that match the observed $s_3$, and so the posterior RSNLE yields for $\theta_3$ (which is determined by $s_3$) becomes biased. This points to a failure mode of RSNLE, wherein contamination in the data biases the observed statistic away from the uncorrupted statistic, but still within the training distribution that RSNLE sees.

\stopappendixtoc
\end{document}